\documentclass{article}

\usepackage[main, final]{neurips_2026}

\usepackage[utf8]{inputenc} 
\usepackage[T1]{fontenc}    
\usepackage[colorlinks,linkcolor=blue,filecolor=blue,citecolor=black,urlcolor=blue]{hyperref}       
\usepackage{url}            
\usepackage{booktabs}       
\usepackage{amsfonts}       
\usepackage{nicefrac}       
\usepackage{microtype}      
\usepackage{xcolor}         
\usepackage{empheq}
\usepackage{xspace}
\newcommand{\ours}{{\sc LITE}\xspace}
\newcommand{\oursmuon}{{Muon-LITE}\xspace}
\newcommand{\ourssoap}{{SOAP-LITE}\xspace}

\usepackage{tocvsec2}
\usepackage{titletoc}

\usepackage{subcaption}
\usepackage{xspace}
\usepackage{amsmath}
\usepackage{amsthm}
\usepackage{amssymb}

\usepackage{multicol}
\usepackage{multirow}
\usepackage{makecell}
\usepackage{enumitem}
\usepackage{wrapfig}
\usepackage{dsfont}
\allowdisplaybreaks
\usepackage{graphicx,epstopdf}

\usepackage{algorithm}
\usepackage{algorithmic}

\usepackage{threeparttable, booktabs}
\usepackage{mathtools}
\mathtoolsset{showonlyrefs=true}
\usepackage{float}

\usepackage{hyperref}[pdfusetitle]
\hypersetup{colorlinks,urlcolor=blue}
\usepackage{tabularray}

\usepackage[skins,listings]{tcolorbox}

\restylefloat{algorithm} 
\newtheorem{assumption}{Assumption}
\newtheorem{theorem}{Theorem}
\newtheorem{remark}{Remark}
\newtheorem{lemma}{Lemma}

\newtheorem{proposition}{Proposition}

\usepackage{diagbox}

\newcommand{\ie}{{\emph i.e.}\xspace}
\newcommand{\eg}{{\emph e.g.}\xspace}

\usepackage[capitalize,noabbrev]{cleveref}

\usepackage[textsize=tiny]{todonotes}

\usepackage{framed}
\colorlet{shadecolor}{orange!15}

\title{Accelerating LLM Pre-Training through Flat-Direction Dynamics Enhancement}

\author{
Shuchen Zhu\thanks{Equal contribution.} \\
  Peking University \\
  \And
  Rizhen Hu$^*$ \\
  Peking University \\
 \And
  Mingze Wang\thanks{Project lead.} \\
  Peking University \\
  \AND
  Mou Sun \\
  Zhejiang Lab, Zhejiang \\
  \And
  Xue Wang \\
  Meituan, Beijing \\
  \And
  Kun Yuan\thanks{Correspondence to: Kun Yuan
<kunyuan@pku.edu.cn> } \\
  Peking University \\
  \And
  Zaiwen Wen \\
  Peking University \\
}

\begin{document}

\maketitle

\begin{abstract}
  Pre-training Large Language Models requires immense computational resources, making optimizer efficiency essential. The optimization landscape is highly anisotropic, with loss reduction driven predominantly by progress along flat directions. While matrix-based optimizers such as Muon and SOAP use fine-grained curvature information to outperform AdamW, their updates tend toward isotropy, which is relatively conservative along flat directions yet potentially aggressive along sharp ones. To address this limitation, we first develop a unified Riemannian Ordinary Differential Equation (ODE) framework that offers a  
continuous-time interpretation of adaptive algorithms: the preconditioner induces a Riemannian geometry that mitigates ill-conditioning, while momentum serves as a Riemannian damping term that promotes convergence. Guided by these insights, we propose \textbf{LITE}, a generalized acceleration strategy that enhances training dynamics by applying larger Hessian damping coefficients and learning rates along flat trajectories. Extensive experiments demonstrate that LITE significantly accelerates both Muon and SOAP across diverse architectures (Dense, MoE), parameter scales (130M to 1.3B), datasets (C4, Pile), and learning-rate schedules (cosine, wsd). Theoretical analysis confirms that LITE facilitates faster convergence along flat directions in anisotropic landscapes, providing a principled approach to efficient LLM pre-training. The code  is 
available at \href{https://github.com/SHUCHENZHU/LITE}{https://github.com/SHUCHENZHU/LITE}.
\end{abstract}

\settocdepth{part}

\section{Introduction}
Large language models (LLMs) have advanced artificial intelligence with exceptional capabilities, yet the pre-training procedure is computationally intensive due to massive model and data scales. Improving pre-training efficiency   is essential for continued scaling, and the optimizer choice is crucial. Currently, AdamW~\cite{kingma2014adam,loshchilov2017decoupled}  serves as the  standard in most LLM pre-training pipelines, favored for its simplicity and efficiency. However, its coordinate-wise scaling (diagonal preconditioner)  ignores off-diagonal Hessian correlations, limiting the  ability to navigate complex optimization landscapes.

To address the limitation of the diagonal preconditioning, algorithms like Shampoo~\cite{pmlr-v80-gupta18a,shi2023distributeddataparallelpytorchimplementation} and KFAC~\cite{pmlr-v37-martens15} introduce matrix-based preconditioners to capture richer geometric structures, demonstrating superior convergence over AdamW in traditional deep learning tasks. Advancing this paradigm to the scale of LLM pre-training, a new wave of matrix-based  optimizers, such as Muon \cite{jordan2024muon} and SOAP \cite{vyas2025soap}, has emerged, achieving significant performance gains over AdamW. Notably, Muon has already achieved validated success in industrial-scale training scenarios~\cite{liu2025muon,kimiteam2025kimik2openagentic}.
These methods signify a paradigm shift towards exploiting matrix-level curvature information to accelerate convergence. 

Despite the empirical success of these advanced optimizers, several   limitations remain: 
\begin{enumerate}[leftmargin=1em]
    \item \textbf{Preconditioner-induced tendency toward isotropic update.} Although adaptive methods correct descent directions by effective preconditioners, prior studies ~\cite{pmlr-v97-staib19a,NEURIPS2020_f3f27a32,liu2025muon,lu2025understandingsoapperspectivegradient,wang2025muonoutperformsadamtailend} suggest that the {\em update magnitudes} 
    often tend towards {\em isotropic}, i.e., of comparable scale across different Hessian eigen-directions (refer to more discussions in Appendix \ref{iso-ver}).
    This behavior is suboptimal in the ill-conditioned landscape, which is  too cautious along flat directions while   potentially aggressive along sharp ones. 
    \item \textbf{Inadequate  momentum mechanisms for non-convexity.} Limited by the simple linear nature of Exponential Moving Average (EMA) that  inherently imposes an isotropic damping effect,   existing momentum  schemes inadequately exploit second-order information, leaving them susceptible to ill-conditioned curvature. While techniques like Nesterov acceleration implicitly  incorporate Hessian-driven damping  \cite{Shibin2021understanding}, the design  largely stems from convex optimization paradigms, making them insufficient for effectively accelerating optimization along non-convex directions.
\end{enumerate}  


Addressing these limitations requires understanding how preconditioning and different  momentum mechanisms jointly influence training dynamics in anisotropic landscapes. This motivates the following questions:

\begin{itemize}[leftmargin=1em]
    \item How to establish a  principled theoretical understanding of the  joint roles  of  the momentum and preconditioner  in adaptive optimization algorithms?
    \item Guided by  this  foundation, how can we   rectify the inadequate isotropic update and refine the momentum mechanism to achieve superior efficiency  in pre-training LLMs?
\end{itemize}

{\bf Our contributions} are as follows:
\begin{itemize}[leftmargin=1em]
    \item  We develop a {\bf unified Riemannian ODE framework} that encapsulates prevalent adaptive optimizers for  LLM pre-training  (\eg  AdamW, Muon, SOAP) and their Nesterov-accelerated counterparts. This framework provides a  continuous-time interpretation the joint mechanism between preconditioning and momentum from a continuous time perspective: the preconditioner induces a Riemannian geometry that mitigates the landscape's ill-condition, while momentum functions as a \textit{Riemannian damping} to foster  optimization    within this   metric. 
    
    \item Building on this framework, we propose  \textbf{LITE}, a  lightweight generalized  strategy  for accelerating adaptive optimizers in  LLM pre-training. Using our ODE analysis, \ours applies \textbf{ larger Hessian damping coefficients and learning rates in the flat directions}, which enhances momentum accumulation and amplifies update magnitudes therein.   
    This approach is designed to {\em accelerate  training dynamics along the flat directions} that dominate loss reduction within the ill-conditioned landscape, thereby  enhancing the  pre-training performance.  Furthermore, we provide a {\bf theoretical analysis} of the training dynamics of \ours within anisotropic landscapes, showing that it boosts dynamics in   flat directions and facilitates optimization, thereby supporting our design intuition.

    \item We conduct \textbf{extensive empirical evaluations} on the acceleration performance of \ours across diverse LLM pre-training settings, spanning different base optimizers (Muon, SOAP), model architectures (LLaMA,  QwenMoE), datasets (C4, Pile), and learning rate schedules (cosine, wsd), with model sizes from 0.13B to 1.3B parameters. Our results demonstrate that \ours-accelerated optimizers achieve remarkable loss reductions  and  more favorable scaling laws compared to standard baselines, delivering notable speedups in long-horizon training (\cref{fig:scaling_law}). This indicates the potential of \ours for superior scalability to larger models and extended token budgets. 

\begin{figure}[!ht]
    \centering
    \includegraphics[width=0.24\linewidth]{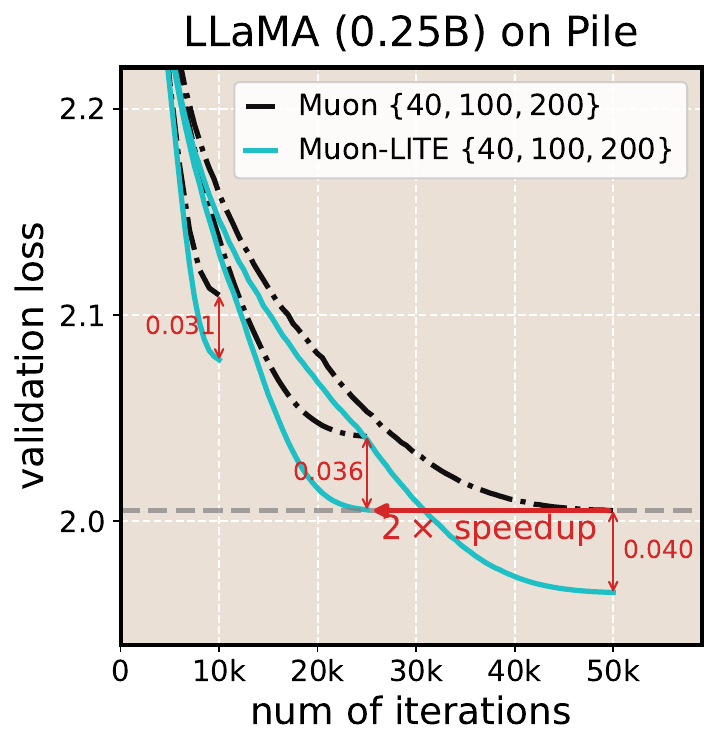}
    \includegraphics[width=0.247\linewidth]{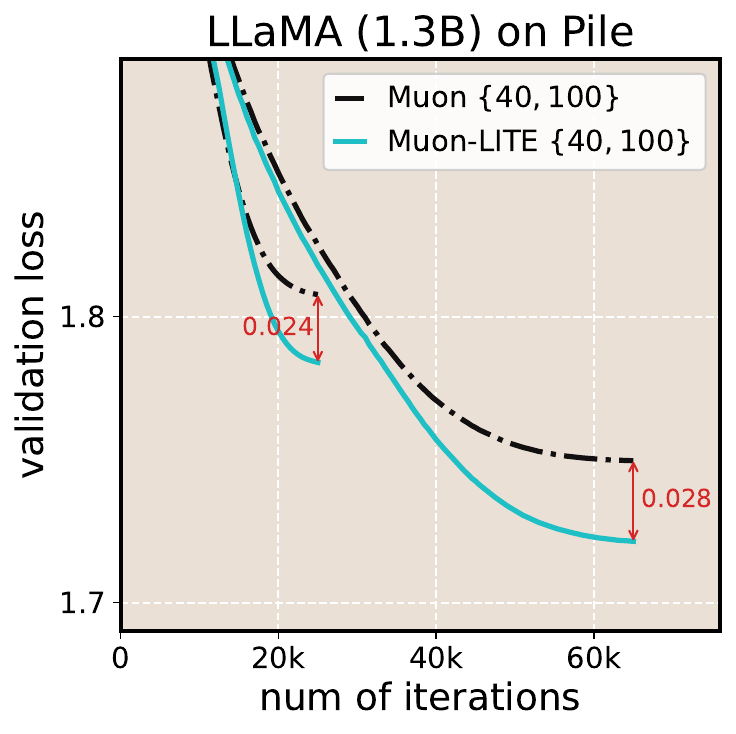}
    \includegraphics[width=0.25\linewidth]{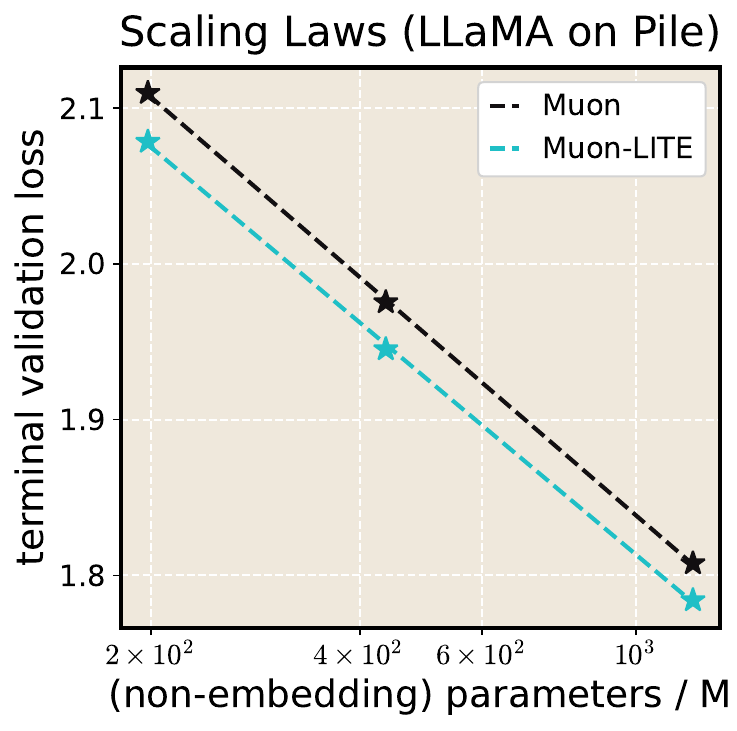}
    \caption{\oursmuon exhibits superior scaling behavior across varying token budgets (left, middle) and model sizes (right) compared to Muon, showing a sustained, non-decaying advantage with increased training. Legend numbers denote the token-to-parameter ratio for each curve (e.g., $100\times$).}
    \label{fig:scaling_law}
\end{figure}

\end{itemize}

\section{Related Work}

\paragraph{Accelerated Optimizers for LLM pre-training.}

Adaptive optimizers serve as the cornerstone of modern LLM pre-training. One stream of research focuses on refining the preconditioner: moving beyond simple element-wise scaling of AdamW~\cite{loshchilov2017decoupled}, methods like ~\cite{liu2024sophia,wang2025the} incorporate explicit curvature estimation, while matrix-based approaches~\cite{jordan2024muon,vyas2025soap,pethick2025training,liu2025cosmoshybridadaptiveoptimizer,lau2026polargradclassmatrixgradientoptimizers,si2025adamuonadaptivemuonoptimizer} use richer geometric structures to better approximate the curvature. Complementing these geometric improvements are refinements to the momentum mechanism, spanning multi-timescale momentum~\cite{pagliardini2025the}, Nesterov-type acceleration~\cite{xie2023adan,yuan2025mars}, and various momentum correction schemes~\cite{huang2025stablespamtrain4bitstably,liang2025cautiousoptimizersimprovingtraining}.

\paragraph{Training Dynamics and Loss Landscapes in Deep Learning.}\label{sec:eos}
Prior studies \cite{pmlr-v97-ghorbani19b,yao2020pyhessian,zhang2024why,su2025isotropiccurvaturemodelunderstanding} demonstrate the \emph{ill-conditioned and anisotropic} nature of the deep learning loss landscapes. Specifically, the Hessian spectrum is dominated by a massive bulk of near-zero and negative eigenvalues (referred to as \emph{flat directions}), while the large positive eigenvalues are significantly greater in magnitude but sparse in number (referred to as \emph{sharp directions}),   as schematically illustrated in \cref{fig:Schematic_illustration} (left). 
Recent investigations~\citep{song2025does, cohen2025understanding, wen2025understanding, wang2025the} into the training dynamics of gradient-based algorithms on this landscape have led to the consensus  
that the  anisotropic landscape induces a distinct {\em time-scale separation} in the training dynamics (\cref{fig:Schematic_illustration} (right)): \textbf{1) Fast dynamics:} Along \emph{sharp} directions, the dynamics exhibit rapid but non-divergent oscillations, which dictate training stability yet contribute minimally to loss reduction. \textbf{2) Slow  dynamics:} Along \emph{flat} directions, the dynamics evolves steadily but  slowly, dominating total loss reduction.

Informed by these insights, we propose \ours to accelerate training by boosting  the slow training dynamics along flat  directions (\cref{fig:Schematic_illustration} (right)).
More related works are in Appendix \ref{app:related_works}.

\begin{wrapfigure}{r}{0.53\textwidth}
    \centering
    \vspace{-.6cm}
\includegraphics[width=0.47\linewidth]{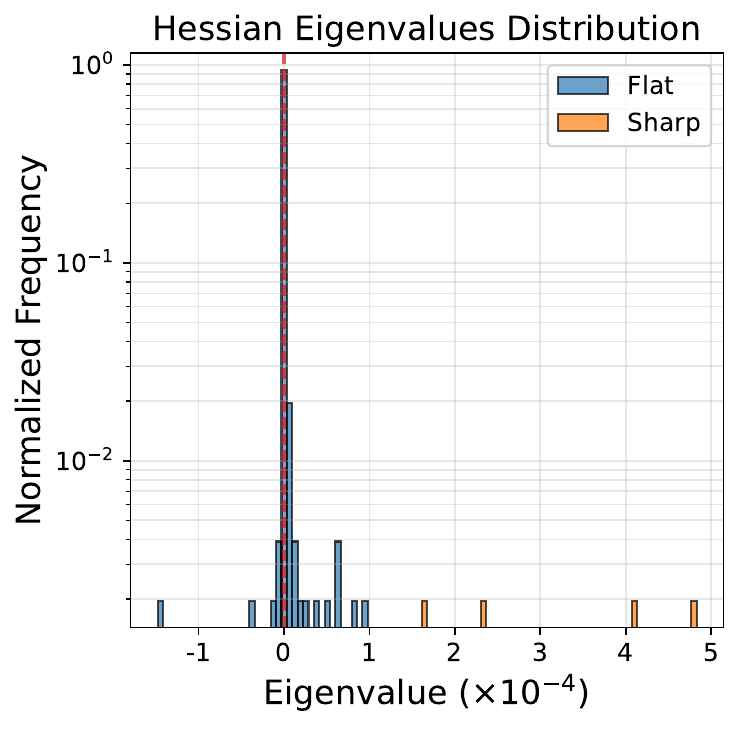} 
\includegraphics[width=0.5\linewidth]{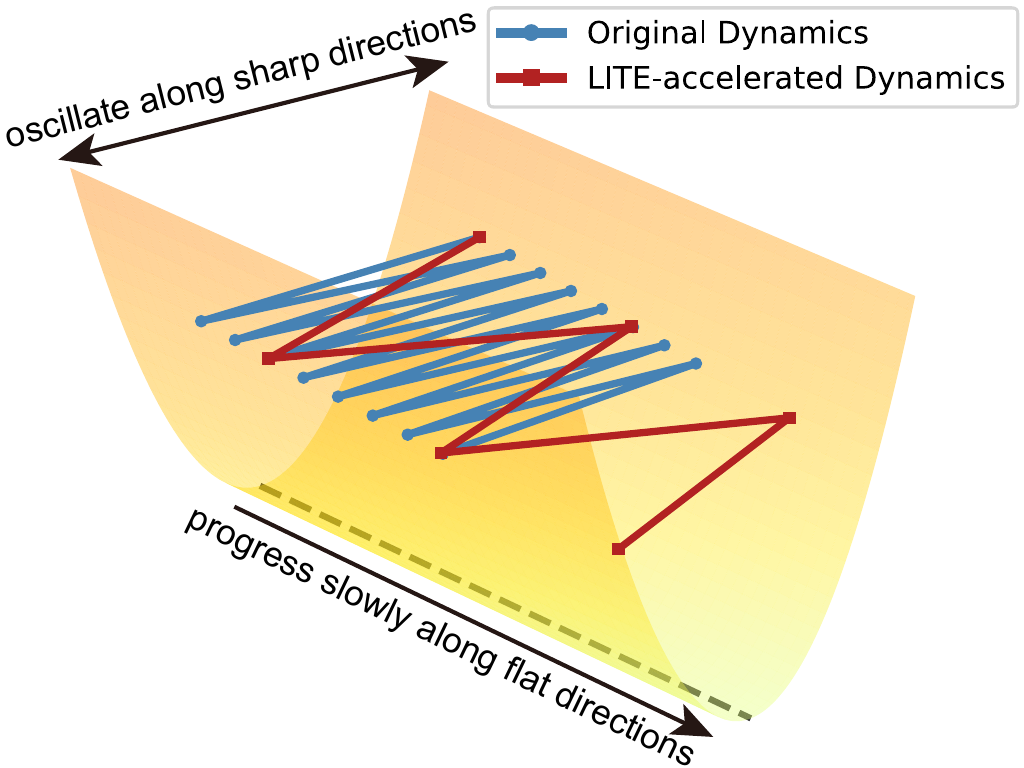}
    \vspace{-.2cm}
    \caption{(left)  Hessian eigenvalue  distribution  of an  \texttt{up\_proj} FFN block in a toy LLaMA model; (right) Schematic illustration of the time-scale separation property and the acceleration mechanism of \ours. }
\label{fig:Schematic_illustration}
    \vspace{-.2cm}
\end{wrapfigure}
 
\section{Preliminary}
\paragraph{Notations.} 
Let $v_k$ denote the $k$-th component of a vector $v$. For a positive definite matrix $F$, we define the induced inner product and norm as $\langle u,v\rangle_F=u^\top F v$ and $\|u\|_F=\sqrt{u^\top F u}$ respectively.  We  distinguish between Riemannian operators ($\operatorname{grad}, \operatorname{Hess}$) and their Euclidean counterparts ($\nabla, \nabla^2$). The symbol $\odot$ represents the element-wise product, and $\otimes$ denotes the Kronecker product. A subscript or superscript $F$ is occasionally attached to geometric operators (\eg $\nabla^F$) to emphasize  they are induced by the Riemannian metric $F$.  
\subsection{Manifold Optimization}
Manifold optimization deals with variables constrained to a Riemannian manifold. In this work, we consider a special case: the variable (parameter) space $\mathbb{R}^p$ not as a standard Euclidean space, but as a Riemannian manifold $\mathcal{M} = (\mathbb{R}^p, F)$ equipped with metric $F(w) \in \mathbb{R}^{p \times p}$ at each  $w\in \mathbb{R}^p$.  This geometric perspective allows us to analyze the optimization landscape through the lens of local curvature defined by $F$.
Detailed discussions are provided in Appendix \ref{app:riemannian-geometry}.
\paragraph{Tangent and Cotangent Spaces.} At any point $w \in \mathcal{M}$, the tangent space $T_w\mathcal{M}\cong \mathbb{R}^p$ represents the vector space of all possible directional velocities.   The Riemannian metric induces a norm $\| \cdot\|_{F(w)}$ on the tangent space.   The cotangent space $T_w^*\mathcal{M}$ is defined as the dual space to $T_w\mathcal{M}$, which  is equipped with the  norm $\|\cdot\|_{F(w)^{-1}}$ and is also isomorphic to $\mathbb{R}^p$. The metric   induces   linear isomorphisms   $F(w)^{-1}:T_w^*\mathcal{M}\to T_w\mathcal{M}$ and  $F(w):T_w\mathcal{M}\to T_w^*\mathcal{M}$. While the Euclidean gradient $\nabla f(w)$ naturally resides in  $T_w^*\mathcal{M}$, the Riemannian gradient is obtained by mapping it back to the tangent space via $\operatorname{grad} f(w) = F(w)^{-1} \nabla f(w)$.   It  gives the  steepest descent steps in the local geometry by $-\frac{F(w)^{-1} \nabla f(w)}{\|F(w)^{-1} \nabla f(w)\|_{F(w)}}=\lim_{\epsilon\to 0}\frac{1}{\epsilon} \operatorname*{argmin}_{\|h\|_{F(w)}\le \epsilon}({f(w+h)-f(w)}) $.

\paragraph{Levi-Civita Connection.} 
Levi-Civita connection $\nabla_{(\cdot)}(\cdot)$ provides a rigorous tool for differentiating vector fields on Riemannian manifolds, generalizing the concept of Euclidean directional derivatives to curved geometries.   For  vector fields $u, v$,  $\nabla_u v$ yields a vector field that  quantifies how $v$ twists or deviates from a parallel field as it moves along $u$. See more discussions in Appendix \ref{Levi-Civita}.

\subsection{ODE Perspective of Momentum-based Algorithms}\label{sec:ode-momentum}
Continuous-time modeling via ODEs  constitutes a rigorous and principled paradigm for analyzing momentum-based algorithms, offering \textit{deep theoretical insights and clear  intuitions} into their underlying dynamics.  \cite{JMLR:v17:15-084,Shibin2021understanding,Attouch2022} introduced an inertial system with Hessian damping  (ISHD) that can depict the momentum-based  algorithms in higher resolution, which takes the form  
\begin{equation}\label{agf-2-order}
    \ddot{w}_t+\alpha_t \dot{w}_t + \beta_t   \nabla^2 f(w_t)\dot{w}_t + \gamma_t \nabla f(w_t) = 0,
\end{equation}
where $\alpha_t,\beta_t,\gamma_t\ge0$   denote the coefficients for momentum decay, Hessian damping, and the gradient driving force, respectively.  
Discretizing \eqref{agf-2-order} recovers various momentum methods: specifically, setting $\beta_t=0$ yields Heavy Ball momentum, while $\beta_t>0$ corresponds to Nesterov-type momentum.  Crucially, the Hessian damping term $\nabla^2f(w_t) \dot{w}_t$ and its discretized counterpart $\nabla f(w_k)-\nabla f(w_{k-1})$ serve to mitigate oscillations, enhancing the optimization stability.
Further details on discretization, equivalent formulations, and the role of Hessian damping are provided in \cref{nes-discussion}.

\section{A Unified Riemannian ODE Framework for Understanding Adaptive Optimizers}\label{sec:ode_Framework}
Adaptive optimizers combine momentum with preconditioners, using momentum-induced acceleration while simultaneously adapting step sizes anisotropically for efficient training.  In this section, we interpret their mechanics  by modeling these optimizers as  inertial ODE systems   with Hessian  damping on   Riemannian manifolds (RISHD) equipped with preconditioner-induced metrics.

\subsection{A Unified Discrete Formulation}\label{sec:uni-dis-form} 
Formally, the update rules of most adaptive optimizers can be encapsulated in a unified formulation: 
\begin{align}\label{2-ode0-d}
\begin{dcases*}
     m_{k}=(1- \alpha)m_{k-1}+ \nabla f(w_k) ,
    \\ w_{k+1} =w_k-  \eta_k F(w_k)^{-1} (m_{k}+\beta\nabla f(w_k)),
\end{dcases*}
\end{align}
where $w_k$ denotes the parameters, $m_k$ is the   momentum \footnote{ The missing coefficient $\alpha$ of $\nabla f$ can be absorbed in $\eta_k$.},   $F(w_k)^{-1}$ serves as the   preconditioner, and  the term $\beta\nabla f(w_k)$  yields a Nesterov-type momentum.  In neural network training, the design of the preconditioner $F^{-1}$ typically involves three features: {\em 1) Block-Diagonal Structure}, treating $F$ block-wise to reduce computational overhead; {\em 2) EMA Estimation}, smoothing mini-batch noise to stabilize curvature estimation; and {\em 3) Hessian Alignment}, where $F$ approximates the  Hessian (\eg diagonal Hessian)  or its structurally aligned proxy (\eg uncentered gradient covariance $(\mathbb{E}[gg^\top])^{1/2}$ and gradient whitening matrix, related to the Gauss-Newton part of Hessian). By varying the structure of $F$ and the momentum coefficient $\beta$, {\bf the formulation \eqref{2-ode0-d} subsumes a broad family of adaptive methods}, including N-AdamW, Muon, and SOAP. See \cref{tab:precond_form} and Appendix \ref{app:adap-opt-example} for further details.
{\renewcommand{\arraystretch}{1.5}  
 \begin{table*}[!htbp]
\centering
\caption{Update forms (matrix and vector) and their estimated deterministic proxies of the metric $F$  for common adaptive optimizers. $m$ denotes the EMA-type momentum,  $g$ represents the stochastic gradient, and $M_\beta=M+\beta G$. Capital letters denote the matrix forms corresponding to their vector counterparts (lowercase). For AdamW, we present its Nesterov-accelerated variant N-AdamW. See Appendix \ref{app:adap-opt-example} for further discussions and details.}
\vspace{-0.1mm}
\resizebox{\textwidth}{!}{
\begin{tabular}{|c|c|c|c|}
\hline\small
Algorithms & \small Matrix Form  & \small Vector Form & \small $F$ (up to a constant factor)\\
\hline
\small N-AdamW & \small  $\frac{M+\beta G}{\sqrt{V}}$ & \small  $\operatorname{diag}v^{-\frac{1}{2}}(m+\beta g)$ & $(\operatorname{diag}\mathbb{E}[g   g^\top])^{\frac{1}{2}}$ \\
\hline
\small Muon & \small $M_\beta(M_\beta^\top M_\beta)^{-\frac{1}{2}}$ & \small $( (M_\beta^\top M_\beta) \otimes I)^{-\frac{1}{2}}(m+\beta g)$ & \small $( \mathbb{E}[G^\top G])^{\frac{1}{2}} \otimes I$ \\
\hline
\small SOAP & \small $Q_l \frac{Q_l^\top MQ_r}{\sqrt{V_{\text{rot}}} }  Q_r^\top $ & \small $(Q_r\otimes Q_l) (\operatorname{diag}  v_{\text{rot}}  )^{-\frac{1}{2}}(Q_r^\top\otimes Q_l^\top) m$ & \small $(Q_r\otimes Q_l) (\operatorname{diag}\mathbb{E}[g_{\text{rot}}g_{\text{rot}}^\top])^{\frac{1}{2}} (Q_r^\top\otimes Q_l^\top)$ \\
\hline
\end{tabular}}
\label{tab:precond_form}
\vspace{-2mm}
\end{table*}
}

\paragraph{Optimizers using Hessian Damping}
Several recent optimizers (\eg Lion, MARS \cite{yuan2025mars}, and Muon) incorporate Nesterov-style momentum  to surpass the performance of vanilla Heavy Ball momentum. The key step is   preconditioning  $  m_k+\beta  g_k$ as in \eqref{2-ode0-d} rather than   the standard momentum $m_k$. The $\beta g_k$ term induces an implicit Hessian damping effect by $\nabla f(w_k)-\nabla f(w_{k-1}) \approx \nabla^2 f(w_k)(w_k-w_{k-1})$, crucial for reducing oscillations in ill-conditioned landscapes (Appendix~\ref{nes-discussion}).

\subsection{The Riemannian ODE Framework} \label{section:rishd} 

\paragraph{Intuition from the discrete update.}
For brevity, we first assume that \(F(w)\equiv F\) and \(\eta_k\equiv \eta\) are constant. Then eliminating the momentum variable \(m_k\) from \eqref{2-ode0-d} yields
\begin{equation}\label{eq:discrete-second-order}
\begin{aligned}
    w_{k+1}={}&w_k+(1-\alpha)(w_k-w_{k-1})-\eta(1+\alpha\beta)F^{-1}\nabla f(w_k)\\
    &-\eta\beta(1-\alpha)F^{-1}\big(\nabla f(w_k)-\nabla f(w_{k-1})\big).
\end{aligned}
\end{equation}
Using the Taylor approximation \(\nabla f(w_k)-\nabla f(w_{k-1})\approx \nabla^2 f(w_k)(w_k-w_{k-1})\), this scheme corresponds to the continuous-time limit
\begin{equation}\label{eq:discrete-second-order-limit}
\ddot{w}_t+\alpha\dot{w}_t+\eta\beta(1-\alpha)F^{-1}\nabla^2 f(w_t)\dot{w}_t+\eta(1+\alpha\beta)F^{-1}\nabla f(w_t)=0.
\end{equation}
In this form, the Heavy-Ball component \(\alpha\dot{w}_t\) induces  isotropic damping, whereas the Nesterov-type momentum $F^{-1}\nabla^2 f(w_t)\dot{w}_t$ contributes additional Hessian damping under the \(F\)-induced geometry. This intuition remains valid when \(F(w)\) and \(\eta_k\) vary slowly. Next, we provide a quantitative continuous-time formulation to establish this (with a slight change of notations for brevity).

\paragraph{Continuous-time Viewpoint.}
The generalized form \eqref{2-ode0-d} can be interpreted as a approximated  discretization of a Riemannian ODE. 
Using Levi-Civita connection $\nabla_{(\cdot)}^F(\cdot)$ to generalize the Euclidean derivatives to the covariant derivatives, we can  extend the Euclidean ISHD \eqref{agf-2-order} to the   Riemannian ISHD (RISHD) on $(\mathcal{M},F)$:
\vspace{-1mm}
\begin{snugshade}

\begin{empheq}{align} \label{R_acc:flow}
\nabla_{\dot{w}_t}^F  \dot{w}_t + \alpha_t \dot{w}_t+ \beta_t \operatorname{Hess}_F(w_t)\dot{w}_t + \gamma_t \operatorname{grad}_F   f(w_t) = 0,
\end{empheq}

\end{snugshade} 
\vspace{-3mm}
where $\operatorname{Hess}_F$ and $\operatorname{grad}_F$ denote $F$-induced Riemannian Hessian and gradient respectively.
\vspace{-3mm}
\paragraph{First-order Tangent-Cotangent Formulation.}
 We can introduce the momentum variable $m_t$  tracking the velocity $\dot{w}_t$ to decouple \eqref{R_acc:flow} into a  first order system that evolves in the tangent space \eqref{eq:tangent} and cotangent space \eqref{eq:cotangent} respectively: 

\begin{proposition}[Adapted from  Proposition \ref{transform-2nd-to-1st-system}]\label{prop:1}
\mbox{}
If $\alpha_t=\alpha-  {\dot{\eta}_t}/{\eta_t}$, $\beta_t=\beta\eta_t$, $\gamma_t=\eta_t(\alpha\beta+1)$,   \eqref{R_acc:flow} is equivalent to the following ODE system (with a Riemannian correction  term $R_t$ defined in \eqref{Rt-expression-append}): 
\vspace*{-.1em}

\begin{snugshade}
\vspace*{-.4em}
\begin{empheq}[left=\empheqlbrace]{align}
     \dot{w}_t &=- \eta_t F(w_t)^{-1} (m_t+\beta\nabla f(w_t)) , \tag{\textsf{T}} \label{eq:tangent} \\
     \dot{m}_t &=-\alpha m_t +\nabla f(w_t)+ R_t. \tag{\textsf{C}} \label{eq:cotangent}
\end{empheq}
\end{snugshade}
\vspace*{-0.8em}
\end{proposition} 
The term $R_t$ mainly depends on $\eta_t$ and $\nabla F(w_t)$, and every summand of $R_t$ carries a factor of $\eta_t$ together with a derivative of $F$ (see \eqref{Rt-expression-append}). Thus $R_t$ is typically negligible  due to the small step size $\eta_t$ and the stable preconditioner (implying a bounded $\nabla F(w_t)$). Thus we can ignore $R_t$ for brevity. Beyond a qualitative interpretation, the exact equivalence in Proposition~\ref{prop:1} quantitatively characterizes how the variation of $F$ affects the curvature-induced error.

The decoupled formulation in \eqref{eq:tangent} and \eqref{eq:cotangent} reveals the structure of RISHD in \eqref{R_acc:flow}: 
Momentum evolves in the cotangent space, while the preconditioner $F^{-1}$ maps it   to the tangent space to drive  the parameter update. This form also naturally gives  the specific form in which modern adaptive methods incorporate momentum: the preconditioner is applied to the momentum accumulated from raw gradients, rather than accumulating preconditioned gradients by EMA. Furthermore, this Riemannian ODE framework is readily extensible to higher-order cases, enabling the characterization of complex momentum schemes such as the multi-timescale momentum in AdEMAMix~\cite{pagliardini2025the} (Appendix \ref{section-ademamix}).

\begin{remark}
A line of work has leveraged ODE techniques to characterize the continuous-time limit of accelerated gradient methods in non-Euclidean geometry~\cite{wibisono2016variational,duruisseaux2022variational}. Compared to these works, our RISHD additionally provides a high-resolution  interpretation tailored to Nesterov-type momentum  under the Riemannian metric.
\end{remark}

\paragraph{Joint effect of the preconditioner and momentum.}  The continuous-time ODE framework \eqref{R_acc:flow} is instrumental and intuitive for understanding the joint mechanism of preconditioning and momentum, which remains opaque in discrete formulations. Assuming that $F(w)$ varies slowly and aligns well with the Hessian, the induced Riemannian metric effectively mitigates ill-conditioning, as it yields a Riemannian Hessian $\operatorname{Hess}_F f(w)\approx F^{-1}(w) \nabla^2 f(w)$ with much  smaller condition number than the Euclidean Hessian $\nabla^2 f(w)$. Within this improved geometry, the Riemannian damping   $-(\alpha_t +\beta_t \operatorname{Hess} f(w))\dot{w}_t$  acts as the momentum effect,  further accelerating  optimization.

\paragraph{Connection with the Discrete Formulation.} Compared to the original second-order system \eqref{R_acc:flow}, the equivalent first-order coupled system offers a more direct path for numerical implementation. Applying a semi-implicit Euler discretization scheme to \eqref{eq:tangent} and \eqref{eq:cotangent} with step size $h=1$ and $R_t\approx0$ yields the discrete formulation \eqref{2-ode0-d}.

\section{The \ours Approach}\label{sec:lite-approach}
In this section, we introduce  \ours, an   acceleration strategy grounded in the characteristics of the loss landscape and the theoretical insights regarding the combined role of momentum and preconditioning established in Section \ref{section:rishd}. 
As discussed in   \cref{sec:eos},  the trajectory along sharp directions exhibits significant fluctuations, and the efficiency of traversal along flat directions becomes the critical determinant of the final loss. Motivated by this,  our strategy aims to \textbf{selectively accelerate optimization along flat directions while  preserving stability in sharp directions.} To implement this, we adopt a decoupled tuning scheme: we anchor the hyper-parameters governing the sharp directions   in \eqref{2-ode0-d}  to maintain stability, while adaptively adjusting those corresponding to the flat directions. In the following discussion, we derive a heuristic adjustment scheme for the flat-direction hyper-parameters from the perspective of continuous-time dynamics in \cref{sec:ode_Framework}.
\subsection{Acceleration Methods from the ODE Perspective}
 We begin by analyzing the role of momentum through the lens of the continuous-time limit  in \eqref{agf-2-order},  focusing initially on the Euclidean case of \eqref{R_acc:flow} for clarity. 
Intuitively, the system models the dynamics of a particle moving within a potential field defined by $f$. The particle at $w_t$ is subject to the combined effects of a driving force $-\gamma_t \nabla f(w_t)$ and a curvature-adaptive damping (drag force) $(\alpha_t I + \beta_t \nabla^2 f(w_t)) (-\dot{w}_t)$ opposing its velocity $  \dot{w}_t$, with its acceleration $\ddot{w}_t$ governed by Newton's Second Law. 
 Along flat directions, we can enhance the accumulation of   velocity by increasing the driving force coefficient $\gamma_t$ while minimizing the effective damping coefficient $\alpha_t I + \beta_t \nabla^2 f(w_t)$ that is positive definite to avoid divergence.  Since flat directions are often characterized by   non-convexity, our  strategy is 
\textbf{increasing  $\gamma_t$ and $\beta_t$ in the flat directions.} \footnote{Regarding damping, we do not tune $\alpha_t$ for simplicity. While $\alpha_t$ offers a similar mechanism for reducing damping, $\beta_t$ is preferred as it provides curvature adaptivity.}  

The above discussion, framed in terms of damping and motion in Euclidean space, can be naturally extended to Riemannian manifolds (curved spaces). The corresponding physical background shifts from Newtonian mechanics to Lagrangian mechanics \cite{arnold1989mathematical}, yet the same conclusions hold. 
For sharp directions, the coefficients remain unchanged to maintain stability conditions. Given the relationships $\gamma_t=(1+\alpha\beta)\eta_t$ and $\beta_t=\beta \eta_t$ in \eqref{R_acc:flow}, we can amplify $\gamma_t$ and $\beta_t$ by increasing $\beta$ and $\eta_t$.
Finally, assuming well-aligned top eigenspaces between $F$ and  Hessian, the accelerated dynamics is given by:
\vspace{-2mm}
\begin{align}\label{2-ode01}
\begin{dcases*}
     \dot{w}_t = 
     \begin{aligned}[t]
        &-\eta_t F(w_t)^{-1} P_t(m_t + \beta_1 \nabla f(w_t)) -\chi\eta_t F(w_t)^{-1}    Q_t (m_t + \beta_2 \nabla f(w_t))  ,
     \end{aligned} \\
     \dot{m}_t = -\alpha m_t + \nabla f(w_t),
\end{dcases*}
\end{align}
where $P_t$ is the projection to the  sharp subspace, $Q_t=I-P_t$,  and $\chi\ge1$, $\beta_2\ge\beta_1\ge0$ are hyper-parameters. Discretizing  \eqref{2-ode01} like \eqref{2-ode0-d}    yields the   accelerated   algorithms framework \ours (\cref{alg:LITE}). This formulation naturally yields amplified update magnitudes and enhanced momentum along flat directions, facilitating faster traversal through these slow-progressing regimes. An illustrative theoretical analysis of  acceleration mechanism in \ours on a quadratic function is in Appendix \ref{app:Illustrative_analysis}.
\begin{algorithm}[H]
  \caption{\ours Strategy}
  \label{alg:LITE}
  \begin{algorithmic}[1]
    \STATE {\bfseries Input:}  $w_0, m_0,\{\eta_k\},\alpha,\chi\ge1,\beta_2\ge\beta_1\ge0$.
    \FOR{$k=0$ {\bfseries to} $K$}
    \STATE Update momentum $m_k=(1-\alpha)m_{k-1}+\nabla f(w_k)$.
    \STATE Estimate projection $Q_k$ onto the flat directions. Then estimate the update direction
    \begin{snugshade}
    \vspace{-.4cm}
    \begin{equation}
    \begin{aligned}
u_k=&{\color{violet}(I-Q_k)}F(w_k)^{-1}(m_k+{\color{violet}\beta_1} \nabla f(w_k))
+ {\color{violet}\chi Q_k} F(w_k)^{-1}(m_k+{\color{violet}\beta_2} \nabla f(w_k)) .
        \end{aligned}\nonumber
    \end{equation} 
    \end{snugshade}
    \vspace{-.1em}
    \STATE Update parameters $w_{k+1}=w_k-\eta_k u_k$.
    \ENDFOR
  \end{algorithmic}
\end{algorithm}

\begin{remark}
Introducing $\chi$ is mathematically equivalent to modifying the preconditioner $F(w_k)^{-1}$ to $(I-Q_k)F(w_k)^{-1} + \chi Q_k F(w_k)^{-1}$. However, the inclusion of $\beta_2$  introduces a distinct adjustment on momentum damping along flat directions, which  cannot be achieved solely by altering  $F^{-1}$.
\end{remark}

\subsection{Practical Implementation of \ours}\label{practical-acc}
In this section, we present practical methods for implementing \ours (\cref{alg:LITE}) in LLM-pretraining. We specifically target the anisotropic landscape inherent to each parameter block, aiming to efficiently identify its sharp directions (and flat directions via the orthogonal complement).  
\paragraph{Efficiently approximating the  flat directions} Prior studies \cite{JMLR:v21:17-678,morwani2025a} demonstrate that the  Shampoo-class  preconditioners effectively approximates the Gauss-Newton component, the dominant term in the Hessian. Since both Muon and SOAP adopt Shampoo-style  Kronecker factorization to approximate the whitening metric, we hypothesize that \emph{the top eigenspaces of their preconditioners and the Hessian are highly aligned.  } 

\begin{figure}[H]
    \centering
    \includegraphics[width=1\linewidth]{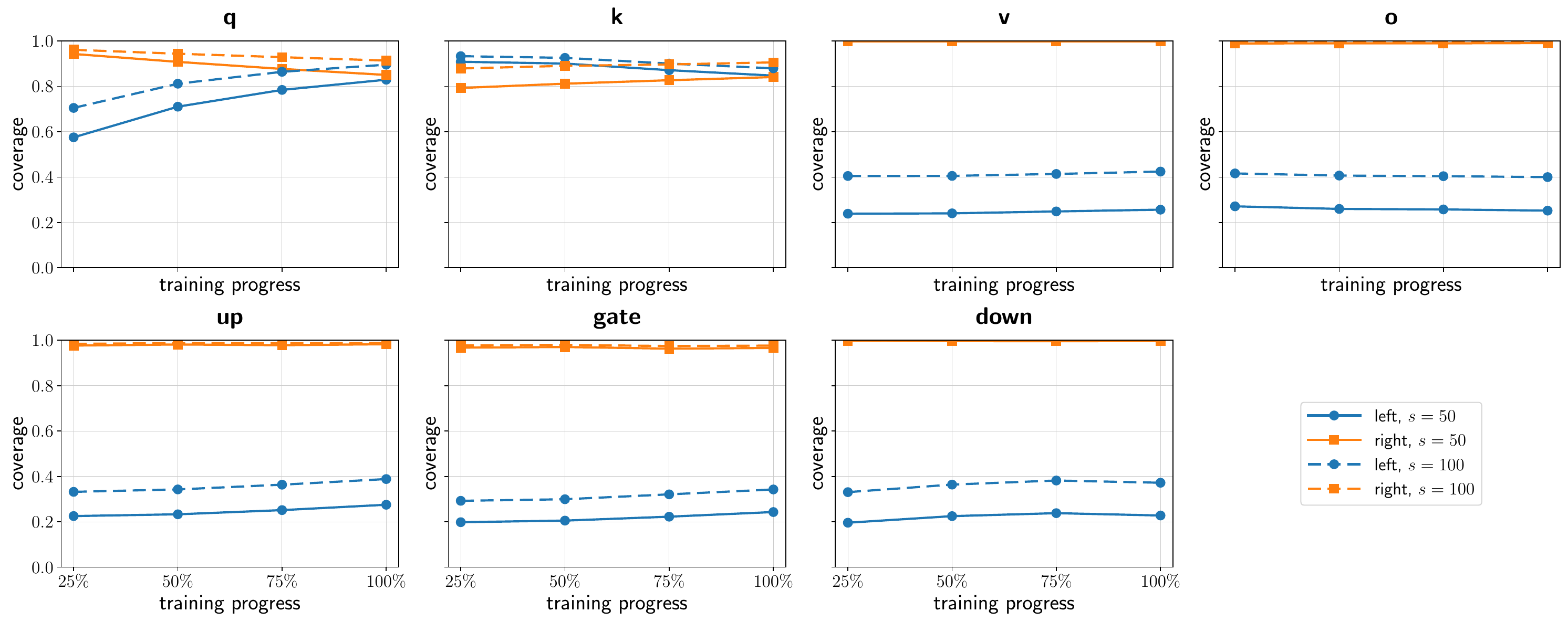}
    \caption{Evolution of the coverage score (defined in \eqref{eq:CoverageScore}) of each block's top-\(k\) Hessian eigenspace by the left and right top-\(s\) gradient singular subspaces in layer 2 of a 120M LLaMA model (\(k=200\), \(s\in\{50,100\}\)). The model is trained for 20,000 steps on approximately 10B tokens. Scores closer to 1 indicate greater containment.}
    \label{fig:top_subspace_align_120m}
\end{figure}

To verify this hypothesis, we conduct experiments on a 120M LLaMA model with \(d_{\text{model}}=768\) (see Appendix~\ref{appen:exp-details-align} for details). For each weight matrix \(W\in\mathbb{R}^{m\times n}\) in the transformer layers (q, k, v, o, up, gate, down), we assess whether the top-\(k\) subspace of the Hessian \(H_W\in\mathbb{R}^{mn\times mn}\) can be well covered by the subspaces spanned by the top-\(s\) eigenspaces of \(GG^\top\) and \(G^\top G\) (with bases denoted by \(B_{\text{left}}\) and \(B_{\text{right}}\), respectively), where \(G\) is the stochastic gradient matrix. The two subspaces have the following bases:
\[
B_{\text{left}}=U_s\otimes I_n\in\mathbb{R}^{mn\times sn},\qquad
B_{\text{right}}=I_m\otimes V_s\in\mathbb{R}^{mn\times ms},
\]
where \(U_s\in\mathbb{R}^{m\times s}\) and \(V_s\in\mathbb{R}^{n\times s}\) are the top-\(s\) eigenvectors of \(GG^\top\) and \(G^\top G\), respectively \footnote{Here we use row-major vectorization.}. 

Figure~\ref{fig:top_subspace_align_120m} shows that throughout the whole training process, \textbf{the top eigenspaces of the block-wise Hessian can be well covered   by the subspace spanned by $B_{\text{right}}$} (corresponding to the right singular space of \(G\)). The top-\(s\) subspace of \(GG^\top\) (i.e., the left singular space) also shows good coverage in some blocks, but to a lesser degree than the right subspace. Since the preconditioners used in Muon and SOAP are closely tied to \(G^\top G\) and \(GG^\top\), we propose the following tailored implementation strategies.

\paragraph{Accelerating Muon ({\color{violet}\oursmuon})}
At iteration $k$, let $G_k \in \mathbb{R}^{m\times n}$ ($m \ge n$) represent the stochastic gradient for a given parameter block. Denote the momentum update  $M_{k} = (1-\alpha) M_{k-1} + G_k$ and the Nesterov correction $\widetilde{M}_{k} = M_{k} + G_k/(1-\alpha)$~\footnote{We adopt the Nesterov formulation consistent with~\cite{jordan2024muon}.}. Recognizing that $\widetilde{M}_{k}^\top \widetilde{M}_{k}$ acts as a stabilized proxy for $ G_k^\top G_k$, we utilize it to characterize the sharp directions according the previous discussion. Specifically, we construct the projection matrix $P_k \in \mathbb{R}^{n\times n}$ onto the top-$d_s$ eigenspace of $\widetilde{M}_{k}^\top \widetilde{M}_{k}$, where $d_s n$ specifies the dimension of the sharp subspace. A numerically efficient, optimizer state-free method for estimating $P_k$ is detailed in Appendix~\ref{appen:details-oursmuon}. Consequently, the update direction is formulated as:
\begin{equation}
\begin{aligned}
U_k=&\widetilde{M}_{k}(\widetilde{M}_{k}^\top \widetilde{M}_{k})^{-\frac{1}{2}}(P_k+ \chi (I_n-P_k))
+  G_k (G_k^\top G_k)^{-\frac{1}{2}}(\beta_1P_k+\chi\beta_2 (I_n-P_k)),
\end{aligned}
\end{equation}
where $\beta_2 \ge \beta_1$ are the Hessian damping coefficients, and $\chi \ge 1$ denotes the amplification ratio of the learning rate in flat directions relative to sharp ones. Setting $\chi=1$ and $\beta_{1,2}=0$ recovers the original Muon baseline. Here we use the precondition $ (G_k^\top G_k)^{-1/2}$ instead of $ (\widetilde{M}_{k}^\top \widetilde{M}_{k})^{-1/2}$ to the gradient term, since it has  superior numerical stability and efficiency of applying NS iterations, and $G_k(G_k^\top G_k)^{-1/2}$  serves as a valid instantaneous estimator of gradient whitening (orthogonalization). In practice, $\chi$ typically requires eventual decay to realize its acceleration effect during training (see details in Appendix~\ref{app:hyper-params-schedule}).

\paragraph{Accelerating SOAP  ({\color{violet}\ourssoap})} 
We retain the notations established above. The SOAP preconditioner explicitly provides the eigenvalues ($V$) and eigenvectors ($Q_r \otimes Q_l$) of the curvature estimation (\cref{tab:precond_form}). Thus we directly identify the sharp directions  in the ($Q_r \otimes Q_l$)-rotated eigenspace by selecting the top-$d_s$ ($d_s\le mn$) elements of $V$.
Let $P_k \in \mathbb{R}^{m\times n}$ denote the mask matrix indicating the sharp directions at iteration $k$, and let $Q_k = 1_m 1_n^\top - P_k$ represent the flat directions. The update direction for accelerating SOAP is given by:
\begin{equation}
\begin{aligned}
U_k = & Q_l\left[V_k^{-\frac{1}{2}} \odot (\beta_1 P_k + \chi \beta_2 Q_k) \odot (Q_l^\top G_k Q_r) 
+ V_k^{-\frac{1}{2}} \odot (P_k + \chi Q_k) \odot (Q_l^\top M_k Q_r)\right] Q_r^\top,
\end{aligned}\nonumber
\end{equation}
where, similar to \oursmuon, the SOAP baseline is recovered when $\chi=1$ and $\beta_{1,2}=0$.
In practice, to enhance training stability, we implement $P_k$ using a soft transition scheme (smoothly changing from 0 to 1) rather than a hard binary cutoff. See Appendix~\ref{appen:details-ourssoap} for implementation details.

\section{Experiments}\label{experiments}

We evaluate \ours in accelerating Muon and SOAP on a wide range of  LLM pre-training scenarios. Experimental details are provided in  Appendices \ref{appen:alg-details} and \ref{appen:exp-details}.

\subsection{Results on Dense Models}\label{result:dense}

\begin{figure*}[!htb]
    \centering
    \includegraphics[width=0.24\linewidth]{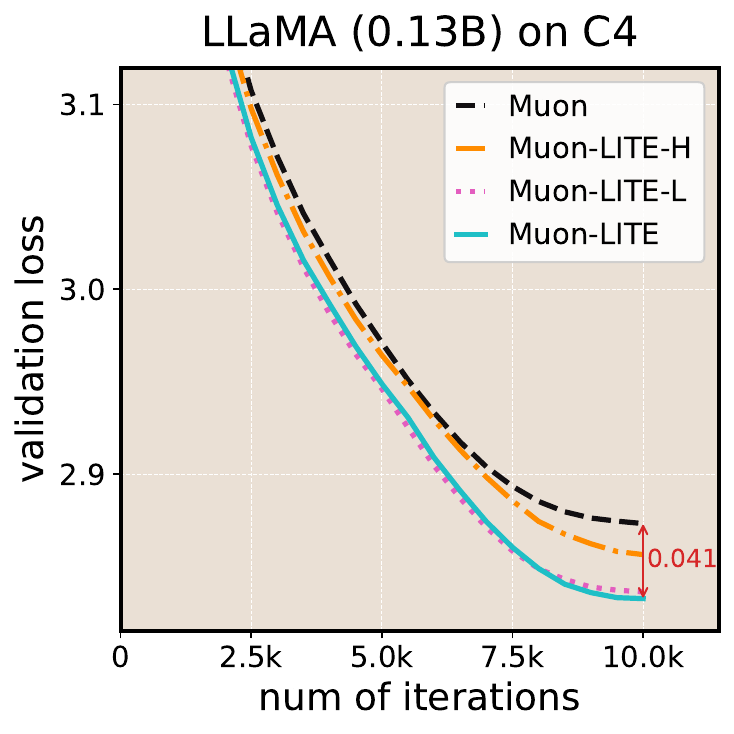}
    \includegraphics[width=0.24\linewidth]{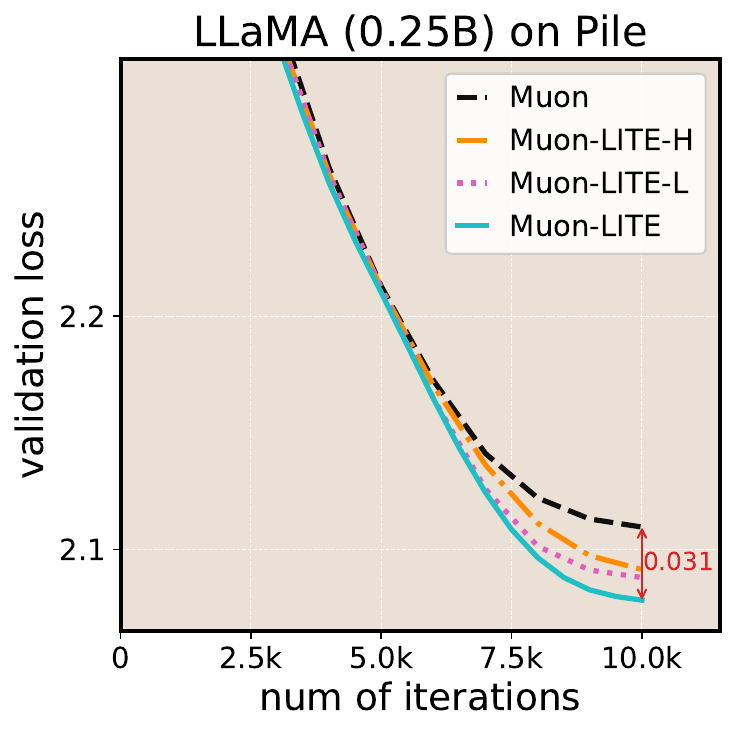}
    \includegraphics[width=0.24\linewidth]{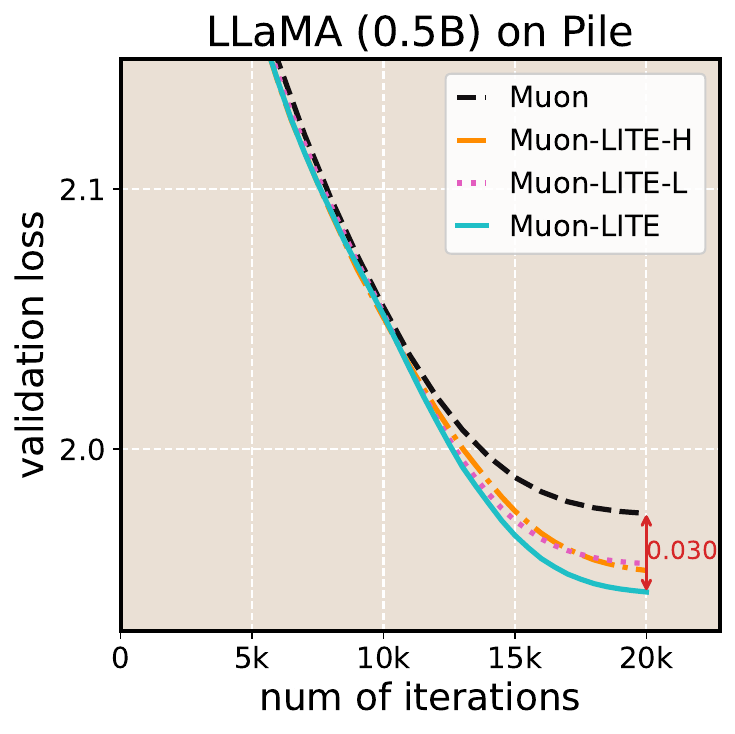}
    \includegraphics[width=0.24\linewidth]{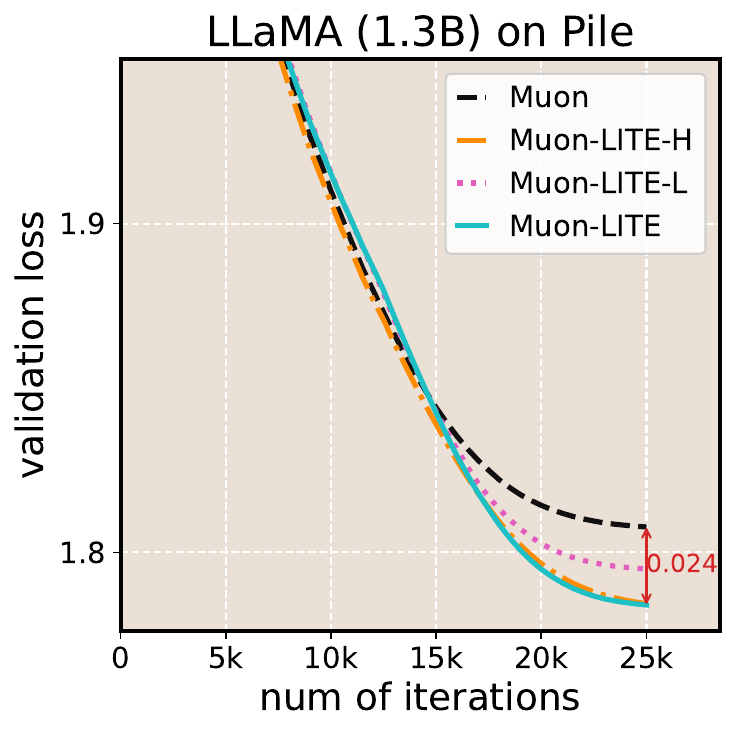}

    \vspace{-.1cm}
    \caption{Performance comparison of Muon   and \oursmuon on LLaMA2 pre-training tasks with a \texttt{cos} learning rate schedule. The suffixes L and H denote two ablation variants of LITE:  in flat directions, L increases only the learning rate ratio $\chi\ge1$, while H increases only the Hessian damping coefficient $\beta_2$. The token batch size is approximately 2M for the 1.3B model, and 1M for others.}
    \label{fig:cosine-llama-muon}
    \vspace{-.37cm}
\end{figure*}

\paragraph{Main Results.} 
\cref{fig:cosine-llama-muon,fig:cosine-llama-soap} compare the performance of \oursmuon and \ourssoap against vanilla Muon and SOAP across a  range of LLaMA model sizes on both the C4 and Pile datasets. The values of  $\{\chi,\beta_{1,2}\}$ in \ours  were determined via a search on the  0.25B model experiments and were then \textbf{used uniformly for  other experiments.} Across all experimental settings, \oursmuon and \ourssoap consistently achieve lower terminal losses than their respective well-tuned baselines. Besides, \ours outperforms both \ours-L and \ours-H, validating the effectiveness of simultaneously increasing $\chi$ and $\beta_2$. We visualize the scaling laws of \oursmuon versus Muon in \cref{fig:scaling_law} (right). The performance gains yielded by \oursmuon remain consistent across a wide range  of model scales, indicating its potential  scalability. Further downstream analyses     are provided in Appendix \ref{app:add-experiment}.

We subsequently investigate the acceleration efficacy of \ours over extended training horizons, increasing the training token budget from   $40\times$ the parameter count to $100\times$  and $200\times$. \cref{fig:scaling_law} (left, middle) indicates that the loss reduction gains yielded by \oursmuon persist across these extended regimes, maintaining a significant speedup over Muon in longer training horizons.  This demonstrates the sustained scalability of \oursmuon with respect to the number of training tokens.

\paragraph{Ablation Studies and Hyper-parameter Robustness.} We demonstrate that the acceleration strategy for flat directions is detrimental when applied to sharp directions, and that LITE is generally robust to hyper-parameter tuning, yielding notable acceleration across a wide range of configurations $(\chi, \beta_{1,2}, d_s)$. See Appendix \ref{abl-sen} for more details.

\subsection{Results on MoE Models}

\begin{wrapfigure}{l}{0.26\textwidth}
\vspace{-5mm}
    \centering
\includegraphics[width=\linewidth]{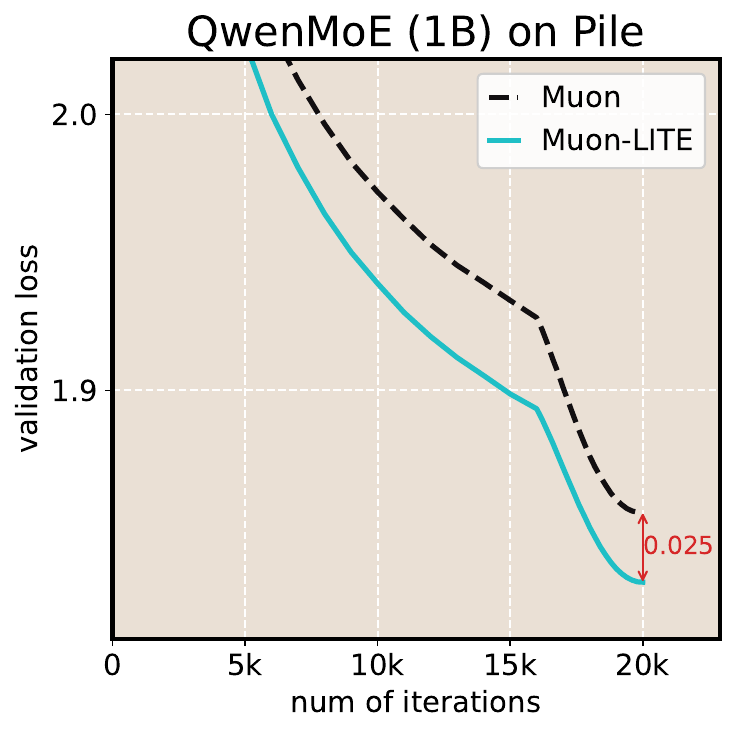}
\vspace{-3.5mm} 
    \caption{Results in QwenMoE pre-training tasks.}
    \label{fig:qwen2-1b}
    \vspace{-5mm} 
\end{wrapfigure}

Mixture-of-experts (MoE) architectures have emerged as a critical  paradigm for modern LLMs. By selectively activating   a subset of experts for each token, MoE models achieve superior scaling efficiency over their dense counterparts.

We  evaluate \oursmuon against the Muon baseline on a 1B-parameter QwenMoE model, using a \texttt{wsd} learning rate schedule. \cref{fig:qwen2-1b} demonstrates that \oursmuon consistently outperforms Muon throughout the training process, ultimately achieving a significantly lower terminal loss.  Furthermore, the performance gap between \oursmuon and Muon progressively widens as iterations proceed during the stable phase. This trend suggests that the efficiency gains offered by \oursmuon are likely to persist or even amplify at larger data scales.

\section{Theoretical  Analysis of  Training Dynamics}\label{sec:theory}

In this section, we provide a theoretical characterization of the acceleration mechanism of \ours. We adopt and extend the \emph{River-Valley} landscape framework  in  \cite{wen2025understanding}, which models the pre-training loss landscape as a deep valley with a flat   river  at its bottom. We {\em generalize this framework} to high-dimensional settings, and {\em analyze the   dynamics} \eqref{2-ode01} corresponding to  \ours (\cref{alg:LITE}). Our analysis demonstrates that \ours effectively accelerates the training dynamics along the River, which  dictates the terminal loss. Below, we introduce the assumptions for our theoretical derivation.

\begin{assumption}[River Structure]\label{ass:river}
Fix an integer $d_s \in \mathbb{Z}_+$ and an open set $U \subset \mathbb{R}^p$. Let $P_s(w)$ denote the Euclidean projection onto the top $d_s$ eigenspace of the Hessian $\nabla^2 f(w)$. We assume that $P_s(w)$ varies smoothly with respect to $w$, and that the set (referred to as  \emph{River})
$\mathcal{R} = \{ w \in U : P_s(w) \nabla f(w) = 0 
\}$ 
constitutes a $(p-d_s)$-dimensional manifold. Furthermore, we assume that within the sharp subspace $\text{Range}(P_s(w))$, the Hessian $\nabla^2 f(w)$ and the preconditioner $F(w)$ share a common eigen-basis. Finally,  the ODE 
$\frac{d}{dt} \psi_t(w) = -P_s(\psi_t(w)) F^{-1}(\psi_t(w)) \nabla f(\psi_t(w))$ 
initialized at any $\psi_0(w) = w \in U$ is assumed to remain and converge   within $U$ as $t \to \infty$. 
\end{assumption}
Under Assumption \ref{ass:river}, we define the ``projection'' map $\Phi: U \to \mathcal{R}$ to River via the limit: $\Phi(w) \coloneqq \lim_{t\to\infty}\psi_t(w)$. This ensures $\Phi(w) \in \mathcal{R}$ since the stationarity of the flow implies $P_s \nabla f = 0$.

\begin{assumption}[Regularity, Smoothness and Conservation] \label{ass:rsc}
Based on Assumption \ref{ass:river}, we assume:
\begin{enumerate}
[leftmargin=1em]
\vspace{-3mm}
\item \label{ass:smoothness}$f$ is analytic in $U$. Let  $(\cdot)|_s$ denote the restriction of an operator to the sharp subspace $\operatorname{Range}(P_s(w))$. For all $w \in U$, the eigenvalues of $\nabla^2 f(w)$, $F(w)$, $F(w)|_s$,  $(F^{-\frac{1}{2}}(w)\nabla^2 f(w) F^{-\frac{1}{2}}(w))|_s$  lie   in $[-L, L]$,     $[\lambda_{F}, \rho]$, $[\lambda_{F_s}, \rho]$ and $[\lambda_{H_F,s}, L_{H_F,s}]$, respectively.

\item  \label{ass:spectral}
There exist  constants $G,\varepsilon \ge 0$, such that for any $w \in U$, $\|\nabla f(w)\|_{F(w)^{-1}}\le G$, $\|\nabla P_s(w)\|_{\mathrm{op}}  \leq  {\varepsilon}/{G}$ and $\|\nabla F(w)\|_{\mathrm{op}}  \leq \varepsilon/G$. 
See \eqref{nabla-Ps} for detailed definition of the operator norm $\|\cdot\|_{\mathrm{op}}$.

\item \label{ass:Conservation}
    There exists a time horizon $T_{\text{max}} > 0$  depending on $(\alpha, \beta_1, \{\eta_t\}_{t\ge0}, \chi, \beta_2)$, such that for any initialization $w_0 \in U$ and  $m_0=0$, the trajectory $w_t$ of \eqref{2-ode01} remains strictly in $U$ for all $t \in [0, T_{\text{max}}]$.
    
\end{enumerate}

\end{assumption}

\begin{theorem}\label{thm:stage-1}
Suppose Assumptions \ref{ass:river}, \ref{ass:rsc} hold and $\dot{\eta}_t \le 0$. Define  $\iota_t = \min\{\alpha/2, \lambda_{H_{F,s}}\beta_1\eta_t\}$ and   the forgetting kernel  $K(t) = \exp\left( - \int_0^t \iota_s ds \right)$. 
There exists a constant $\varepsilon_0 > 0$ dependending on $f$ and hyper-parameters in \eqref{2-ode01} such that for all $\varepsilon \le \varepsilon_0$, the trajectory $w_t$ governed by \eqref{2-ode01} satisfies the attraction bound $\|w_t-\Phi(w_t)\|_2^2\lesssim K(t)+\varepsilon$, 
and the ``projected" trajectory $z_t = \Phi(w_t)$ follows the flat-direction dynamics with   $s_0=0$ and a perturbation $\epsilon_{z,t}$:
\begin{align}\label{dyn-z}
\begin{dcases*}
     \dot{z}_t  =  - \eta_t  \left({\color{violet}\chi} P_\mathcal{R}(z_t)F(z_t)^{-1} (s_t + {\color{violet}\beta_2} \nabla f(z_t))+ \epsilon_{z,t}\right)  , \\
      \dot{s}_t=-\alpha s_t +\nabla f(z_t). 
\end{dcases*}
\end{align}
Finally, defining $b_t =  \| P_\mathcal{R} (z_t)F(z_t)^{-1}\nabla\! f(z_t)\|_{F(z_t)}^2  +  \| P_\mathcal{R}(z_t)F(z_t)^{-1}s_t\|_{F(z_t)}^2$, the perturbation error satisfies $
\|\epsilon_{z,t}\|_{F(z_t)}^2\lesssim  K(t) +\varepsilon^2 b_t
     + \varepsilon  \int_0^t K(\tau-t)b_\tau d\tau$  .
\end{theorem}
The proof is in Lemmas \ref{app:thm1-0} and \ref{app:thm1-1}. We treat $\varepsilon$ as a small constant  to enforce a slowly varying  $\mathcal{R}$.  
Inspired by \cite{Attouch2022}, we can take $\eta_t\equiv\eta$, $\alpha=2\lambda_{F_{H,s}}^{1/2}\eta $, $\beta_1=\lambda_{F_{H,s}}^{-1/2}$ to  achieve a fast decay $K(t)=\exp( -\eta \lambda_{F_{H,s}}^{1/2}t)$ (\ie fast convergence on sharp directions). 
In this case, Theorem \ref{thm:stage-1} indicates that the dynamics \eqref{2-ode01} first converge to the $\mathcal{O}(\varepsilon^{1/2})$ neighborhood of $\mathcal{R}$ after a $\Theta(\eta^{-1}\lambda_{F_{H,s}}^{-1/2}\log(\varepsilon^{-1}))$ interval, then faithfully track the decoupled acceleration dynamics on  $\mathcal{R}$  closely,  where the tracking error exhibits a fading memory property, driven predominantly by the immediate history of the underlying dynamics. Next we clarify  how ${\color{violet}\chi}$ and ${\color{violet}\beta_2}$ influence the convergence of the decoupled dynamics in $\mathcal{R}$, where the detailed proof is in Lemma \ref{app:thm2-0}.  
\begin{theorem}\label{thm:river-rate}
 Suppose Assumptions \ref{ass:river},  \ref{ass:rsc} hold and $\dot{\eta}_t\le0$. Let  $\Delta_f=f(z_0)-\inf f$. For the unperturbed dynamics  in \eqref{dyn-z} with $\epsilon_{z,t}\equiv0$, there exists a constant $\varepsilon_1$ (depending on $f$ and hyper-parameters in \eqref{2-ode01}), such that for any  $\varepsilon\le \varepsilon_1 $ and any $T\le T_{\mathrm{max}}$,  the trajectory $z_t $ satisfies 
 \begin{align*} 
    &\frac{1}{T}\int_0^T\eta_t \|P_\mathcal{R}(z_t) F(z_t)^{-1}\nabla f(z_t)\|_{F(z_t)}^2 dt\le \frac{2\Delta_f}{{\color{violet}\chi\beta_2} T},
    \,\,\,\text{and} \,\,\, \frac{1}{T}\int_0^T\eta_t \|s_t\|_{F(z_t)}^2 dt \le \frac{2\Delta_f}{{\color{violet}\chi}\alpha T}.
     \end{align*} 
\end{theorem}
   \vspace{-1mm}
 \cref{thm:river-rate} indicates that increasing both ${\color{violet}\chi \ge 1}$ and ${\color{violet}\beta_2}$ yields a tighter upper bound for the Riemannian gradient on  $\mathcal{R}$ ($P_\mathcal{R}(z_t) F(z_t)^{-1}\nabla f(z_t)$)  and the momentum $s_t$, thereby theoretically justifying the design motivation behind \ours.

\section{Conclusion}
We propose \ours, a generalized strategy   to accelerate LLM pre-training by enhancing flat-direction training dynamics.  \ours  enhances state-of-the-art optimizers such as Muon and SOAP, achieving significantly lower terminal loss and exhibiting superior scaling behavior across various LLM pre-training tasks. Several limitations remain in this work. \ours introduces extra hyper-parameters that require a preliminary search on a proxy model, and the binary partition of sharp and flat directions leaves room for finer-grained precision.  For future research, it is promising to explore dynamic hyper-parameter adaptation mechanisms in \ours, and to  integrate it with other emerging optimizers. 

\newpage
\section{Acknowledgment}
This work is funded by the National Key Research and Development Program of China (No. 2024YFA1012902), and the National Natural Science Foundation of China (No. 12288101, W2441021).  This research is also supported by AI for Science Institute, Beijing, China. Part of the computational resources were supported by the Centerfor Intelligent Computing and Song-Shan Lake HPC Cen-ter (SSL-HPC) in Great Bay University, Dongguan, China.

{
\small

\bibliography{ref}
\bibliographystyle{abbrv}

}
\newpage
\resettocdepth

\newpage
\appendix
\begin{center}
\LARGE
    \textbf{Appendix}
\end{center}

\vspace{5mm}
\tableofcontents

\vspace{5mm}
\newpage

\section{More Related Works}\label{app:related_works}

\subsection{Comparison with methods enhancing dynamics along flat directions.}
Recent strategies accelerate LLM pre-training by assigning larger learning rates to flat components. \cite{wang2024improving} adjust learning rates based on an element-wise diagonal Fisher approximation, while \cite{wang2025the} apply block-wise learning rates to blocks identified as flat via averaged sharpness. \cite{zhou2025bsfaleveragingsubspacedichotomy} further employs  a matrix-level precision subspace  detected via a lazy PCA-based estimator.
Compared to these methods, \ours identifies flat subspaces with matrix-level precision using highly efficient numerical techniques that introduce no extra optimizer states.
Crucially, prior studies focus exclusively on learning rate adjustment, a strategy \ours subsumes by simply setting $\chi \ge 1$. In contrast, we \textbf{also address the dynamics of momentum} in anisotropic landscapes. We show how momentum and preconditioners jointly influence training dynamics and \textbf{enhance momentum evolution} in flat directions by explicitly increasing Hessian damping.

\paragraph{Empirical comparison with existing flat-direction enhancements.} 
We empirically compared \ours with the aforementioned flat-direction enhancement methods \cite{wang2024improving, wang2025the, zhou2025bsfaleveragingsubspacedichotomy}. Because these prior methods were primarily evaluated using AdamW rather than matrix-based optimizers (like Muon and SOAP), we implemented them across both AdamW and Muon baselines for a comprehensive comparison. 

Specifically, we applied the block-wise learning rate strategy proposed by \cite{wang2025the} to our 250M LLaMA training setup (matching the experimental setting in Figure \ref{fig:cosine-llama-muon}). Following their proposed  sharpness principle and block-wise learning rate searching ranges, we searched for optimal hyperparameter ratios across the Embedding, QK, VO, and FFN blocks. As shown in Table \ref{tab:blk-lr}, aggressively increasing the learning rates for the QK, VO, and FFN blocks actually degrades performance (increasing the loss). We hypothesize that this occurs because block-wise scaling is a  coarse-grained  operation: it indiscriminately amplifies both flat and sharp directions within each individual block, thereby exacerbating unstable oscillations along the sharp directions. Consequently, the best-performing configuration for this baseline achieves substantially weaker acceleration compared to \oursmuon, which reaches a final validation loss of 2.078.

\begin{table}[!ht]
    \centering
    \begin{tabular}{|c|c|c|c|c|}
        \hline
        \diagbox{optimizer}{block wise lr}& $(6,4,3,2)$ & $(4,2,2,2)$ & $(4,1,1,1)$ & $(1,1,1,1)$ (baseline) \\
        \hline
       AdamW & 6.540  & 2.295 & \textbf{2.121} & 2.129 \\
        \hline
       Muon & 3.673 & 2.212  & \textbf{2.105} & 2.109 \\
        \hline
    \end{tabular}
    \vspace{3mm}
    \caption{Results of Muon and AdamW using block wise learning rate on 250M LLaMA model trained on Pile. The optimal baseline configuration for AdamW, determined via hyperparameter search, is learning rate (lr)=0.003, $(\beta_1,\beta_2)=(0.9,0.95)$, and weight decay=0.1.  $(a,b,c,d)$ means the learning rate for Embedding, QK, VO, FFN, RMSNORM blocks are $(a:b:c:d:1)$, where 1 means the base learning rate. The best results are bold.}\label{tab:blk-lr}
\end{table}

Regarding the methods proposed by \cite{wang2024improving} and \cite{zhou2025bsfaleveragingsubspacedichotomy}, we found them difficult to scale to our 250M large-batch pre-training regime due to prohibitive memory overheads (notably, both prior works primarily verified their acceleration on models smaller than 230M parameters with smaller batch sizes). Specifically, \cite{wang2024improving} requires sampling from logits-generated distributions, which consumes over 31GB of additional Video Random Access Memory (VRAM) in our setup. Furthermore, its reliance on an Adam-style diagonal Fisher estimation to identify sharp directions is structurally incompatible with matrix-based optimizers like Muon. Similarly, \cite{zhou2025bsfaleveragingsubspacedichotomy} requires caching massive optimizer states (\eg, storing approximately 60 times the parameter count of the transformer matrices in their original LLaMA experiments) resulting in Out-of-Memory  in our setting. In contrast, \ours  introduces negligible memory overhead while seamlessly integrating with modern matrix-based optimizers.

\subsection{Comparison with algorithmic frameworks for adaptive optimizers.}
Recent works~\cite{pethick2025training,bernstein2024oldoptimizernewnorm,xie2026controlledllmtrainingspectral} interpret the update step of certain adaptive optimizers as the solution to a constrained optimization problem governed by operator norms:
\begin{align}
U^* = \operatorname*{argmin}_{\|U\|_{\operatorname{op}} \le 1} \langle U, M \rangle,
\end{align}
where $M$ denotes the momentum at the current iterate, $\langle \cdot, \cdot \rangle$ is the Frobenius inner product, and $\|\cdot\|_{\operatorname{op}}$ represents a specific operator norm. This formulation formally unifies methods like Muon and Lion, offering insights into their stability (e.g., by controlling the magnitude of activations during propagation). Since the geometry induced by this formulation is generally independent of the local curvature at the current iterate, its scope does not easily extend to methods like AdamW and SOAP, which utilize iterate-dependent preconditions and decouple momentum and preconditioner updates. It also leaves room for further explanations  regarding the algorithmic design of accumulating momentum before preconditioning.  Offering a different lens, our proposed Riemannian ODE framework unifies a broader class of adaptive optimizers. It uncovers the joint role of momentum and preconditioners in shaping training dynamics and extends naturally to complex momentum schemes (see Appendix~\ref{section-ademamix}).

\section{Algorithm Details}\label{appen:alg-details}

\subsection{Details of \oursmuon}\label{appen:details-oursmuon}
For Muon and \oursmuon, we adopt the Newton-Schulz (NS) iteration  termed Polar Express in \cite{amsel2025polarexpressoptimalmatrix}, and set   
Newton–Schulz iterations to 6.    This approach yields faster convergence to the theoretical optimum than the method in \cite{jordan2024muon}.

For the \texttt{embedding, norm, output} blocks, we use the same methods as in \ourssoap with $Q_l=I,Q_r=I$ to estimate the projection $P_k$ to sharp directions. See more details in \cref{appen:details-ourssoap}. Next we  introduce how to $P_k$ for other blocks using \oursmuon.

\paragraph{Estimating $P_k$ via Efficient Composite Newton-Schulz.} 
To efficiently estimate the sharp subspace projection $P_k$ at iteration $k$, we employ a lightweight composite Newton-Schulz (NS) scheme. The process involves two main steps: constructing a filtering operator and dynamically adjusting the threshold. For any matrix $A$ with  Singular Value Decomposition (SVD) $A=U \Sigma V^\top$, we define $\operatorname{NS}(A)=A(A^\top A)^{-\frac{1}{2}}=UV^\top$, where the inverse   denotes the Moore-Penrose pseudo-inverse.

\begin{enumerate}
    \item 
\textbf{Construction of the Filtering Operator.}
We first define a dynamic threshold $\tau_k = l_k\|\widetilde{M}_{k}\|_F$, where $l_k$ is a scaling factor. This threshold $\tau_k$ serves to delineate the sharp subspace: specifically, we select the eigenspace of $\widetilde{M}_{k}^\top \widetilde{M}_{k}$ associated with eigenvalues larger than $\tau_k^2$ (singular values of $\widetilde{M}_{k}$ larger than $\tau_k$). The adaptive update mechanism for $l_k$ (and thus $\tau_k$) will be detailed subsequently.  
We then compute the filtering operator $T_k \in \mathbb{R}^{m\times n}$ as:
\begin{equation}
    T_k = \frac{1}{2}\operatorname{NS}(\widetilde{M}_{k}) + \frac{1}{2}\operatorname{NS}\left( \frac{\widetilde{M}_{k}}{\tau_k} - \operatorname{NS}(\widetilde{M}_{k}) \right).
\end{equation}

The projection matrix onto the sharp subspace is subsequently obtained by $P_k = T_k^\top T_k$,  the rationale for which is detailed in the subsequent parts (\eqref{ns-res} and \eqref{ns-res-t}).  It is worth noting that this method for estimating $P_k$ necessitates relatively high precision from the  NS  function. Consequently, we employ the Polar Express method \cite{amsel2025polarexpressoptimalmatrix} rather than the vanilla NS iterations used in \cite{jordan2024muon}.

\item \textbf{ Dynamic Sharp Subspace Dimension   Adjustment.}
To ensure the rank of the projection $P_k$ approximates the target dimension $d_s$, we update the scaling factor $l_k$ based on the effective rank (measured by $\|P_k\|_F^2$):
\begin{align}
    l_{k+1} = 
    \begin{cases}
        1.05 \cdot l_{k}, & \text{if } \|P_k\|_F \ge \sqrt{d_s}, \\
        0.95 \cdot l_{k}, & \text{if } \|P_k\|_F < \sqrt{d_s}.
    \end{cases}
\end{align}
This feedback loop effectively maintains $\tau_k$ near the $d_s$-th largest singular value of $\widetilde{M}_{k}$.
\end{enumerate}

\textbf{Theoretical Intuition.}
To understand why $T_k$ acts as a subspace filter, consider the SVD of $\widetilde{M}_{k} = U\Sigma V^\top$. Let the singular values be ordered as $\sigma_1 \ge \dots \ge \sigma_j > \tau_k > \sigma_{j+1} \ge \dots \ge 0$.
Since the Newton-Schulz iteration converges to the matrix sign function, we have $\operatorname{NS}(\widetilde{M}_{k}) = UV^\top$. Consequently, the input to the second NS term becomes:
\begin{equation}\label{ns-res}
\frac{\widetilde{M}_{k}}{\tau_k} - \operatorname{NS}(\widetilde{M}_{k})
= U \operatorname{diag}\left( \frac{\sigma_1}{\tau_k}-1, \dots, \frac{\sigma_r}{\tau_k}-1 \right) V^\top.
\end{equation}
Applying $\operatorname{NS}(\cdot)$ to \eqref{ns-res} maps positive entries (where $\sigma_i > \tau_k$) to $+1$ and negative entries (where $\sigma_i < \tau_k$) to $-1$. Therefore, we have \begin{equation}\label{ns-res-t}
    T_k = \frac{1}{2}\operatorname{NS}(\widetilde{M}_{k}) + \frac{1}{2}\operatorname{NS}\left( \frac{\widetilde{M}_{k}}{\tau_k} - \operatorname{NS}(\widetilde{M}_{k}) \right) = U \operatorname{diag}(1, \dots, 1, 0, \dots, 0) V^\top .
\end{equation}

\textbf{Computational Efficiency of \oursmuon.} Overall, \oursmuon requires only two additional NS functions per step compared to vanilla Muon, which can be executed in parallel.  Since NS iterations constitute a nearly negligible fraction of the computational cost in large-batch pre-training (dominated by gradient computation), this slight overhead is virtually imperceptible in practice.  Table \ref{tab:wall_clock_times} reports the average per-iteration wall-clock times of Muon and Muon-LITE across model sizes from 0.25B to 1.3B, with torch.compile enabled for the model and Newton--Schulz iterations sharded across GPUs. The results indicate that Muon-\ours introduces only a modest time overhead. Since the extra overhead of LITE is independent of batch size, it diminishes with larger batches, as gradient computation will occupy significantly more time, making this benefit particularly relevant for large-scale industrial training with token batch sizes of tens of millions or more.  
For example, training a LLaMA-1.3B model on the Pile dataset (global batch size 8192, micro batch size 16, sequence length 1024, on $8\times$ A800-80GB GPUs) yields a throughput of 100.4k tokens/s for \oursmuon versus 101.5k tokens/s for Muon. This reflects a throughput drop of only $\approx 1\%$. Furthermore, efficiency can be enhanced via system-level optimizations, such as kernel fusion for the NS iterations.

\begin{table}[H]
    \centering
    \caption{Average per-iteration wall-clock times (seconds) of Muon and Muon-LITE across model sizes.}
    \label{tab:wall_clock_times}
    \begin{tabular}{lccc}
        \hline
        \textbf{Model} & \textbf{0.25B} & \textbf{0.5B} & \textbf{1.3B} \\
        \hline
        Muon & 0.957 & 1.382 & 4.908 \\
        Muon-LITE & 0.989 (+3.3\%) & 1.410 (+2.0\%) & 4.980 (+1.4\%) \\
        \hline
    \end{tabular}
\end{table}

\subsection{Details of \ourssoap}\label{appen:details-ourssoap}

In \ourssoap, we employ a dynamic thresholding strategy  generating a smoothed projection onto the sharp directions to enhance training stability. To avoid the computational latency and memory overhead associated with exact \texttt{top\_k} operations in PyTorch, especially for large matrices, we approximate the top-$k$ thresholds via adaptive scalar variables, similar to the approach in \oursmuon.
We define $d_s$ as the target dimension of the sharp subspace, and $d_{\text{smooth}}$ as the dimension of the transitional subspace bridging the sharp and flat directions. We maintain two dynamic scalars, $l_k^s$ (controlling the sharp boundary) and $l_k^{\text{smooth}}$ (controlling the smooth boundary), initialized as $l_0^s=1.0$ and $l_0^{\text{smooth}}=0.5$.

At the $k$-th iteration, we estimate the thresholds for the top-$d_s$ and top-$(d_s+d_{\text{smooth}})$ elements based on the mean magnitude of the tensor $V_k$: $\tau_k^s = l_k^s \operatorname{mean}(V_k)$ and $\tau_k^{\text{smooth}} = l_k^{\text{smooth}} \operatorname{mean}(V_k)$. The smoothed projection matrix $P_k \in \mathbb{R}^{m\times n}$ is then constructed element-wise as:
\begin{equation}
    (P_k)_{i,j} =
    \begin{cases}
        1, & \text{if } (V_k)_{i,j} \ge \tau_k^s, \\[6pt]
        \displaystyle \frac{(V_k)_{i,j} - \tau_k^{\text{smooth}}}{\tau_k^{s} - \tau_k^{\text{smooth}}}, & \text{if } \tau_k^{\text{smooth}} \le (V_k)_{i,j} < \tau_k^s, \\[6pt]
        0, & \text{if } (V_k)_{i,j} \le \tau_k^{\text{smooth}}.
    \end{cases}
    \label{eq:threshold_smooth}
\end{equation}
Then we dynamically adjust the coefficients $l_k^s$ and $l_k^{\text{smooth}}$ based on the current sparsity levels:

\begin{align}
    l_{k+1}^s = 
    \begin{cases}
        1.05 \cdot l_k^s, & |\{V_{i,j} : V_{ij} > \tau_k^s\}| \ge d_s, \\ 
        0.95 \cdot l_k^s, & |\{V_{i,j} : V_{i,j} > \tau_k^s\}| < d_s,
    \end{cases}
\end{align}
and
\begin{align}
    l_{k+1}^{\text{smooth}} = 
    \begin{cases}
        1.05 \cdot l_k^{\text{smooth}}, & |\{V_{i,j} : V_{ij} > \tau_k^{\text{smooth}}\}| \ge d_s+d_{\text{smooth}}, \\ 
        0.95 \cdot l_k^{\text{smooth}}, & |\{V_{i,j} : V_{i,j} > \tau_k^{\text{smooth}}\}| < d_s+d_{\text{smooth}},
    \end{cases}
\end{align}
where $|\cdot|$ denotes the cardinality of the set. Finally, we enforce the constraint $l_{k+1}^{\text{smooth}} \leftarrow \min\{0.95 l_{k+1}^s, l_{k+1}^{\text{smooth}}\}$ to ensure that the smoothing threshold remains strictly lower than the sharp threshold ($l_k^{\text{smooth}} < l_k^s$). Our approach eliminates the use of inefficient top-$k$ functions, thereby incurring almost no additional computational or temporal burden.

\subsection{Details of Hyperparameter Schedulers}\label{app:hyper-params-schedule}

We keep $\beta_1$ and $\beta_2$ fixed throughout training and impose a schedule on $\chi$ in \ours. When the sharp directions use the \texttt{cos} learning rate scheduler, we also use the \texttt{cos} scheduler for the flat-direction learning rate. Specifically, we define the \texttt{cos} scheduler as
{\footnotesize
\begin{equation}
    \operatorname{cos}(\eta_{\min},\eta_{\max},T_{\max},T_{\text{warm}},t) = 
    \begin{dcases}
          \frac{t\eta_{\max}}{T_{\text{warm}}}, & 0 \le t < T_{\text{warm}},\\
        \eta_{\min} + \frac{1}{2}(\eta_{\max} - \eta_{\min})\left(1 + \cos\left(\frac{\pi (t - T_{\text{warm}})}{T_{\max} - T_{\text{warm}}}\right)\right), & T_{\text{warm}} \le t \le T_{\max}.
    \end{dcases}
    \label{eq:cosine_schedule}
\end{equation}}
If the flat directions adopt the learning-rate scheduler $\operatorname{cos}(\eta_{\min},\eta_{\max},T_{\max},T_{\text{warm}},\cdot)$, then the flat-direction learning-rate ratio $\chi$ (i.e., $\chi(t)$) follows the scheduler
\begin{equation}
  \chi(t)=\min\left\{\underbrace{ 1+\frac{\max\{t-T_{\text{warm}},0\}}{T_{\max}-T_{\text{warm}}} \frac{\tilde{\chi}-1}{T^{\text{flat}}_{\text{warm}}}}_{\text{warmup}},\underbrace{ \frac{\operatorname{cos}(\eta_{\min},\eta_{\max}\tilde{\chi},T_{\max},T_{\text{warm}},t)}{\operatorname{cos}(\eta_{\min},\eta_{\max},T_{\max},T_{\text{warm}},t)}}_{\text{decay}} \right\},
\end{equation}
where $\tilde{\chi}\ge1$ is the amplification ratio (e.g., $2.0$ in the experiments). We set $T^{\text{flat}}_{\text{warm}}\ge T_{\text{warm}}$ so that $\chi$ has a longer warmup period and training remains stable during the initial phase. We set $T^{\text{flat}}_{\text{warm}}=0.5T_{\max}$ in all experiments using the \texttt{cos} learning rate scheduler.

\section{Experimental Details}\label{appen:exp-details}

{\bf Models.} We utilize two popular classes of LLM models for our pre-training experiments:
\begin{itemize}[leftmargin=2em]    
    \item {\bf LLaMA.} LLaMA~\citep{touvron2023llamaopenefficientfoundation} is a widely adopted dense decoder-only Transformer architecture. It employs Rotary Positional Embeddings (RoPE)~\citep{su2024roformer}, Swish-Gated Linear Units (SwiGLU), and Root Mean Square Layer Normalization (RMSNorm). In this work, we pre-train LLaMA models with sizes ranging from 130M to 1.3B parameters. Detailed model configurations are provided in Table~\ref{table: dense model config and max lrs}.
    \item {\bf QwenMoE.} Qwen2MoE~\citep{yang2024qwen2technicalreport} is a  prominent open-source Mixture-of-Experts (MoE) decoder-only Transformer.  In contrast to LLaMA, Qwen2MoE integrates a hybrid attention mechanism (combining sliding window and full attention) alongside its MoE architecture. For our experiments, we disable sliding window attention given the relatively short context length. We configure the model to activate 4 experts per token.  Auxiliary losses including  $z$-loss with coefficient 0.001 and load balancing loss with coefficient 0.01 are used to make sure stable training.  Refer to Table~\ref{table: moe model config and max lrs} for comprehensive configuration details.
 
\end{itemize}

\begin{table}[!ht]
		\centering
        \renewcommand{\arraystretch}{1.25}
		\caption{\small Dense model configurations and optimally-tuned peak learning rates for Muon and SOAP.}
		\label{table: dense model config and max lrs}
		\begin{small}
		\begin{tabular}{l|c|c|c|c|c|c|c}
		\hline 
		Acronym & Size & $d_{\mathrm{model}}$ & $d_{\mathrm{FF}}$ & n$\_$head & depth &  \texttt{lr\_max} of Muon  &  \texttt{lr\_max} of SOAP \\\hline\hline 
	LLaMA (0.13B) & 134M & 768  & 2048 & 12 & 12 & 5e-3 (on C4) & - \\
	LLaMA (0.25B) & 247M & 768 & 2560 & 16 & 24 & 3e-3 (on Pile) &  3e-3 (on Pile)\\
	LLaMA (0.5B) & 518M & 1280 & 3456 & 20 & 22 & 2e-3 (on Pile)&  2e-3 (on Pile) \\
	LLaMA (1.3B) & 1339M & 2048 & 5461 & 32 & 24 & 1e-3 (on Pile) & -\\ 
	 \hline 
	\end{tabular}
	\end{small}
\end{table}

\begin{table}[!ht]
		\centering
        \renewcommand{\arraystretch}{1.25}
		\caption{\small MoE model configurations and optimally-tuned peak learning rates for Muon on Pile.}
		\label{table: moe model config and max lrs}
		\begin{small}
		\begin{tabular}{l|c|c|c|c|c|c|c|c}
		\hline 
		Acronym & Size & Activated Size & $d_{\mathrm{model}}$ & $d_{\mathrm{FF}}$ & n$\_$head & depth & n$\_$experts & \texttt{lr\_max} \\
		\hline\hline 
	QwenMoE (1B) & 1040M & 297M & 768 & 3072 & 12 & 15 & 32 & 2e-3 \\
	 \hline 
		\end{tabular}
		\end{small}
\end{table}

\paragraph{Datasets.} We conduct pre-training on the following datasets: \begin{itemize}[leftmargin=2em] \item \textbf{C4}~\citep{raffel2020exploring}. The Colossal Clean Crawled Corpus (C4) dataset is employed for our small-to-medium scale pre-training experiments. We use the T5 tokenizer with a vocabulary size of 32,100.

\item \textbf{The Pile}~\citep{gao2020pile}. We utilize The Pile for large-scale (or larger-context) pre-training tasks. For this dataset, we adopt the LLaMA-2 tokenizer~\citep{touvron2023llama} with a vocabulary size of 32,000. 
\end{itemize}

\paragraph{Sequence Packing/Batching.}
For the Pile dataset, we employ a standard \emph{sequence packing} strategy: documents are tokenized, concatenated into a continuous stream, and then segmented into fixed-length sequences of size $L$ (depending on the dataset). This approach minimizes padding overhead and ensures a near-constant number of effective tokens per batch. Training progress is measured in terms of total tokens, with evaluation performed on the official validation splits.

{\bf LR schedulers.}
We evaluate two popular learning-rate (LR) scheduling strategies with 1,000 warm-up steps in all pre-training experiments:
\begin{itemize}[leftmargin=2em]
\item \texttt{cos} (cosine scheduler)~\citep{touvron2023llama}:
a linear warm-up to peak \texttt{lr\_max}, followed by cosine decay to a terminal LR \texttt{lr\_min}.
Following~\cite{hoffmann2022training}, we set $\texttt{lr\_min}=0.1\times \texttt{lr\_max}$.

\item \texttt{wsd} (warmup-stable-decay scheduler)~\citep{hu2024minicpm,hagele2024scaling}:
a linear warm-up to \texttt{lr\_max}, followed by a stable phase keeping LR at \texttt{lr\_max} (up to 80\% of total steps), and then a linear decay to 0 over the final 20\% steps.
\end{itemize}




\subsection{Experimental Details in Section~\ref{practical-acc} and Extended Experiments }\label{appen:exp-details-align}

In this subsection, we demonstrate that gradients (first-order information) can effectively approximate the top eigenspaces of the Hessian (second-order information) for matrix block parameters in language models. 

\begin{wrapfigure}{l}{0.49\textwidth}
\vspace{-4mm}
    \centering
\includegraphics[width=0.99\linewidth]{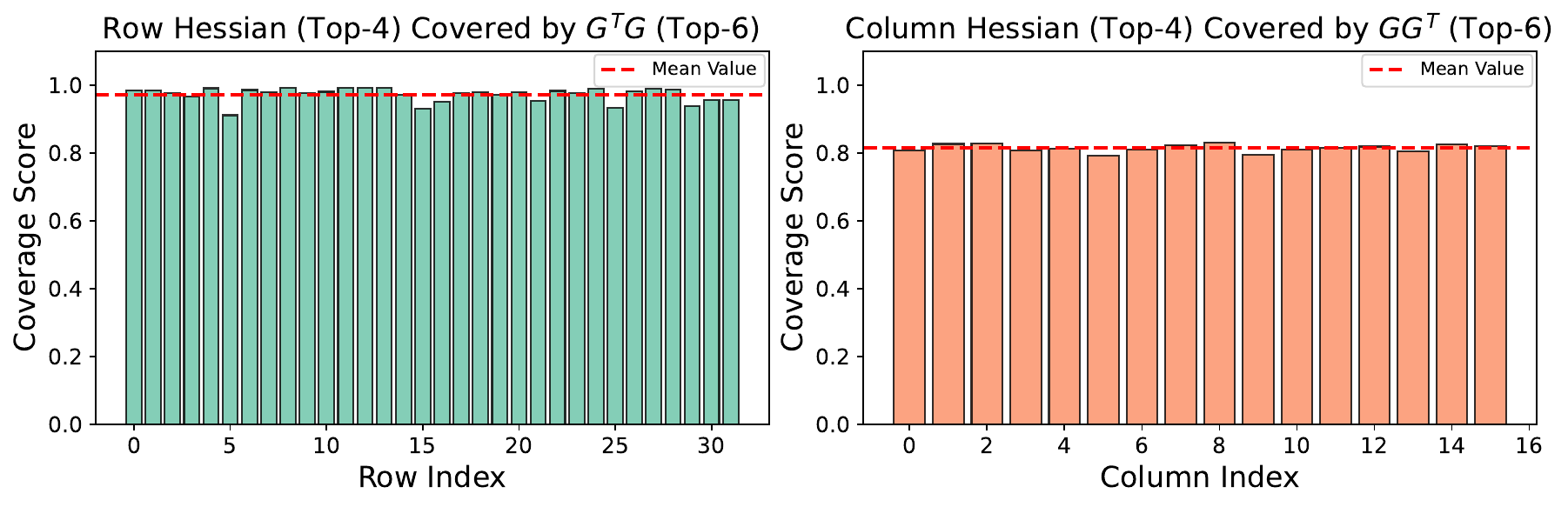}
    \caption{Coverage of the top eigenspaces of row (column) Hessians by those of $G^\top G$ ($G G^\top$) for the \texttt{up\_proj} block of an FFN layer. A higher score indicates  a greater degree of containment. }

    \label{fig:top-align}
\end{wrapfigure}
\paragraph{Validation on a Toy Model}
Hessian-noise alignment phenomena have been extensively investigated and verified in simplified models, both theoretically and empirically \cite{wang2023the,wu2022the}.  To validate this in our context, we conduct an  experiment on a small LLAMA2 model with $d_{\operatorname{model}}=16,d_{\operatorname{ffn}}=32,\text{n\_head}=4,\text{depth}=4,\text{vocabulary size}=8$, $\#\text{params}=10,640$. The model is trained using AdamW for 100 steps with a  token batch size $8\times8=64$ and $\text{lr}=0.001$. For Q,K,V,O, FFN (up, gate, down) blocks, we compute the Hessian eigenvalue distribution within each block (\cref{fig:app-block-hessian,fig:Schematic_illustration}), test the alignment degree of each row's Hessian with $G^\top G$, and each column's Hessian with  $G G^\top$ (\cref{fig:top-align,fig:app-align}). The results shows that for a small dimension $d_s=4$ ($4/32$ and $4/16$), the top-$d_s$ eigenspaces of the Hessian (for rows/columns) are well covered by a subspace of $G^\top G$ ($G G^\top$) of dimension only slightly larger than $d_s$ (\cref{fig:top-align,fig:app-align}), where $G$ is the stochastic gradient matrix.

\begin{figure*}[!ht]
    \centering

    \parbox{0.46\linewidth}{\centering
        \includegraphics[width=\linewidth]{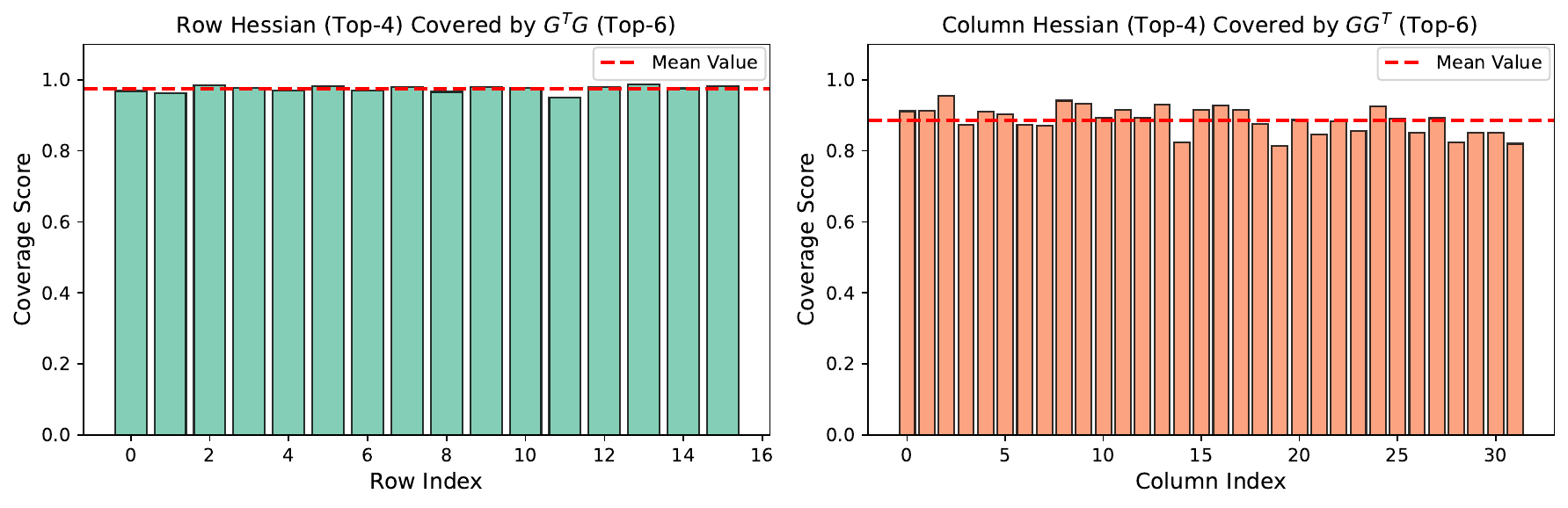}\\
        \small (a) \texttt{down\_proj}
    }
    \hfill
    \parbox{0.46\linewidth}{\centering
        \includegraphics[width=\linewidth]{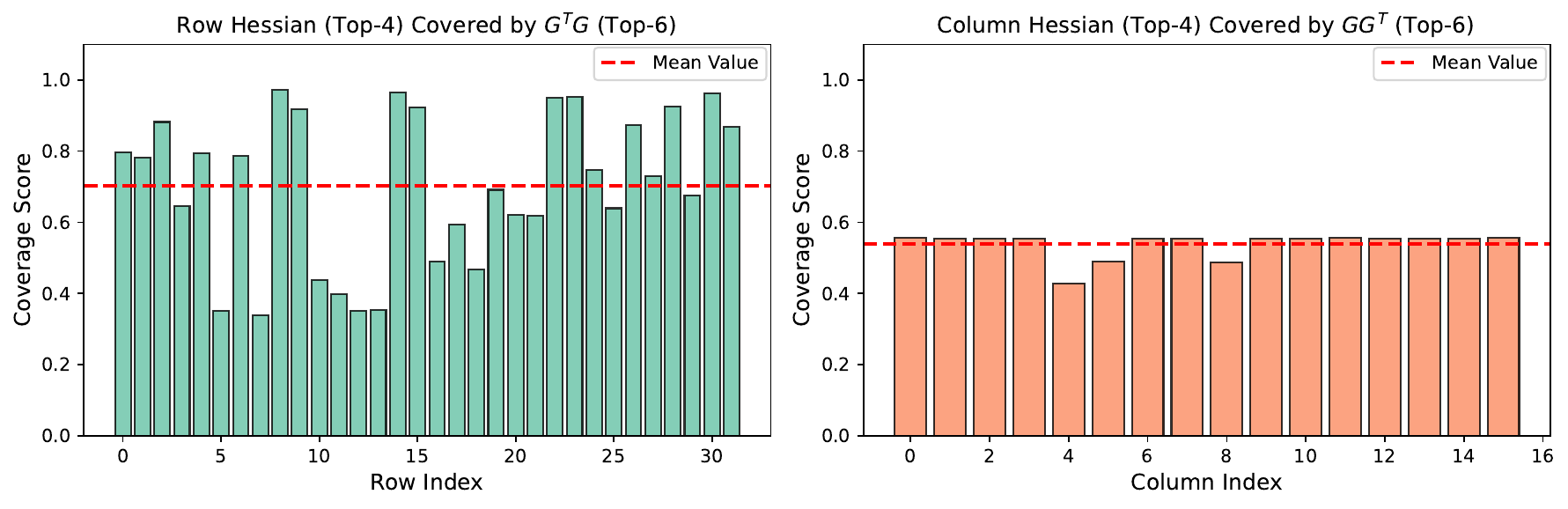}\\
        \small (b) \texttt{gate\_proj}
    }
    
    \vspace{0.5cm}

    \parbox{0.46\linewidth}{\centering
        \includegraphics[width=\linewidth]{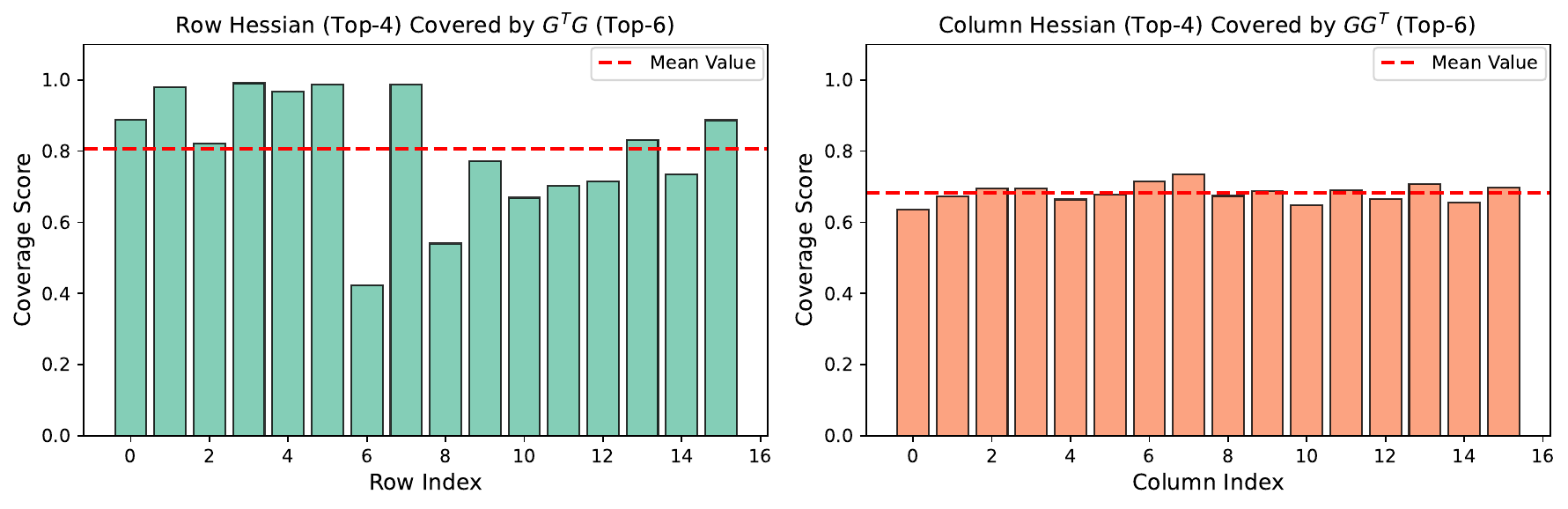}\\
        \small (c) \texttt{q\_proj}
    }
    \hfill
    \parbox{0.46\linewidth}{\centering
        \includegraphics[width=\linewidth]{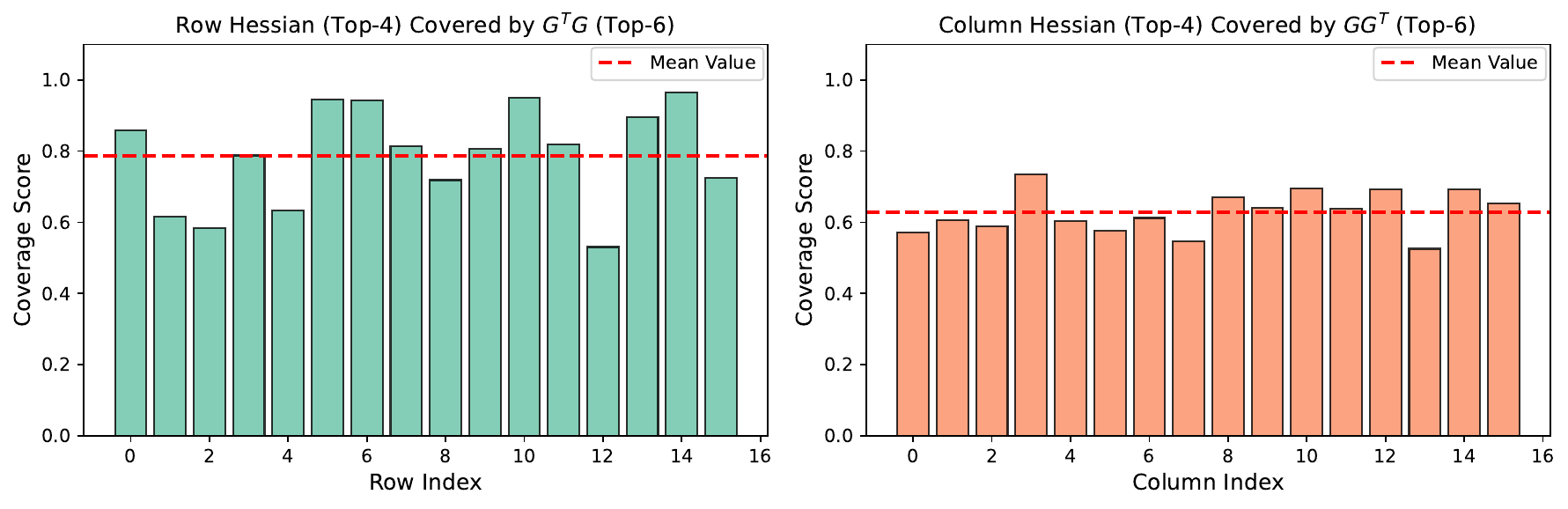}\\
        \small (d) \texttt{k\_proj}
    }
    
    \vspace{0.5cm}

    \parbox{0.46\linewidth}{\centering
        \includegraphics[width=\linewidth]{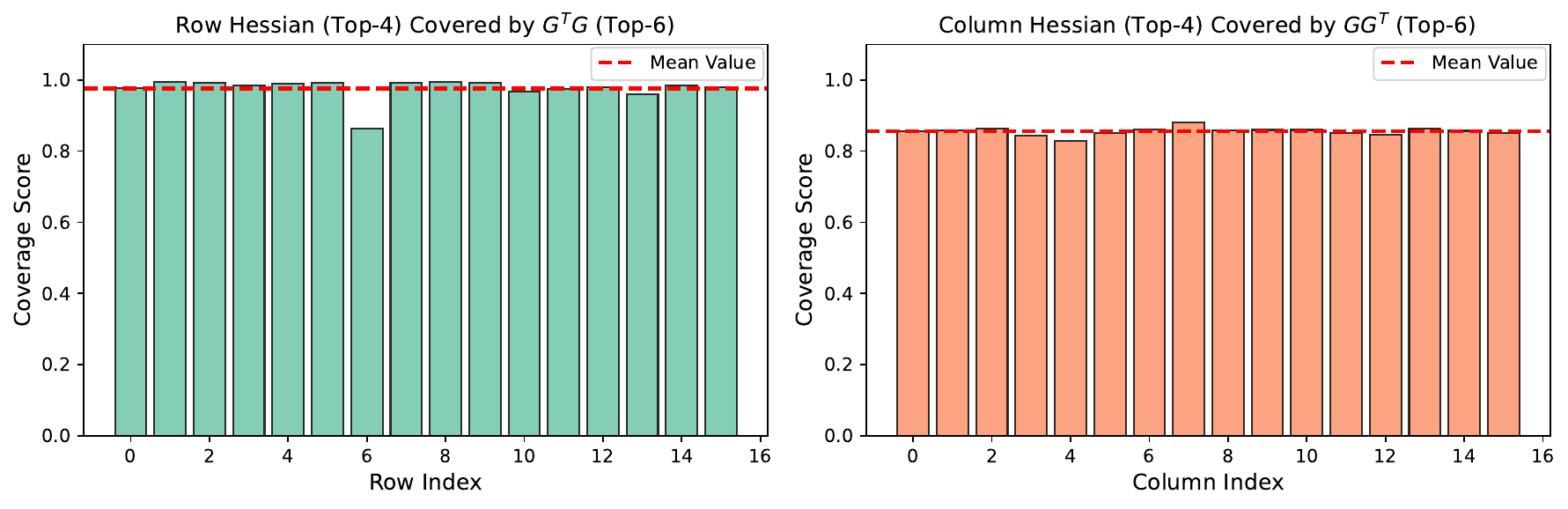}\\
        \small (e) \texttt{v\_proj}
    }
    \hfill
    \parbox{0.46\linewidth}{\centering
        \includegraphics[width=\linewidth]{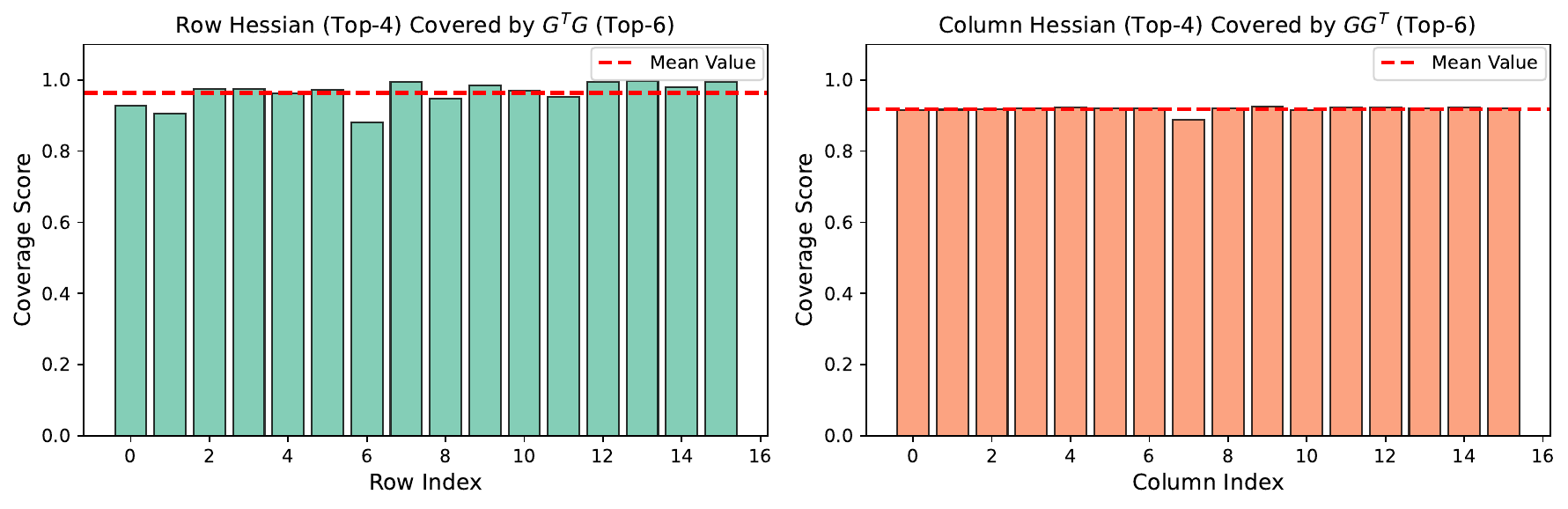}\\
        \small (f) \texttt{o\_proj}
    }
    
    \caption{Coverage scores across row and column indices of different blocks in the toy LLaMA2 model. The results of the \texttt{up\_proj} are in \cref{fig:top-align}. }
   
    \label{fig:app-align}
\end{figure*}

\begin{figure*}[!ht]
    \centering

    \parbox{0.3\linewidth}{\centering
        \includegraphics[width=\linewidth]{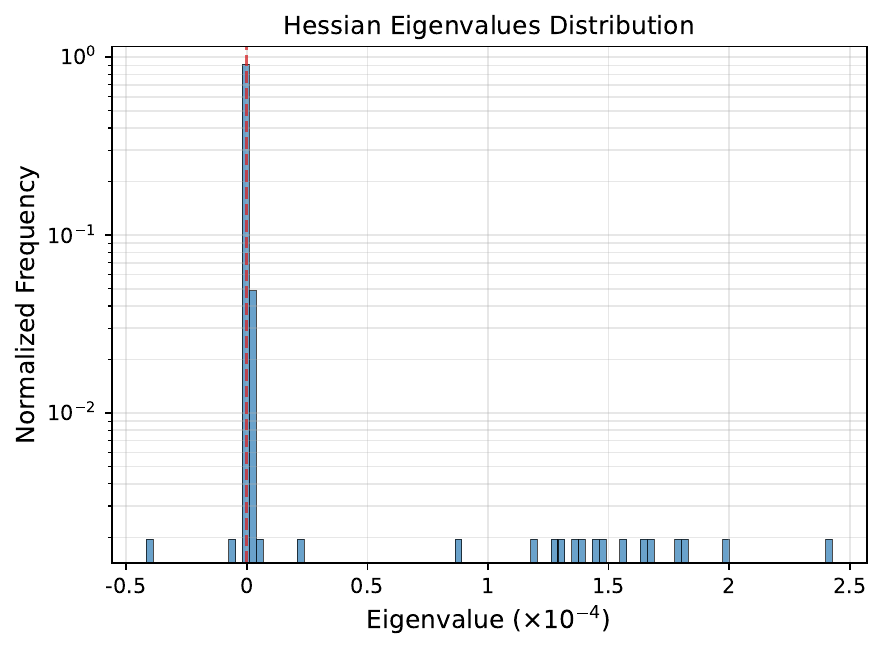}\\
        \small (a) \texttt{down\_proj}
    }
    \hfill
    \parbox{0.3\linewidth}{\centering
        \includegraphics[width=\linewidth]{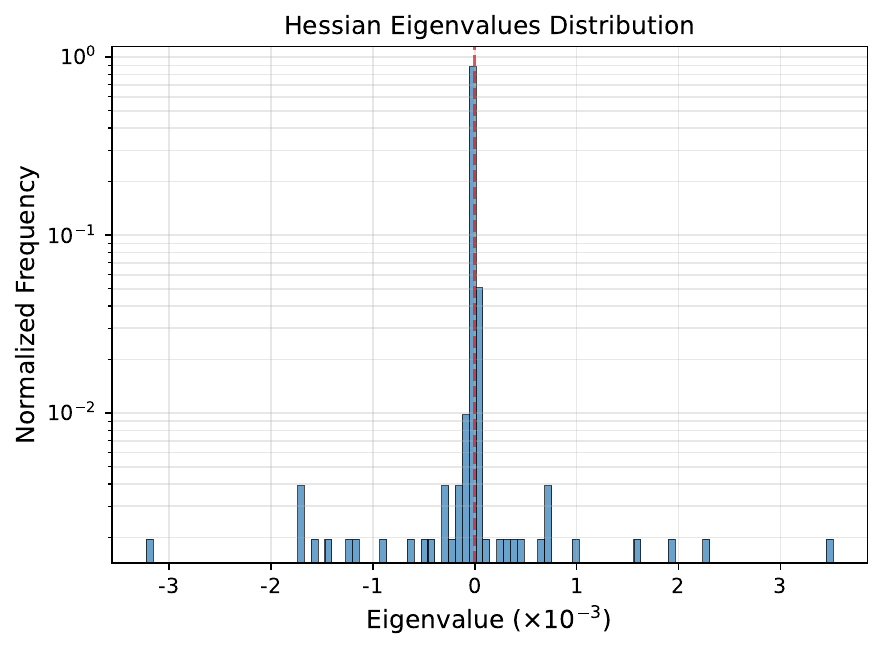}\\
        \small (b) \texttt{gate\_proj}
    }
    \hfill
    \parbox{0.3\linewidth}{\centering
        \includegraphics[width=\linewidth]{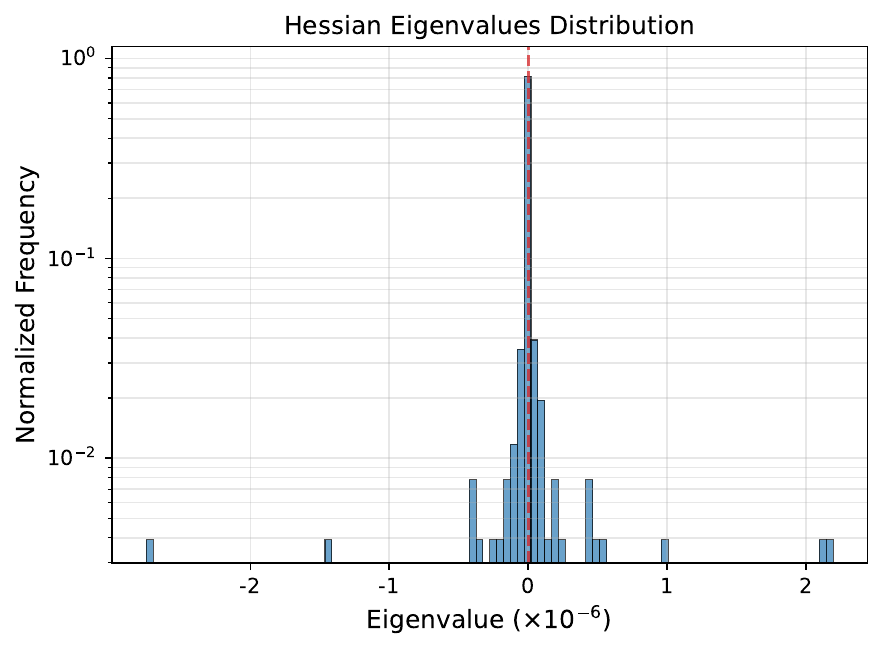}\\
        \small (c) \texttt{q\_proj}
    }

    \vspace{0.5cm}

    \parbox{0.3\linewidth}{\centering
        \includegraphics[width=\linewidth]{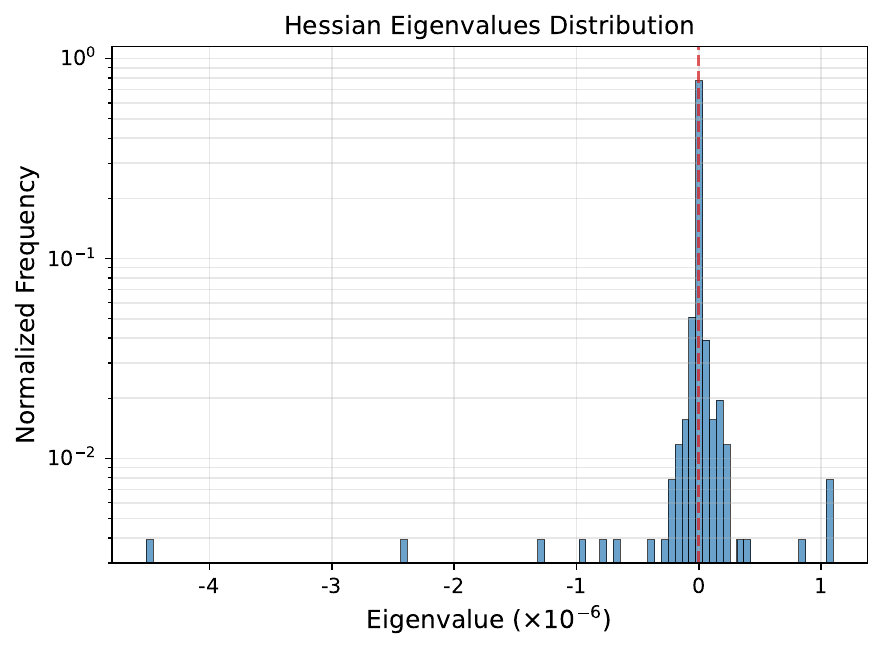}\\
        \small (d) \texttt{k\_proj}
    }
    \hfill
    \parbox{0.3\linewidth}{\centering
        \includegraphics[width=\linewidth]{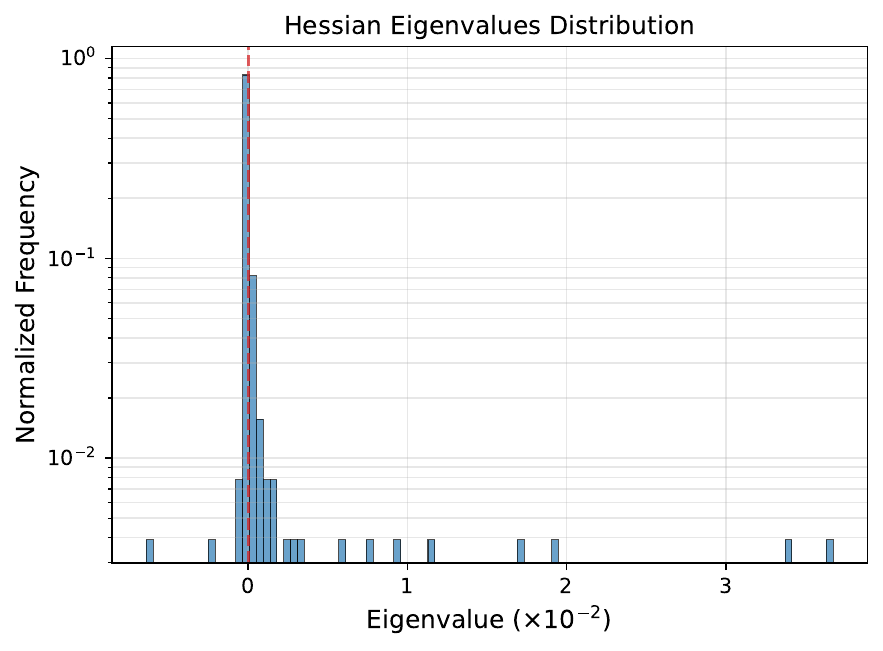}\\
        \small (e) \texttt{v\_proj}
    }
    \hfill
    \parbox{0.3\linewidth}{\centering
        \includegraphics[width=\linewidth]{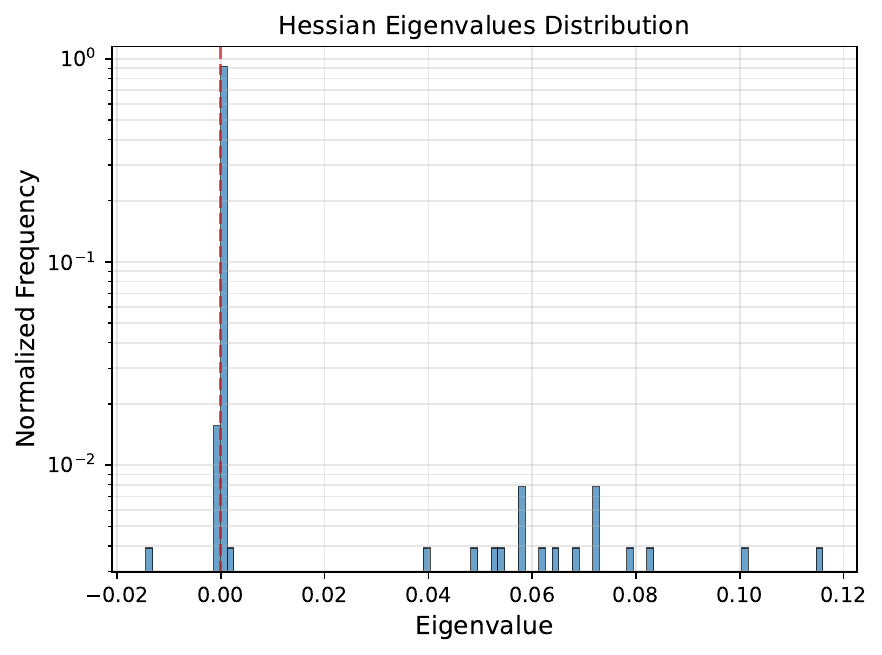}\\
        \small (f) \texttt{o\_proj}
    }

    \caption{Hessian eigenvalue distribution of different blocks in the toy LLaMA2 model. The results of the \texttt{up\_proj} are in \cref{fig:Schematic_illustration}. }
    \label{fig:app-block-hessian}
\end{figure*}

\paragraph{Validation on a 120M LLaMA Model} Now we verify the block-wise Hessian-gradient alignment in a 120M dense LLaMA model with \(d_{\text{model}}=768\), \(d_{\text{FFN}}=2048\), \(n_{\text{head}}=12\), and \(n_{\text{layer}}=6\), trained with a batch size of 480 and a sequence length of 1024. 
We fix a matrix parameter block
$W\in\mathbb{R}^{m\times n}$, and   use $\mathrm{vec}_{\mathrm{row}}(W)$ to denote row-major
vectorization: $w=\mathrm{vec}_{\mathrm{row}}(W)\in\mathbb{R}^{mn}$. The corresponding gradient and  Hessian  on a mini-batch $\mathcal{B}$ is
\begin{equation}
    G\in\mathbb{R}^{m\times n},\quad H_W
    =
    \nabla_w^2 L_{\mathcal{B}}
    \in\mathbb{R}^{mn\times mn}.
\end{equation}
\begin{enumerate}
    \item \textbf{Matrix-free block-Hessian eigenspace estimation.} We avoid explicitly constructing $H_W$. For a matrix direction
$D\in\mathbb{R}^{m\times n}$, by $\mathrm{vec}_{\mathrm{row}}
    \left(\mathcal{H}_W[D]\right)
    =
    H_W\mathrm{vec}_{\mathrm{row}}(D)$, we have $\mathcal{H}_W[D]
    =
    \nabla_W\langle G,D\rangle_F$. We apply restarted Lanczos iterations to this block-restricted operator.
Starting from a normalized random vector $q_1$, each iteration computes
\begin{align}
    z_j &= H_W q_j,\\
    \alpha_j &= q_j^\top z_j,\\
    r_j
    &=
    z_j-\alpha_jq_j-\beta_{j-1}q_{j-1},\\
    \beta_j &= \|r_j\|_2,\\
    q_{j+1} &= r_j/\beta_j.
\end{align}
After $p$ iterations, let
$Q_p=[q_1,\ldots,q_p]$ 
denote the Krylov basis. The projected Hessian is
$T_p=Q_p^\top H_WQ_p$. We compute its eigendecomposition
$
    T_p y_i=\widehat{\lambda}_i y_i$
and map the resulting Ritz vectors back to the original parameter space as $
    \widehat{e}_i=Q_p y_i$.

We retain the $k$ algebraically largest Ritz values and define the
\textbf{estimated top-$k$ Hessian eigenspace as}
\begin{equation}
    Q_H=
    [\widehat{e}_1,\ldots,\widehat{e}_k]
    \in\mathbb{R}^{mn\times k}.
\end{equation}
Each Ritz vector can equivalently be represented as a matrix
$
    E_i
    =
    \mathrm{vec}_{\mathrm{row}}^{-1}(\widehat{e}_i)
    \in\mathbb{R}^{m\times n}$. 
Convergence is checked using the relative Ritz residual
\begin{equation}
    \rho_i
    =
    \frac{
        \|H_W\widehat{e}_i
        -\widehat{\lambda}_i\widehat{e}_i\|_2
    }{
        1+|\widehat{\lambda}_i|
    }.
\end{equation}

\item \textbf{Top subspace of the gradient.}
Using the same fixed mini-batch, we compute the singular value
decomposition of the gradient $
    G=U\Sigma V^\top$. 
Let
    $U_s=U_{[:,1:s]},\,
    V_s=V_{[:,1:s]}$
denote the leading $s$ left and right singular vectors.  Then the top left/right subspaces for the gradient  are 
\begin{align}
    B_{\text{left}}
     =
    U_s\otimes I_n
    \in\mathbb{R}^{mn\times sn},\quad
    B_{\text{right}}
     =
    I_m\otimes V_s
    \in\mathbb{R}^{mn\times ms}.
\end{align}
\end{enumerate}

For the up, gate, down, q, k, v, and o matrices of a given layer, Figure~\ref{fig:top_subspace_align_120m}  reports the coverage scores between the top gradient subspaces and the Hessian subspaces. The results show that \textbf{the top right gradient subspaces achieve well coverage of the top Hessian subspaces}, while the top left gradient subspaces also exhibit a certain degree of alignment with the top Hessian subspaces.

\paragraph{Measure Coverage Degree between Two Subspaces.} Let $A$ and $B$ be two linear subspaces of $\mathbb{R}^d$ with $\dim(A) \le \dim(B)$. Let $P_A=[u_1,...,u_{k_A}] \in \mathbb{R}^{d \times k_A}$ and $P_B=[v_1,...,v_{k_B}] \in \mathbb{R}^{d \times k_B}$ be the matrices consisting of the orthonormal bases of $A$ and $B$, respectively. We define the Coverage Score to quantify the extent to which $A$ is contained within $B$:
\begin{equation}\label{eq:CoverageScore}
    \text{Coverage Score}(A,B)=\frac{1}{k_A}\|P_A ^\top P_B\|_{*}=\frac{1}{k_A}\sum_{i=1}^{k_A}\cos \theta_i,
\end{equation}
where $\|\cdot\|_*$ denotes the nuclear norm, and $\theta_i$ represent 
   the $i$-th principal angle between subspaces $A$ and $B$. 
By definition, the Coverage Score lies in the interval $[0,1]$.  A score of $1$ indicates that $A \subseteq B$, while a score of $0$ implies that $A$ is orthogonal to $B$ ($A \perp B$). A value closer to $1$ suggests a higher degree of containment of $A$ within $B$.

\subsection{Experimental Details in Sections~\ref{experiments} }
\paragraph{Muon Baselines.}
We adopt the standard Muon configuration described in \cite{liu2025muon}. Specifically, we apply   Muon specifically to 2D matrix blocks within transformer layers, while using AdamW for all other parameters (RMSNorm (\texttt{norm}),  Embedding (\texttt{emb}) and the Output (\texttt{out})). We set the learning rate scaling factor to $c=0.2\sqrt{\max\{m,n\}}$ for  blocks that use    Muon with shape $\mathbb{R}^{m\times n}$, ensuring the update RMS norm aligns with that of AdamW. The baseline is configured with Nesterov momentum using a decay coefficient $\theta_{\text{muon}}=0.95$, weight decay $\lambda=0.1$, and a   gradient clipping threshold of $1.0$.

\paragraph{SOAP Baselines.} We maintain the same weight decay and gradient clipping settings as in Muon. Following \cite{vyas2025soap}, we set the hyper-parameters to $\theta=0.95$, $\beta_{v}=0.99$, and $k=10$. regarding the shampoo momentum, although \cite{vyas2025soap} recommends $\beta_{\operatorname{shampoo}}=0.99$, we compared $\beta_{\operatorname{shampoo}} \in \{0.99, 0.95\}$ and selected $0.95$, as it yielded lower validation losses in our experiments.

\paragraph{Learning Rate Schedules.} For the \texttt{cos} schedule, following~\cite{hoffmann2022training}, the minimum learning rate \texttt{lr\_min} is set to $0.1 \times \texttt{lr\_max}$. For the  \texttt{wsd}  schedule, we allocate the warmup and stable phases to the first $80\%$ of total iterations, followed by a linear decay to $0$ over the remaining $20\%$. We set the warm-up steps to $1000$ for all pre-training experiments.

Additionally, 
\begin{itemize}[leftmargin=2em]
    \item {\bf C4 pre-training.}  We use a sequence length of 256 and a global batch size of 3,600. The total number of training tokens is set to approximately 80 times the number of model parameters, corresponding to $4\times$ the Chinchilla scaling law~\citep{hoffmann2022training}. The training process includes 1,000 warm-up steps. We perform a grid search for the maximum learning rate \texttt{lr\_max} over the set $\{\texttt{1e-3}, \texttt{2e-3}, \texttt{3e-3}, \texttt{5e-3}, \texttt{7e-3}\}$. The resulting optimal learning rates for each model are detailed in Table~\ref{table: dense model config and max lrs}. The relevant experiments are performed on four NVIDIA A100 GPUs,  each with 80 GB of memory.

    \item {\bf Pile pre-training.} We set the sequence length to 1,024. The batch size is set to 1,024 for the 0.25B and 0.5B models, and 2,048 for the 1.3B model. The total number of training tokens is approximately 40 times the number of model parameters (corresponding to $2\times$ the Chinchilla scaling law~\citep{hoffmann2022training}), which includes 1,000 warm-up steps. A grid search for \texttt{lr\_max} is performed over the set $\{\texttt{7e-4}, \texttt{1e-3}, \texttt{2e-3}, \texttt{3e-3}, \texttt{5e-3}, \texttt{7e-3}\}$. The optimal learning rates for each model are detailed in   \cref{table: dense model config and max lrs,table: moe model config and max lrs}. The relevant experiments are performed on four (model $<1$ B) or eight  (model $\ge1$ B) NVIDIA H100 GPUs, each with 80 GB of memory.
 
\end{itemize}

 \paragraph{Hyper-parameter tuning for \oursmuon.} For LLaMA models, we identified the optimal hyper-parameters $\{\chi, \beta_1, \beta_2\}$ via  searching on the 250M model (Pile) and applied them uniformly to the 130M (C4), 500M (Pile), and 1.3B (Pile) experiments. Specifically, fixing the base learning rate to Muon's optimal value, we searched the sharp subspace dimension ration $r_s\in \{0.1, 0.2\}$, and set $d_s =r_s \min\{m,n\}$ for \texttt{q,k,v,o,ffn} blocks of size $\mathbb{R}^{m\times n}$, and $d_s =r_smn$ for \texttt{emb, norm} blocks, with all of them  using the same ratio $r_s$. For AdamW blocks, we set $\beta_{1,2}=0.0$ for brevity, and $d_{\operatorname{smooth}} = 0.1mn$ (the final loss is insensitive to this parameter).  Following \cite{wang2025the}, we exclude the \texttt{output} block from \ours (\ie treating  the entire block as sharp directions), as accelerating it offers  minimal gains and may risk  instability.  To avoid expensive searching for the hyper-parameters in \ours, \textbf{we use an efficient sequential and greedy search on a 250M LLaMA model} as follows. This practice of \textbf{searching for hyperparameter configurations on small models and directly scaling them up to larger models and longer training procedures  is also common in recent optimizer studies}, such as \cite{wang2025the}, \cite{yuan2025mars}, \cite{wang2024improving}. We denote the value of $\chi$ for blocks using AdamW and Muon by $\chi_{\text{adamw}},\chi_{\text{muon}}$ respectively. 
 \begin{enumerate}[leftmargin=*, label=\textbf{Step \arabic*:}]
    \item Fix $d_s$ to the top-20\% of the total dimensions and set $\beta_1 = 0$. Search over $\chi \in \{2, 4\}$ for blocks using Muon and AdamW, respectively.
    \item Based on the best configuration from Step~1, search over $\beta_2 \in \{0.5, 1.0, 2.0\}$.
    \item Based on the best configuration from Step~2, test $\beta_1 = -0.25$ to determine whether it yields further improvement.
    \item Finally, reduce $d_s$ to the top-10\% of the total dimensions to examine potential additional gains.
\end{enumerate}
 
Here we permit $\beta_1\le 0$ as  Muon  is inherently equipped with Nesterov momentum. As a result, the parameter $\beta_1$ here is not required to be positive and functions primarily as an extra corrective term. The entire search procedure requires a maximum of 9 individual searches. 
 Finally, we selected $r_s = 0.1$ and the specific settings in \ours variants for LLaMA models are:
 \begin{itemize}\item  {\oursmuon-H:} $(\beta_1, \beta_2) = (-0.25, 2.0)$.
 \item  {\oursmuon-L:} $\chi_{\text{adamw}} = 4$, $\chi_{\text{muon}} = 2$.
 \item  {\oursmuon  :} Adopts th $\chi$ settings from \oursmuon-L and set $(\beta_1, \beta_2) = (-0.25, 1.0)$.
 \end{itemize}

For \oursmuon in the QwenMoE experiment, we set $r_s=0.1$ and $\beta_1=0.0$ for brevity. We  searched $\chi_{\text{adamw}} \in \{4, 8\}$ (while maintaining $\chi_{\text{muon}}=1$ for others) and $\beta_2 \in \{1.0, 2.0\}$ for   blocks using Muon. The final configuration adopted was $\chi_{\text{adamw}}=8$   and $\beta_2=1.0$.

\begin{remark}[The value of $\beta_1$] Taking \(\beta_1 < 0\) in Algorithm 1 would violate the convergence condition. In practice, we allow $\beta_1<0$, as long as the ``effective damping coefficient \(\beta_1\)" remains positive. The apparent discrepancy arises from a notational difference: the \(\beta_1\) notation in Algorithm  \ref{alg:LITE} (denoted as \(\beta_1\) itself) differs from that in the implemented Muon-LITE update (denoted as \(\beta_1^{\text{Muon-LITE}}\)), where the latter already includes an intrinsic Nesterov (Hessian-damping) term, making \(\beta_1\) in Muon-LITE merely an additive correction. The Muon-LITE update is
\[
U_k = \widetilde{M}_{k}((\widetilde{M}_{k})^\top M_k)^{-\frac{1}{2}}(P_k + \chi(I_n - P_k)) + G_k(G_k^\top G_k)^{-\frac{1}{2}}(\beta_1^{\text{Muon-LITE}}P_k + \chi\beta_2(I_n - P_k)).
\]
Expanding the sharp-direction part, the operator acting on the sharp-direction gradient \(G_k\) is
\[
A_k := \theta^{-1}((\widetilde{M}_{k})^\top \widetilde{M}_{k})^{-\frac{1}{2}} + \beta_1^{\text{Muon-LITE}}(G_k^\top G_k)^{-\frac{1}{2}}.
\]
Approximately we have \((\widetilde{M}_{k})^\top \widetilde{M}_{k} \approx (1-\theta^2)^{-1}G_k^\top G_k\) (Section \ref{app:r-exa-pre}). Thus for \(\theta = 0.95\) and \(\beta_1^{\text{Muon-LITE}} = -0.25\), the effective "\(\beta_1\)" in Algorithm  \ref{alg:LITE} is positive: \(\theta^{-1}(1-\theta^2)^{1/2} + \beta_1^{\text{Muon-LITE}} > 0.3 - 0.25 > 0\). We explored negative \(\beta_1^{\text{Muon-LITE}}\) values only to add flexibility in fine-tuning the sharp-direction adaptive damping. Compared with \(\beta_1^{\text{Muon-LITE}} = 0\), the value -0.25 yields a modest and gradually diminishing loss reduction during the early-to-late training stages (from about 0.003 at 2k steps down to 0.001 at 10k steps). Thus one can also set \(\beta_1^{\text{Muon-LITE}} = 0\) for simplicity and still achieve comparable performance. We also note that setting \(\beta_1^{\text{Muon-LITE}} = -0.5\) drives the effective coefficient below zero and degrades the final loss by approximately 0.01, which further confirms that our practical choice stays within the theoretically valid region. 
\end{remark}

\paragraph{Extra hyper-parameter tuning for the Muon baseline.} The hyper-parameters we used for the Muon baseline ($\lambda=0.1$, $\theta_{\text{muon}}=0.95$) are standard in the related literature. As discussed in the previous section, we conducted hyperparameter search for \oursmuon. This search was performed on relatively short runs with a small model, and the optimized hyper-parameters were directly transferred to larger models and longer training runs. Nevertheless, for a fairer comparison, we additionally tuned a more extensive set of hyper-parameters for the Muon baseline, including block-wise learning rate, weight decay coefficient, and momentum decay coefficient, to explore its performance ceiling.  After tuning the  block-wise learning rate (Table  \ref{tab:blk-lr}), weight decay, and momentum decay coefficients (Table \ref{tab:muon-base-extra-tune}), we observed that the best configuration only outperforms the baseline by 0.005, which is substantially smaller than the gain delivered by \oursmuon  (0.031). This indicates that the standard Muon baseline hyper-parameters are already near-optimal, and the improvements from \oursmuon stem from its algorithmic design.

\begin{table}[h]
    \centering
    \begin{tabular}{|c|c|c|c|}
        \hline
        \diagbox{$\theta_{\text{muon}}$}{$\lambda$}& 0.05 & 0.1 & 0.2 \\
        \hline
        0.9 & 2.104 (best, -0.005) & 2.105 & 2.105 \\
        \hline
        0.95 & 2.110 & 2.109 (base) & 2.109 \\
        \hline
        0.98 & 2.155 & 2.155 & 2.154 \\
        \hline
    \end{tabular}
    \vspace{2mm}
    \caption{Results of tuning momentum decay coefficient $\theta_{\text{muon}}$ and weight decay coefficient $\lambda$ on Muon (250M LLaMA model trained on Pile).}\label{tab:muon-base-extra-tune}
\end{table}

\paragraph{Hyper-parameter tuning for \ourssoap.}  Similar to \oursmuon. we identified the optimal hyper-parameters $\{\chi, \beta_1, \beta_2\}$ via grid search on the 250M model (Pile) and applied them to other scenarios. Specifically, fixing the base learning rate to SOAP's optimal value, we searched the sharp subspace dimension ration $r_s\in \{0.1, 0.2\}$, and set  $d_s =r_smn$ for all  blocks, with all of them  using the same ratio $r_s$. The search space for other coefficients is $\beta_1=0.0$, $\beta_2 \in \{0.25,0.5, 1.0\}$, and $\chi \in \{2.0, 4.0\}$, where we use similar greedy and sequential search as for \oursmuon to avoid exhausted grid search. We set $d_{\operatorname{smooth}} = 0.1mn$ uniformly (the final loss is insensitive to this parameter). Following \cite{wang2025the}, we exclude the \texttt{output} block from \ours (\ie treating  the entire block as sharp directions), as accelerating it offers minimal gains and may risk  instability. Finally, we selected $r_s = 0.2$ and the specific settings in \ourssoap variants for LLaMA models are:
 \begin{itemize}\item  {\ourssoap-H:} $(\beta_1, \beta_2) = (0.0, 0.5)$.
 \item  {\ourssoap-L:}  $\chi_{\text{adamw}} = 4$, $\chi_{\text{soap}} = 2$.
 \item  {\ourssoap  :} Adopt the $\chi$ settings from \ourssoap-L and set $(\beta_1, \beta_2) = (0.0, 0.25)$.
 \end{itemize}

\section{Additional Experimental Results}\label{app:add-experiment}

\subsection{Verification of Isotropic Updates}\label{iso-ver}

Isotropy in optimization landscapes manifests not only as strict per-step normalization but also as distributional isotropy (e.g., conforming to a standard Gaussian). Sign-based algorithms like signSGD and Muon explicitly enforce strict isotropy by clipping the update magnitude along each eigen-direction to a fixed constant per step.  Methods  like AdamW and SOAP decouple the preconditioner and momentum update, achieving isotropy  through their statistical distributions. 

We empirically demonstrate this distributional isotropy using SOAP as an example. Let $(m,g)$ and $M,G$ represent the vector-form and matrix-form momentum and stochastic gradient. SOAP computes a whitened update $u=F^{-1}m$ (vector-form) with $U=Q_l \frac{Q_l^\top MQ_r}{\sqrt{V_{\text{rot}}} }  Q_r^\top $ (matrix-form), where $Q_l, Q_r$ are orthogonal matrices. 
Here, $F=(Q_r\otimes Q_l) (\operatorname{diag}  v_{\text{rot}}  )^{\frac{1}{2}}(Q_r^\top\otimes Q_l^\top)$ effectively estimates the uncentered  covariance $\mathbb{E}[gg^\top]$ (see discussions in Appendix \ref{app:r-exa-pre}). This whitening process drives $u$ toward distributional isotropy, meaning its covariance tends to $aI$ for some constant $a$. We verify this   using a FFN block from a 130M LLaMA model. In the SOAP-rotated eigenspace, the update $U_{\text{rot}}=\frac{Q_l^\top MQ_r}{\sqrt{V_{\text{rot}}} }$ (matrix-form)  demonstrates nearly uniform row and column norms (Figure \ref{fig:sub_rot_m}), with its individual elements closely follow an independent Gaussian distribution (Figure \ref{fig:sub_gauss}). Notably, the unwhitened  raw momentum $M_{\text{rot}}= {Q_l^\top MQ_r} $ fails to exhibit such isotropic properties (Figure \ref{fig:sub_raw_m}).

\begin{figure}[htbp]
    \centering

    \begin{subfigure}{0.48\textwidth}
        \centering
        \includegraphics[width=\linewidth, keepaspectratio]{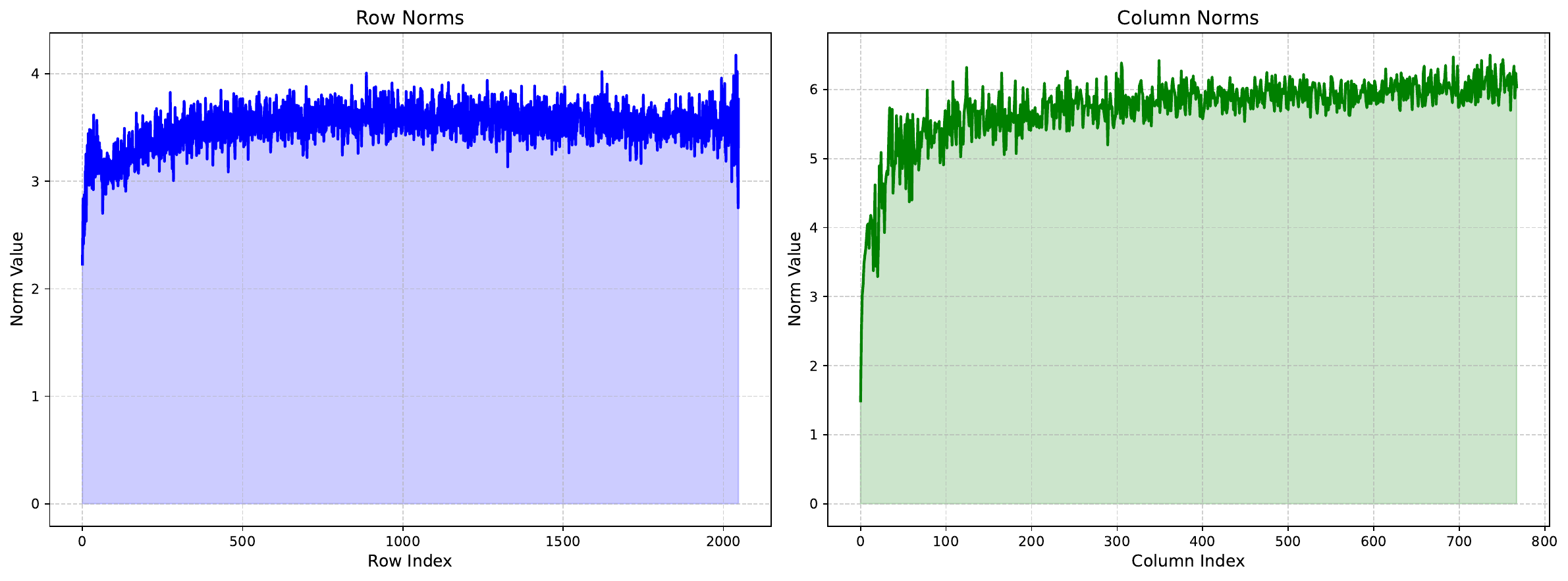}
        \caption{Row and Column norms of SOAP update $\frac{M_{\text{rot}}}{\sqrt{V_{\text{rot}}}}$ matrix in the rotated space.} \label{fig:sub_rot_m}
    \end{subfigure}
    \hfill 
    \begin{subfigure}{0.48\textwidth}
        \centering
        \includegraphics[width=\linewidth, keepaspectratio]{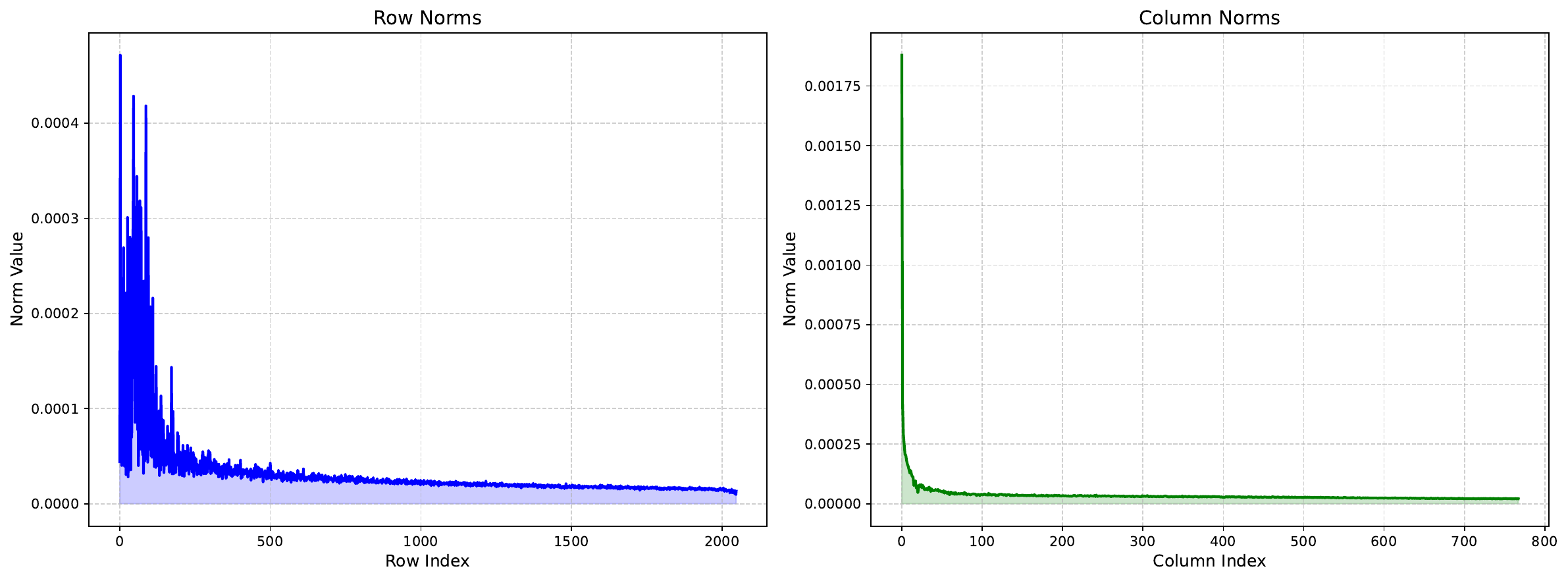}
        \caption{Row and Column norms of SOAP momentum ($M_{\text{rot}}$) matrix in the rotated space.} \label{fig:sub_raw_m}
    \end{subfigure}

    \begin{subfigure}{0.65\textwidth}
        \centering
        \includegraphics[width=\linewidth, keepaspectratio]{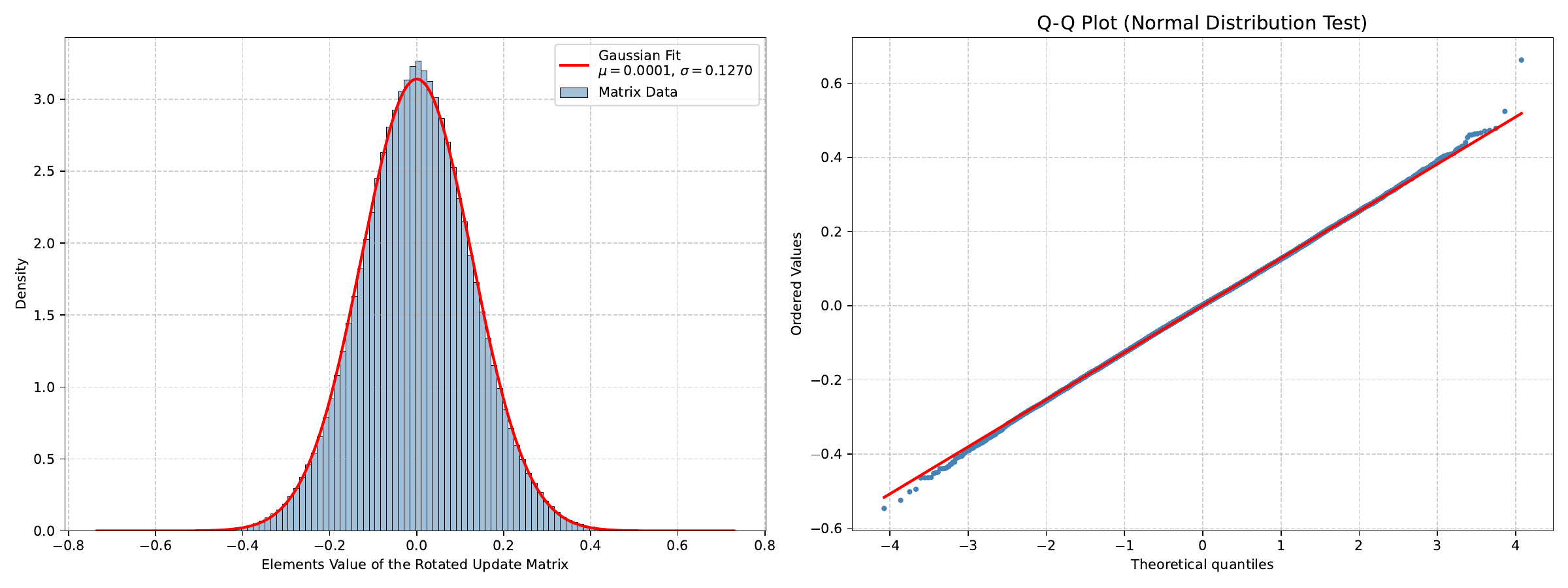}
        \caption{Left: Distribution of the update matrix elements in the rotated space. Right: The Q-Q plot for the distribution shown on the left. The adherence to the straight line validates the Gaussianity of the elements.} 
        \label{fig:sub_gauss}
    \end{subfigure}

    \caption{Visualization of SOAP update statistics. The figure illustrates the underlying metrics of a gate\_proj block of a 130M LLaMA model at training step 5000, capturing the mid-to-late stage dynamics.}
\end{figure}

\subsection{Ablation Studies and Hyper-parameter Robustness}\label{abl-sen}

\paragraph{Ablation Studies.}
Under the experimental setting for the 0.25B model in \cref{fig:cosine-llama-muon}, we tested applying uniform damping coefficients $\beta_{1,2}= 0.5$ and $\beta_{1,2}= 1.0$. Despite these values being   smaller than the coefficient used for flat directions in \oursmuon-H ($\beta_2=2.0$), they yielded terminal losses of $2.113$ and $2.129$  respectively, both even  inferior to the vanilla Muon baseline ($2.110$). The results validate our design choice of maintaining hyper-parameters along sharp directions when selectively accelerating the flat ones.

\paragraph{Hyper-parameter Robustness.} Although LITE introduces extra hyper-parameters $(d_s,\chi,\beta_{1,2})$ to achieve acceleration,   it exhibits reasonably robustness to these choices. Here we use the experimental setting for the 0.25B model in \cref{fig:cosine-llama-muon} for illustration. 
Table~\ref{tab:performance} shows that \oursmuon is insensitive to small choices of the sharp direction dimensionality. We attribute this robustness to the true sharp directions constituting a minimal fraction of the total dimensions. The top 10\% subspace is adequate to cover sharp directions. Since top 10\% to 20\% range mainly lie in flat directions and account for a small proportion, their acceleration effect is minimal. Besides, table \ref{tab:param_results} shows that \oursmuon achieves stable acceleration to Muon across a wide range of $(\chi,\beta_{1,2})$. Based on these findings, one can simply adopt moderately small values for safe and stable acceleration, with further speedup attainable through extensive hyper-parameter search.

\begin{table}[H]
\centering
\caption{Influence of the sharp directions dimensionality (ratio) $d_{s}$ to \oursmuon and the variants. }
\label{tab:performance}
\begin{tabular}{l|cc}
\hline
\diagbox{optimizer}{$d_{s}$} &   10\% & 20\% \\
\hline
Muon-LITE-L &     2.088 & 2.092  \\
Muon-LITE-H &     2.091 & 2.093 \\
Muon-LITE   &     2.078 & 2.082 \\
\hline
\end{tabular}
\end{table}

\begin{table}[H]
\centering
\caption{\oursmuon achieves stable loss reduction to Muon across a wide range of hyper-parameters ($\chi_{\text{adamw}}, \chi_{\text{muon}}, \beta_1, \beta_2$), where $d_s=0.1$ in these experiments.}
\label{tab:param_results}
\begin{tabular}{l|c}
\hline
configs ($\chi_{\text{adamw}}, \chi_{\text{muon}}, \beta_1, \beta_2$) & loss \\
\hline
$(4, 2,  -0.25, 1.0)$ & 2.078 (-0.031) \\
$(4, 2,  -0.25, 0.5)$ & 2.082 (-0.027) \\
$(4, 2,  0.0, 0.5)$   & 2.082 (-0.027) \\
$(2, 1,  0.0, 1.0)$   & 2.092 (-0.017) \\
$(2, 1,  0.0, 2.0)$   & 2.088 (-0.021) \\
$(2, 2, 0.0, 0.0)$        & 2.099 (-0.010) \\
$(4, 2,  0.0, 0.0)$   & 2.088 (-0.021) \\
$(1, 1, 0.0, 2.0)$        & 2.093 (-0.016) \\
$(1, 1, 0.0, 1.0)$        & 2.097 (-0.012) \\
$(1, 1, 0.0, 0.5)$        & 2.102 (-0.007) \\
$(1, 1, 0.0, 0.0)$        & 2.109 (Muon baseline) \\
\hline
\end{tabular}
\end{table}

\subsection{Supplementary Results on Dense Models}\label{app:SupplementaryResults}

Here we present the supplementary results for the experiments in \cref{result:dense}.

\paragraph{Results of accelerating SOAP (\ourssoap).} The results are shown in Figure \ref{fig:cosine-llama-soap}. 
\begin{figure} [H]
\vspace{-2mm}
    \centering
    \includegraphics[width=0.24\linewidth]{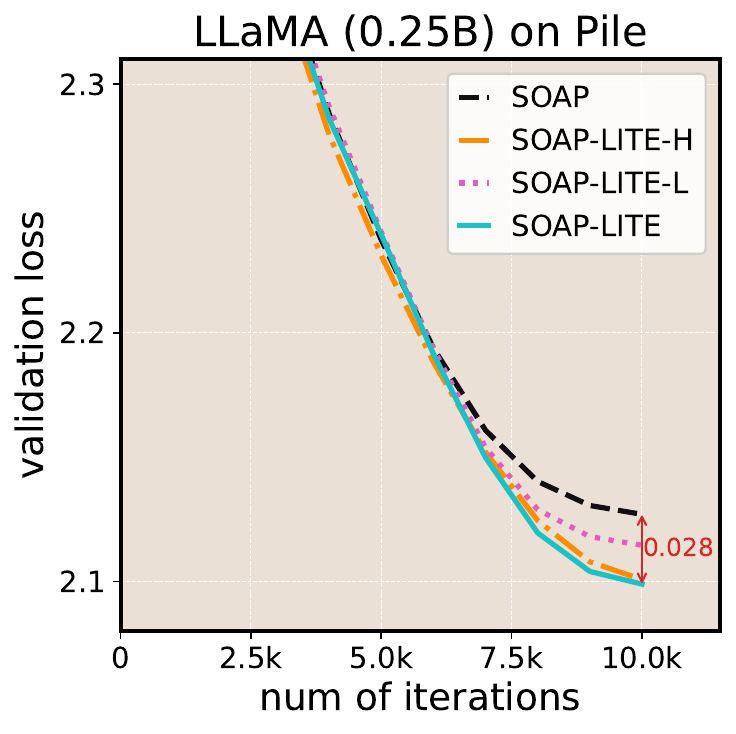}
    \includegraphics[width=0.24\linewidth]{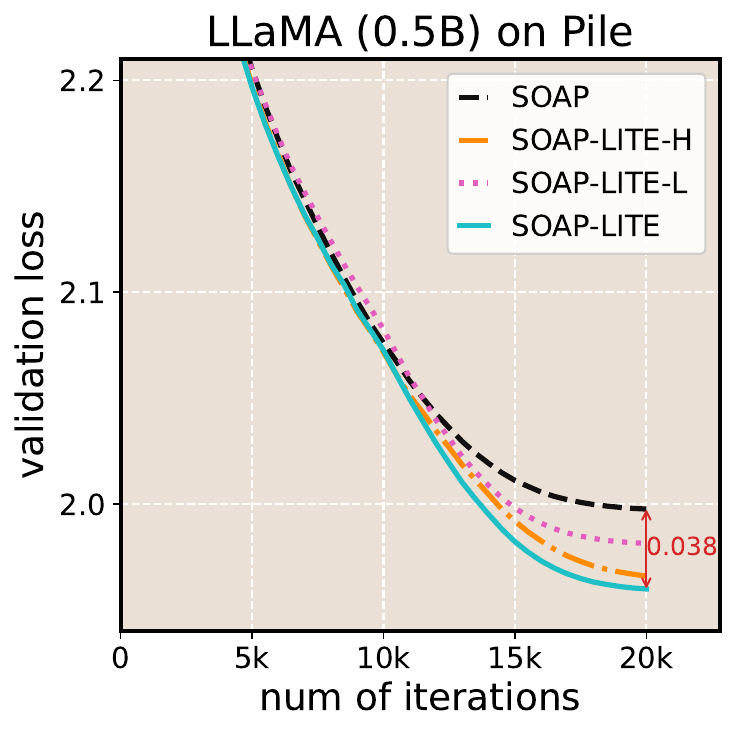}
    \caption{Performance comparison of SOAP and \ourssoap on LLM pre-training tasks. Same experimental setup and notations are employed as in \cref{fig:cosine-llama-muon}.}
    \label{fig:cosine-llama-soap}
  \vspace{-3mm}
\end{figure}

\paragraph{Results of downstream tasks.} 
We evaluate the 0-shot performance of the pre-trained LLaMA 1.3B models for 25k steps (\cref{fig:cosine-llama-muon})  in \cref{tab:downstream evaluation-1b}. \oursmuon demonstrates consistent improvements over the Muon baseline across the eight evaluated tasks. Notably, \oursmuon achieves a substantial gain   on BoolQ  and  exhibits strong generalization on reasoning-heavy benchmarks, showing significant margins on ARC-Challenge (+1.79) and MMLU (+1.39). These results indicate that MUON-LITE effectively enhances   model performance across diverse domains ranging from common sense reasoning to truthful QA.
\begin{table*}[!h]
    \centering
    \caption{Evaluation results on downstream tasks (0-shot with lm-evaluation-harness) of LLaMA models (1.3B) pre-trained on Pile using \oursmuon and Muon. The best scores in each column are bolded. Abbreviations:  \texttt{T\_mc2}=\texttt{TruthfulQA\_mc2}, \texttt{AVG}=Average score. 
    }
    \resizebox{\textwidth}{!}{
    \begin{tabular}{c|c|c|c|c|c|c|c|c|c}
    \hline\hline
    Method & \texttt{ARC\_E} & \texttt{ARC\_C} & \texttt{PIQA} & \texttt{HellaSwag} & \texttt{BOOLQ} & \texttt{WinoGrande} & \texttt{MMLU}& \texttt{T\_mc2}& \texttt{AVG} \\\hline
       Muon  & \small 52.99 &  \small 22.10 &  \small 67.41 &  \small 35.30 &  \small 52.45 &  \small 53.35 &  \small 23.76&  \small 40.14 &  \small 43.44    \\\hline
       \oursmuon  &  \small {\bf 54.38} &  \small {\bf 23.89} &  \small {\bf 67.63} &  \small {\bf 36.72} &  \small {\bf 60.52} &  \small {\bf 54.78} &  \small {\bf 25.15} &  \small {\bf 40.83} &  \small {\bf 45.49} \\\hline\hline
    \end{tabular}}
    \label{tab:downstream evaluation-1b}
    \vspace{-.05cm}
\end{table*}

\paragraph{Results under different random seeds} The loss variation across different seeds is negligible relative to the reduction by \ours (larger than 2e-2), consistently below 2e-3. Table \ref{tab:seed_comparison} reports the results of Muon and Muon-LITE on the 250M LLaMA2 model across four random seeds, indicating that seed-induced loss differences are negligible.

\begin{table}[H]
    \centering
    \caption{Comparison of Muon and Muon-\ours across 4 random seeds (with mean \(\pm\) std)}
    \label{tab:seed_comparison}
    \begin{tabular}{lccccc}
        \toprule
        Muon & 2.10962 & 2.10986 & 2.11034 & 2.10927 & 2.10977 \(\pm\) 0.00045 \\
        Muon-LITE & 2.07828 & 2.07930 & 2.07979 & 2.07956 & 2.07923 \(\pm\) 0.00067 \\
        \bottomrule
    \end{tabular}
\end{table}
\section{Discussions on   Momentum Formulations   and their Continuous-Time ODE Counterparts}\label{nes-discussion} 

\subsection{Continuous-Time ODE Formulations}

In this subsection, we introduce the continuous-time limits of momentum-based algorithms and demonstrate how to derive discrete optimizers from these continuous formulations (modeled as second-order or first-order systems). The derivation proceeds by reducing the second-order dynamics into a first-order system, followed by numerical discretization. This pipeline serves as a foundation for the Riemannian setting \eqref{R_acc:flow}, which is the primary focus of this work (see Appendix \ref{app:tran-R-2nd-1st}). We consider the inertial system  with Hessian damping  (ISHD) mentioned in \cref{sec:ode-momentum}
\begin{equation}\label{app:2nd-form-eu}
    \ddot{w}_t+\alpha_t \dot{w}_t + \beta_t   \nabla^2 f(w_t)\dot{w}_t + \gamma_t \nabla f(w_t) = 0,
\end{equation}
where $\alpha_t,\beta_t,\gamma_t\ge0$   denote the coefficients for momentum decay, Hessian damping, and the gradient driving force, respectively. By introducing an auxiliary momentum variable $m_t$  to track the velocity $\dot{w}_t$, we can reformulate this second-order equation \eqref{app:2nd-form-eu} into an equivalent first-order system: 
\begin{align}
\begin{dcases*}
     \dot{w}_t =- \gamma_t m_t-\beta_t\nabla f(w_t)  , \\
     \dot{m}_t=-\alpha_t m_t +(1- {\alpha_t\beta_t}/{\gamma_t}-\dot{\beta}_t)\nabla f(w_t). 
\end{dcases*}
\end{align}
Applying a discretization with step size $h>0$ yields the following update rule at the $k$-th iteration:
\begin{align} 
\begin{dcases*}
     m_{k}=(1-\alpha_k h) m_{k-1} +(1-{\alpha_k\beta_k}/{\gamma_k}-\dot{\beta}_k)h \nabla f(w_k),  \\
     w_{k+1} =w_k- \gamma_k h m_{k}-\beta_k h\nabla f(w_k).
\end{dcases*}
\end{align}
This formulation subsumes various momentum methods, recovering Heavy Ball momentum when $\beta_k=0$ and Nesterov-type momentum when $\beta_k>0$. Crucially,  the Hessian damping term $\beta_t\nabla^2f(w_t) \dot{w}_t$ and its discretized counterpart  $\beta_k\nabla f(w_t)$ serves to enhance standard Heavy Ball momentum by suppressing oscillations. We provide a detailed discussion in Appendix \ref{app:hb-nes-hd}.

 \subsection{Heavy Ball Momentum, Nesterov   Momentum,    and Hessian Damping}\label{app:hb-nes-hd}
In this section, we discuss equivalent formulations of Nesterov-type momentum and Heavy Ball momentum (also known as Polyak momentum). Furthermore, we analyze the acceleration dynamics of Nesterov momentum, highlighting the intrinsic connections between \textbf{Hessian damping and gradient correction.} 
For brevity, we assume the preconditioner is the identity matrix ($F=I$) throughout this discussion. 

\paragraph{Heavy Ball Momentum.} The Heavy Ball momentum update is defined as:
\begin{align}\label{hb-m}
\begin{dcases*}
m_k=(1-\alpha) m_{k-1} + \nabla f(w_k),\\
 w_{k+1} =w_k- \eta m_k .
\end{dcases*}
\end{align}
This recurrence can be rewritten in a single-variable form equivalent to:
\begin{align}\label{hb-m-1}
w_{k+1}=w_k+(1-\alpha)(w_k-w_{k-1})-\eta\nabla f\left(w_k\right). 
\end{align}
 This formulation aligns with the momentum schemes adopted in AdamW and many other adaptive optimizers. Note that here, the coefficient $\alpha$ of the gradient term (typical in EMA-style updates) is absorbed into the effective learning rate $\eta$.

\paragraph{Nesterov Momentum.} Unfortunately, Heavy Ball momentum lacks a theoretical guarantee of globally achieving a faster convergence rate than Gradient Descent, even in deterministic strongly convex settings. In contrast, Nesterov momentum provably achieves acceleration: the convergence rates for Gradient Descent and Nesterov momentum are approximately $(1-\frac{1}{\kappa})^k$ and $(1-\frac{1}{\sqrt{\kappa}})^k$ respectively, where $\kappa$ denotes the condition number.

Nesterov momentum with constant coefficients is typically formulated as:
\begin{align}\label{nes-m-0}
\begin{dcases*}
x_k=w_{k-1} - \eta\nabla f(w_{k-1}),\\
 w_{k} =x_k+(1-\alpha)(x_k-x_{k-1}).
\end{dcases*}
\end{align}
Eliminating the auxiliary variable $x$ or $w$ yields two equivalent interpretations of \eqref{nes-m-0}: 
\begin{equation}\label{nes-m-0-1}
    \begin{aligned}
     x_{k+1}=x_k+(1-\alpha)(x_k-x_{k-1})-\eta \underbrace{\nabla f\left(x_k+(1-\alpha)(x_k-x_{k-1})\right)}_{\text{gradient correction}},  
    \end{aligned}
\end{equation}
and 
\begin{equation}\label{nes-m-0-2}
    \begin{aligned}
     w_{k+1}=w_k+(1-\alpha)(w_k-w_{k-1})-\eta \nabla f\left(w_k\right)  -(1-\alpha) \eta\underbrace{(\nabla f(w_k)-\nabla f(w_{k-1}))}_{\text{gradient correction}}.
    \end{aligned}
\end{equation}

\paragraph{The Gradient Correction (Hessian Damping) Insight.} 
Comparing \eqref{nes-m-0-2} with the Heavy Ball form \eqref{hb-m-1}, it is evident that Nesterov momentum introduces an additional \textbf{gradient correction term}. Using a   Taylor expansion, we observe that:
\begin{equation}
    \begin{aligned}
        &\nabla f(x_k+(1-\alpha)(x_k-x_{k-1}))\approx (1-\alpha)\nabla^2f(x_k) (x_k-x_{k-1}),
        \\& \nabla f(w_k)-\nabla f(w_{k-1}) \approx  \nabla^2f(w_k) (w_k-w_{k-1}).
    \end{aligned}
\end{equation}
This reveals that the gradient correction implicitly incorporates \textbf{Hessian information}. This mechanism is the key to the  acceleration of Nesterov momentum: it exerts an inertial damping effect (often referred to as Hessian damping) which is proportional to the curvature. Consequently, this term applies stronger damping in directions corresponding to large Hessian eigenvalues, effectively suppressing oscillations.

The formulation in \eqref{nes-m-0-2} is also mathematically equivalent to: 
\begin{equation}\label{nes-m-0-2-beta}
    \begin{aligned}
     w_{k+1}=w_k+(1-\alpha)(w_k-w_{k-1})-\eta_1(1+\alpha\beta) \nabla f\left(w_k\right)  -\beta \eta_1(1-\alpha)\underbrace{(\nabla f(w_k)-\nabla f(w_{k-1}))}_{\text{gradient correction}}.
    \end{aligned}
\end{equation}
with $\eta_1=\eta/(1+\alpha\beta)$ and $\beta=1 /(1-\alpha)$. Note that \eqref{nes-m-0-2-beta} can be transformed as: 
\begin{align}\label{nes-m-2}
\begin{dcases*}
m_k=(1-\alpha) m_{k-1} + \nabla f(w_k),\\
w_{k+1} =w_k- \eta_1 (m_k+\beta \nabla f(w_k)),
\end{dcases*}
\end{align}
which corresponds to the discretized form discussed in Section \ref{sec:ode-momentum}.  

We can observe that the term $\beta\nabla f(w_k)$  in \eqref{nes-m-2} (\ie the gradient correction term in \eqref{nes-m-0-2-beta}) arises from discretizing the Hessian damping term $\beta\nabla^2f(w_t)\dot{w}_t$ in Eq. \eqref{agf-2-order}. 

These insights lead to two key conclusions: \begin{enumerate} \item \textbf{Implicit Hessian Awareness:} Compared with Heavy ball momentum, Nesterov momentum introduces an additional gradient correction term that implicitly encapsulates Hessian information (\ie Hessian damping), which is instrumental in achieving acceleration. \item \textbf{Continuous-Discrete Correspondence:} This gradient correction term naturally arises from the discretization of the Hessian damping term found in the continuous-time limiting flow. \end{enumerate}

\paragraph{Further Improving Nesterov momentum} 
Although Nesterov momentum uses Hessian damping (via gradient correction) to achieve acceleration, the coefficient $\beta$ of this term is rigidly constrained by the momentum parameter (\ie $\beta=1/(1-\alpha)$). It is therefore natural to relax this constraint and search for an optimal $\beta$, a direction supported by recent theoretical results \cite{Attouch2022, xie2025ode}. In this work, we adopt this generalized perspective and incorporate richer second-order information, employing a larger $\beta$ in flat directions than in sharp directions to further enhance acceleration.

\subsection{Modeling AdEMAMix via a Third-Order ODE}\label{section-ademamix}
The AdEMAMix optimizer \cite{pagliardini2025the} builds upon the paradigm of AdamW by incorporating an additional ``slow'' momentum term. At the $k$-th iteration, the parameter update rule is given by:
\begin{equation}\label{ademamix-alg}
    \begin{aligned}
        m_k^{\mathrm{fast}}=&(1-\alpha_1)m_{k-1}^{\mathrm{fast}}+\alpha_1g_k,
        \\ m_k^{\mathrm{slow}}=&(1-\alpha_2)m_{k-1}^{\mathrm{slow}}+\alpha_2g_k,
        \\ v_k=&(1-\beta_{\mathrm{adam}})v_{k-1}+\beta_{\mathrm{adam}}g_k^{\odot2}
        \\ w_{k+1}=&w_k-\eta_k \frac{m_k^{\mathrm{fast}}+\kappa m_k^{\mathrm{slow}}}{v_k},
    \end{aligned}
\end{equation}
where $g_k$ denotes the stochastic gradient. For brevity, we omit the weight decay  and assume constant momentum hyper-parameters. Typically, $\alpha_1$ is set to approximately $0.1$, while $\alpha_2$ takes much smaller values, such as $10^{-4}$.  

To derive its continuous-time version, we first simplify the analysis by omitting the preconditioner $v_k$ and treating the learning rate $\eta_k$ as a constant $\eta$. This yields the following system of differential equations:
\begin{equation}
    \begin{aligned}
        \frac{d}{dt}m_t^{\mathrm{fast}}=&-\alpha_1m_{k-1}^{\mathrm{fast}}+\alpha_1 \nabla f(w_k),
        \\ \frac{d}{dt}m_t^{\mathrm{slow}}=&-\alpha_2m_{k-1}^{\mathrm{slow}}+\alpha_2\nabla f(w_k),
        \\ \frac{d}{dt}w_t=& -\eta  (m_t^{\mathrm{fast}}+\kappa m_t^{\mathrm{slow}}).
    \end{aligned}
\end{equation}
We note that applying a discretization with step size $h=1$ to the above system recovers the update rules in \eqref{ademamix-alg}.  Next, we compute the second and third-order time derivatives of the parameters, $\frac{d^2}{dt^2}w_t$ and $\frac{d^3}{dt^3}w_t$:
\begin{equation}\label{ademamxi-2-3-d}
    \begin{aligned}
   \frac{d^2}{dt^2}w_t=& -\eta  (-\alpha_1m_t^{\mathrm{fast}} -\alpha_2\kappa m_t^{\mathrm{slow}}+(\alpha_1+\kappa \alpha_2)\nabla f(w_t)), \\      \frac{d^3}{dt^3}w_t=& -\eta  ( \alpha_1^2m_t^{\mathrm{fast}}+\alpha_2^2\kappa m_t^{\mathrm{slow}}-(\alpha_1^2+\kappa \alpha_2^2)\nabla f(w_t)+(\alpha_1+\kappa \alpha_2)\nabla^2 f(w_t)\frac{d}{dt}w_t).
    \end{aligned}
\end{equation}
By combining \eqref{ademamxi-2-3-d} with the velocity equation $\frac{d}{dt}w_t = -\eta (m_t^{\mathrm{fast}} + \kappa m_t^{\mathrm{slow}})$ to eliminate the momentum variables $m_t^{\mathrm{fast}}$ and $m_t^{\mathrm{slow}}$, we arrive at the following third-order  ODE :

\begin{equation}\label{eq:third-order-ode-ademamix}
   \underbrace{\frac{d^3}{dt^3}w_t}_{\text{jerk}}+(\alpha_1+\alpha_2) \underbrace{\frac{d^2}{dt^2}w_t}_{\substack{\text{accelerated} \\ \text{velocity}}}+\alpha_1\alpha_2\underbrace{\frac{d }{dt }w_t}_{\text{velocity}}+\eta(\alpha_1+\kappa \alpha_2)\underbrace{\nabla^2 f(w_t)\frac{d}{dt}w_t}_{\text{Hessian damping}}+\eta (1+\kappa)\alpha_1\alpha_2\underbrace{\nabla f(w_t)}_{\text{gradient}} =0.
\end{equation}
Similar to \eqref{R_acc:flow}, the Riemannian counterpart of \eqref{eq:third-order-ode-ademamix} can be derived as
\begin{equation}
  \begin{aligned}
\nabla_{\dot{w}_t}\nabla_{\dot{w}_t}\dot{w}_t&+(\alpha_1+\alpha_2)\nabla_{\dot{w}_t}\dot{w}_t+\alpha_1\alpha_2\frac{d }{dt }w_t\\&+\eta(\alpha_1+\kappa \alpha_2)\frac{d}{dt}\operatorname{grad}f(w_t)+\eta (1+\kappa)\alpha_1\alpha_2\operatorname{grad} f(w_t) =0, 
  \end{aligned}
\end{equation}
where we use $\operatorname{Hess}(w_t)\dot{w}_t=\frac{d}{dt}\operatorname{grad}f(w_t)$.

\paragraph{Intuitive  Interpretation.}
Standard momentum methods (like  Heavy Ball) typically correspond to second-order ODEs, describing a particle with mass and friction. In contrast, Eq. \eqref{eq:third-order-ode-ademamix} reveals that AdEMAMix is governed by a \textbf{third-order} dynamics. The term $\frac{d^3}{dt^3}w_t$ introduces the concept of ``jerk'' (rate of change of accelerated velocity $\ddot{w}_t$) into the optimization trajectory.  The third-order structure    with Hessian  damping  allows the optimizer to regulate not only velocity but also its rate of change across multiple time scales. As a result, transient gradient fluctuations induced by rapidly varying curvature are attenuated before influencing parameter updates, enabling more stable and robust navigation of highly ill-conditioned landscapes.

\section{Foundations of Riemannian Geometry}\label{app:riemannian-geometry}
In this section, we provide a detailed introduction to the  foundation for Riemannian Geometry of $(\mathbb{R}^p, F)$ considered in the main text. The parameter space is   the Euclidean space $\mathbb{R}^p$, while the metric is the Riemannian metric induced by the  preconditioner $F(w)$, which is formally denoted as the manifold $\mathcal{M} = (\mathbb{R}^p, F)$.

\subsection{Riemannian Metric}
The core structure of our framework is based on Riemannian metric, a smooth matrix-value map $F: \mathcal{M}=\mathbb{R}^p \to \mathbb{S}_{++}^p$ that assigns an inner product to the tangent space $T_w\mathcal{M} \cong \mathbb{R}^p$ at each point $w\in \mathcal{M}$. For any tangent vectors $u, v \in T_w\mathcal{M}$, the inner product is defined as $\langle u, v \rangle_{F(w)} \coloneqq u^\top F(w) v$. 
This metric induces a   norm $\|u\|_{F(w)} = \sqrt{\langle u, u \rangle_{F(w)}}$ in the tangent space $T_w\mathcal{M}$, which measures the  magnitude  of perturbations $u$ relative to the local curvature $F(w)$. 

While the tangent space $T_w\mathcal{M}$ characterizes  velocities or perturbations, its dual, the cotangent space $T_w^*\mathcal{M}$, consists of differentials that act linearly on these velocities, naturally representing gradients of functions. Formally, $T_w^*\mathcal{M}$ is the dual space of $T_w\mathcal{M}$, consisting of all linear functionals $\xi: T_w\mathcal{M} \to \mathbb{R}$. It also holds that $T_w\mathcal{M} \cong \mathbb{R}^p$. The cotangent space   is equipped with the dual norm $\| u\|_{F(w)^{-1}} = \sqrt{u^\top F(w)^{-1} u}$ for $u  \in T^*\mathcal{M}$. In particular, the Riemannian metric $F(w)$ induces a natural isomorphism between the tangent and cotangent spaces: 
\begin{equation}
   F(w): T_w\mathcal{M}\to T_w^*\mathcal{M},\quad F(w)^{-1}: T_w^*\mathcal{M}\to T_w\mathcal{M}. 
\end{equation}   Under this identification, an element $\xi\in T_w^*\mathcal{M}$
 acts as a linear functional on the tangent space via the Euclidean inner product $\xi(u)=\xi^\top u$ for $u\in T_w\mathcal{M}$. Inner-product structures of $T_w\mathcal{M},T_w^*\mathcal{M}$ and between them are $\langle\cdot,\cdot\rangle_{F(w)}$, $\langle\cdot,\cdot\rangle_{F(w)^{-1}}$, $\langle\cdot,\cdot\rangle $ respectively.

Based on the Riemannian metric, we can define the length of a smooth curve $\gamma: [0, 1] \to \mathcal{M}$. The Riemannian length $L(\gamma)$ is given by integrating the local norm of the velocity vector $\dot{\gamma}(t)$:\begin{equation}L(\gamma) = \int_0^1 \|\dot{\gamma}(t)\|_{F(\gamma(t))}   dt = \int_0^1 \sqrt{\dot{\gamma}(t)^\top F(\gamma(t)) \dot{\gamma}(t)}   dt.
\end{equation}
The Riemannian distance (or geodesic distance) between two points $x, y \in \mathcal{M}=\mathbb{R}^p$ is defined as the infimum of lengths over all piecewise smooth curves connecting them:
\begin{equation}
d_F(x, y) = \inf_{\gamma \in D} { L(\gamma) },\end{equation}
where $D$ is the set containing all smooth curves $\gamma$ connecting $(x,y)$ with $\gamma(0)=x,\gamma(1)=y$.

The distance reduces to the Euclidean metric when $F=I$. Furthermore, when $(x,y)$ are sufficiently close locally, we have $d_F(x,y)\approx \|x-y\|_{F(x)}$. The Riemannian gradient $\operatorname{grad} f$ is an example of vector field.

\paragraph{Smooth Vector Fields and the Tangent Bundle.} With the tangent space $T_w\mathcal{M}$ defined at each point, we consider the collection of all tangent vectors across the manifold, known as the tangent bundle, denoted by $T\mathcal{M} \coloneqq \bigcup_{w \in \mathcal{M}} (\{w\} \times T_w\mathcal{M})$. The tangent bundle is itself a smooth manifold of dimension $2p$. A vector field $V$ is a mapping (formally called a \textit{section}) that assigns to every point $w \in \mathcal{M}$ a specific tangent vector $V(w) \in T_w\mathcal{M}$. The space of smooth vector fields, denoted by $\mathcal{C}^{\infty}(\mathcal{M}, T\mathcal{M})$, consists of all such fields $V$ where the assignment varies smoothly with $w$. In the global coordinate system of $\mathcal{M} \cong \mathbb{R}^p$, any vector field $V \in \mathcal{C}^{\infty}(\mathcal{M}, T\mathcal{M})$ can be uniquely expressed as a linear combination of the coordinate basis vectors $\{\frac{\partial}{\partial w_1}, \dots, \frac{\partial}{\partial w_p}\}$ (which form the standard basis for each $T_w\mathcal{M}$):\begin{equation}V(w) = \sum_{i=1}^p v_i(w) \frac{\partial}{\partial w_i},\end{equation} where each component function $v_i: \mathcal{M} \to \mathbb{R}$ is a smooth   scalar function. For a function $f$, $(Vf)|_w:=V(w) f(w)=\sum_{i=1}^p v_i(w) \frac{\partial}{\partial w_i}f(w)$ and $VU f:=V  (Uf)$ for vector field $U$.

\subsection{The Levi-Civita Connection}\label{Levi-Civita}
In Euclidean spaces, comparing vectors at different points is trivial because the tangent spaces are identical. On a curved manifold, tangent spaces $T_x\mathcal{M}$ and $T_{y}\mathcal{M}$ are distinct inner-product spaces. To differentiate vector fields, we require an affine connection $\nabla$, which provides a rule for  connecting  adjacent tangent spaces. \paragraph{Definition.} The Levi-Civita connection $\nabla_{(\cdot)}(\cdot):\mathcal{C}^{\infty}(\mathcal{M},T\mathcal{M})\times\mathcal{C}^{\infty}(\mathcal{M},T\mathcal{M})\to \mathcal{C}^{\infty}(\mathcal{M},T\mathcal{M})$ is the fundamental connection associated with the metric $F$. It is uniquely determined by two geometric conditions:

\begin{enumerate}\item Torsion-Free: For any vector fields $X, Y\in \mathcal{C}^{\infty}(\mathcal{M},T\mathcal{M})$, the connection is symmetric: $\nabla_X Y - \nabla_Y X = [X, Y]$, where $[X, Y]$ is the Lie bracket (in $\mathbb{R}^p$, defined by $[X,Y] = XY-YX $).

\item Metric Compatibility: The connection preserves the inner product structure. Formally, for vector fields $X, Y, Z$:
\begin{equation}X \langle Y, Z \rangle_F = \langle \nabla_X Y, Z \rangle_F + \langle Y, \nabla_X Z \rangle_F.\end{equation}\end{enumerate}

\paragraph{Coordinate Representation (Christoffel Symbols).}Since our manifold is globally $\mathbb{R}^p$, we can express the covariant derivative explicitly using the standard basis $\{e_1, \dots, e_p\}$. Let $\Gamma^k_{ij}(w)$ denote the Christoffel symbols of the second kind. The covariant derivative of a vector field $v$ along $u$ is given component-wise by:\begin{equation}(\nabla_u v)_k  = \sum_{i=1}^p u_i \frac{\partial v_k}{\partial w_i} + \sum_{i,j\le p} \Gamma^k_{ij}(w) u_i v_j.\end{equation}The first term captures the standard Euclidean directional derivative, while the second term corrects for the twisting of the coordinate system induced by the metric. The Christoffel symbols are derived from the metric $F(w)$:
\begin{equation}\Gamma^k_{ij} = \frac{1}{2} \sum_{l} (F^{-1})_{kl} \left( \frac{\partial F_{jl}}{\partial w_i} + \frac{\partial F_{il}}{\partial w_j} - \frac{\partial F_{ij}}{\partial w_l} \right).\end{equation}Here, $(F^{-1})_{kl}$ denotes the $(k,l)$-entry of the inverse metric matrix, and indices denote partial derivatives with respect to coordinates. The Levi-connection would reduce to the Euclidean directional derivative when $F=I$, as $\Gamma_{ij}^k=0$ in this case.

\subsection{Riemannian Gradient and Hessian}
With the metric and connection defined, we can rigorously define the following differential operators.

\paragraph{Riemannian Gradient.}The Riemannian gradient $\text{grad} f(w)$ is the unique tangent vector   representing the differential $df(w)$ via the metric:
\begin{equation}
    \langle \text{grad} f(w), u\rangle_{F(w)}=d f(w)[u]:=\lim_{\epsilon \to 0} \frac{f(w+\epsilon u)-f(w)}{\epsilon} \text{ for any } u\in \mathbb{R}^p.
\end{equation}
It has the closed form
\begin{equation}\text{grad} f(w) = F(w)^{-1} \nabla f(w),\end{equation}
where $\nabla f(w)$ is the standard Euclidean gradient. This confirms that the preconditioner $F^{-1}$ acts as the ``inverse metric" mapping cotangent vectors (gradients) to tangent vectors (directions).

\paragraph{Riemannian Hessian.}The Riemannian Hessian is a linear operator on the tangent space, defined as the covariant derivative of the gradient vector field:\begin{equation}\text{Hess} f(w)[u] = \nabla_u \text{grad} f(w), \quad \forall u \in T_w\mathcal{M}.\end{equation}Unlike the Euclidean Hessian $\nabla^2 f(w)$, the Riemannian Hessian accounts for the curvature of the manifold. Explicitly:
\begin{equation}
   \text{Hess} f(w) = F(w)^{-1} \left( \nabla^2 f(w) - \mathcal{K}(w) \right) ,
\end{equation} 
where 
$$[\mathcal{K}(w)]_{ij} = \sum_{k=1}^p \Gamma^k_{ij}(w) \frac{\partial f}{\partial w_k}(w).$$
When $F(w)$ changes slowly with $w$, we have $\Gamma_{ij}^k\approx0$ and   $\text{Hess} f(w)\approx F(w)^{-1} \nabla ^2 f(w)$.

\section{Transforming the Second Order RISHD into a First Order System}\label{app:tran-R-2nd-1st}

\begin{lemma}\label{levi-civita-expression}
 Let $\mathcal{M} = \mathbb{R}^p$ be endowed with a Riemannian metric tensor $F(w) \succ 0$. For any   vector fields $u,v\in \mathcal{C}^\infty(\mathcal{M},T\mathcal{M})$, we have the following expression for the Levi-Civita connection:
 \begin{equation}
     \nabla_u v = \underbrace{\partial_u v}_{\substack{\text{Euclidean} \\ \text{derivative}}} + \underbrace{\frac{1}{2}F^{-1}\left[(\partial_u F) v + (\partial_v F) u - \nabla_w (u^\top F v)\right]}_{\text{Geometric correction}}.
\end{equation}
\begin{proof} By Koszul formula (Corollary 5.11 in \cite{Lee2018IRM}), we have
\begin{equation}\label{kos-1}
\begin{aligned}
    2\left\langle \nabla_{u} v, h \right\rangle_F = \partial_u \left\langle v, h \right\rangle_F  + \partial_v \left\langle u, h \right\rangle_F  - \partial_h \left\langle u, v \right\rangle_F  + \left\langle [u, v], h \right\rangle_F - \left\langle [u, h], v \right\rangle_F - \left\langle [v, h], u \right\rangle_F.
\end{aligned}
\end{equation}
Substituting 
\begin{equation}
    \begin{aligned}
    \partial_u(v^\top F h) &= (\partial_u v)^\top F h + v^\top (\partial_u F) h + v^\top F (\partial_u h),
         \\ \partial_v(u^\top F h) &= (\partial_v u)^\top F h + u^\top (\partial_v F) h + u^\top F (\partial_v h),
         \\  -\partial_h(u^\top F v) &= -(\partial_h u)^\top F v - u^\top (\partial_h F) v - u^\top F (\partial_h v),
    \\ [ u , v ] & =\partial_u v-\partial_v u  ,
    \\ [u,h] & = \partial_u h - \partial_h u  , \\ [v,h] &= \partial_v h - \partial_h v  ,
    \end{aligned}
\end{equation}
into \eqref{kos-1} yields 
\begin{equation}
    2 h^\top F \nabla_u v = h^\top \left[ 2 F \partial_u v + (\partial_u F) v + (\partial_v F) u - \nabla_w (u^\top F v) \right] .
\end{equation}
We get the conclusion.

\end{proof}
    
\end{lemma}

\begin{proposition}\label{transform-2nd-to-1st-system}

    Let $\mathcal{M} = \mathbb{R}^d$ be endowed with a Riemannian metric tensor $F(w) \succ 0$. Then the RISHD 
\begin{align}  \label{prop-qrgf-form0}
     \nabla_{\dot{w}_t} \dot{w}_t + \alpha_t \dot{w}_t + \beta_t \nabla_{\dot{w}_t} \operatorname{grad} f(w_t) + \gamma_t \operatorname{grad} f(w_t) = 0.
    \end{align}
with $\alpha_t=\alpha-\dot{\eta}_t/\eta_t$, $\beta_t=\beta\eta_t$, $\gamma_t=\eta_t(\alpha\beta+1)$ is equivalent to the following first order system:
\begin{align}\label{2-ode0-appendix}
\begin{dcases*}
     \dot{w}_t =- \eta_t F(w_t)^{-1} (m_t+\beta\nabla f(w_t)) , \\
     \dot{m}_t=-\alpha m_t +\nabla f(w_t)+ R_t, 
\end{dcases*}
\end{align}
where we define $u_t=m_t+\beta\nabla f(w_t)$ for notational brevity and 
\begin{equation}\label{Rt-expression-append}
    \begin{aligned}
        R_t=&\frac{\eta_t}{2}\nabla_w (u_t^\top F^{-1}(w) m_t)|_{w=w_t}+\frac{\beta\eta_t}{2}\nabla F(w_t) [F^{-1}(w_t)u_t] F^{-1}(w_t)\nabla f(w_t)
        \\&-\frac{\beta\eta_t}{2}  \nabla F(w_t) [F^{-1}(w_t)\nabla f(w_t)]F^{-1}(w_t)u_t.
    \end{aligned}
\end{equation}
 
\begin{remark}
\eqref{prop-qrgf-form0} is equivalent to \eqref{R_acc:flow} as $\nabla_{\dot{w}_t} \operatorname{grad} f(w_t)=\operatorname{Hess}f(w_t)[\dot{w}_t]$.    
\end{remark}
\begin{proof}
First, we let  $m_t=-F(w_t)\dot{w}_t/\eta_t-\beta\nabla f(w_t)        $ and get
\begin{equation}
\begin{aligned}
 \dot{w}_t =&- \eta_tF(w_t)^{-1} (m_t  +\beta  \nabla f(w_t) ) , 
 \\\ddot{w}_t=& - \dot{\eta}_tF(w_t)^{-1} (m_t  +\beta  \nabla f(w_t) )  - \eta_t\frac{d}{dt}F(w_t)^{-1} (m_t  +\beta  \nabla f(w_t) ) 
 \\&-  \eta _tF(w_t)^{-1} (\dot{m}_t  +\beta  \nabla^2 f(w_t)\dot{w}_t ). 
 \end{aligned}
\end{equation}
By Lemma \ref{levi-civita-expression}, we have
\begin{equation}
  \nabla_{\dot{w}_t} \dot{w}_t= \ddot{w}_t + F(w_t)^{-1}\dot{F}(w_t)\dot{w}_t   - \frac{1}{2} F^{-1}(w_t) \nabla_w (\dot{w}_t^\top F(w) \dot{w}_t)|_{w=w_t}, 
\end{equation}
and
\begin{equation}
\begin{aligned}
   &\nabla_{\dot{w}_t} \operatorname{grad} f(w_t) \\=& \nabla_{\dot{w}_t} (F^{-1}(w_t) \nabla f(w_t)) 
  \\=& \frac{d}{dt}{(F^{-1}(w_t))} \nabla f(w_t) + F^{-1}(w_t) \nabla^2 f(w_t) \dot{w}_t 
  + \frac{1}{2}F^{-1}(w_t)\frac{d}{dt} F(w_t) F^{-1}(w_t) \nabla f(w_t)
  \\&+\frac{1}{2}F^{-1}(w_t)\nabla F(w_t)[F^{-1}(w_t) \nabla f(w_t)]\dot{w}_t
  -\frac{1}{2}F^{-1}(w_t)\nabla_w(\dot{w}_t^\top F (w)F^{-1}(w_t) \nabla f(w_t)|_{w=w_t} \\=&   F^{-1}(w_t) \nabla^2 f(w_t) \dot{w}_t - \frac{1}{2}F^{-1}(w_t)\frac{d}{dt} F(w_t) F^{-1}(w_t) \nabla f(w_t)
  \\&+\frac{1}{2}F^{-1}(w_t)\nabla F(w_t)[F^{-1}(w_t) \nabla f(w_t)]\dot{w}_t
 -\frac{1}{2}F^{-1}(w_t)\nabla_w(\dot{w}_t^\top F (w)F^{-1}(w_t) \nabla f(w_t))|_{w=w_t} .      
\end{aligned}
\end{equation}

Substituting them into \eqref{prop-qrgf-form0} and using $\nabla_w (\dot{w}_t^\top F(w) \dot{w}_t)=-\eta_t^2\nabla_w (u_t^\top F^{-1}(w) u_t)$ yields
\begin{equation}
\begin{aligned}
     \dot{m}_t=&-\alpha m_t+\nabla f(w_t)+\frac{\eta_t}{2}\nabla_w (u_t^\top F^{-1}(w) u_t)\big|_{w=w_t}
     \\&-\frac{\beta}{2}\nabla F(w_t)[\dot{w}_t]F^{-1}(w_t) \nabla f(w_t)+\frac{\beta}{2} \nabla F(w_t) [F^{-1}(w_t)\nabla f(w_t)]\dot{w}_t
    \\&-\frac{\beta}{2} 
\nabla_w \left( \dot{w}_t^\top F(w)F^{-1}(w_t)\nabla f(w_t) \right) \big|_{w=w_t}.
    \end{aligned}
\end{equation}
We get the conclusion.
\end{proof}

\end{proposition}

\begin{remark}
In general, we assume that $F(w)$ varies slowly with respect to $w$ to ensure the stability of the preconditioner. This assumption is practically justified by the common use of Exponential Moving Average (EMA) with a decay coefficient close to $1$, which enforces a smooth evolution of the estimated preconditioner. Consequently, based on the expression of $R_t$ in \eqref{Rt-expression-append}, this term is typically negligible relative to $m_t$ and $\nabla f(w_t)$, as it scales with the small step size $\eta_t$ and the limited variation $\|  \nabla F(w)\|_{\mathrm{op}}$.
\end{remark}

\section{Theoretical Analysis on Quadratic Objectives: An   Illustrative Example}
\label{app:Illustrative_analysis}

To demonstrate the acceleration mechanism of \ours, we analyze the discrete-time dynamics of the update rule \eqref{2-ode0-d}, which corresponds to non-accelerated adaptive optimizers, on a quadratic objective function:
\begin{equation}
  f(w)=\frac{1}{2}\sum_{i=1}^p \lambda_i w_i^2+b^\top w.
\end{equation}
We assume the eigenvalues are sorted as $\lambda_1 > \dots > \lambda_k > 0 \ge \dots \ge \lambda_p$, representing an anisotropic landscape where negative $\lambda_i$ correspond to local nonconvex directions. For simplicity, we set the step size $h=1$ and the preconditioner $F(w)=I_p$. Since the Hessian is diagonal, the dynamics decouple across coordinates.

Let $w_\star = -[b_1/\lambda_1, \dots, b_p/\lambda_p]^\top$ denote the stationary point. To characterize the trajectory of the error term $e_k = w_k - w_\star$, we eliminate the auxiliary momentum variable $m_k$.
From \eqref{2-ode0-d} with $F=I_p$, the $i$-th coordinate of $m_k$ can be expressed in terms of $e_k$ as:
\begin{equation}
    \eta (m_k)_i = (1 - \eta\beta\lambda_i)(e_k)_i - (e_{k+1})_i, \quad 1 \le i \le p. \label{eq:m_iso}
\end{equation}
Applying a time shift to \eqref{eq:m_iso}, we obtain:
\begin{equation}
    \eta (m_{k-1})_i = (1 - \eta\beta\lambda_i)(e_{k-1})_i - (e_{k})_i. \label{eq:m_iso_prev}
\end{equation}
Combining these relations leads to a homogeneous second-order linear recurrence:
\begin{equation}
    (e_{k+1})_i - 2T_i (e_k)_i + D_i (e_{k-1})_i = 0,
\end{equation}
where the coefficients are defined as:
\begin{align}
  T_i =T(\lambda_i):= 1 - \frac{\alpha + \eta\lambda_i(\beta+1)}{2}, \quad
  D_i = D(\lambda_i):= (1-\alpha)(1 - \eta\beta\lambda_i).
\end{align}
The characteristic equation is given by $r^2 - 2T_ir + D_i = 0$, with roots:
\begin{equation}
    r_{i,1,2} = T_i \pm \sqrt{T_i^2 - D_i}.
\end{equation}
Based on the discriminant $\Delta_i = T_i^2 - D_i$, the dynamics fall into three regimes, yielding the following closed-form solutions ($c_{i,1}, c_{i,2}$ are constants determined by initialization):
\begin{align}
    w_{k,i} = -\dfrac{b_i}{\lambda_i} +
    \begin{cases}
      c_{i,1} r_{i,1}^k + c_{i,2} r_{i,2}^k, & \Delta_i > 0\ \text{(\textbf{Overdamped})}, \\
     (c_{i,1} + c_{i,2} k) D_i^{k/2}, & \Delta_i = 0\ \text{(\textbf{Critically damped})}, \\
      D_i^{k/2} \left[ c_{i,1} \cos(k\theta_i) + c_{i,2} \sin(k\theta_i) \right], & \Delta_i < 0\ \text{(\textbf{Underdamped})},
    \end{cases}
\end{align}
where $\theta_i = \arccos(T_i/\sqrt{D_i})$.

\paragraph{Stability Analysis.}
Convergence in convex directions requires the magnitude of all characteristic roots to be at most $1$. According to the Jury Stability Criterion, the necessary and sufficient conditions are $|P_i(0)| \le 1$, $P_i(1) \ge 0$, and $P_i(-1) \ge 0$, where $P_i(z) = z^2 - 2T_i z + D_i$. Solving these inequalities yields the stability bound:
\begin{equation}\label{quatratic-stab}
    \eta \le \frac{2(2-\alpha)}{\lambda_1 \max\{1 + 2\beta - \alpha\beta, \, 2\beta(1-\alpha)\} }.
\end{equation}
Ideally, this condition confirms that the maximum allowable step size is governed by the sharpest direction $\lambda_1$.

\paragraph{Regime Classification.}
We next analyze how the curvature $\lambda_i$ determines the dynamic regime. The equation $T(\lambda)^2 = D(\lambda)$ with respect to $\lambda$:
\begin{equation}
  \left(1 - \frac{\alpha + \eta\lambda(\beta+1)}{2}\right)^2 = (1-\alpha)(1 - \eta\beta\lambda )
\end{equation}
possesses two real positive roots, denoted $\tilde{\lambda}_1 < \tilde{\lambda}_2$. By   the properties of   quadratic functions, we refer eigenvalues $\lambda$ satisfying $\lambda < \tilde{\lambda}_1$ to \emph{flat directions (parts)}, which fall into the \textbf{overdamped regime}. Conversely, eigenvalues $\lambda \ge \tilde{\lambda}_1$ correspond to \emph{sharp directions}, whose dynamic regimes depend on the specific magnitude of the curvature (typically transitioning into the underdamped regime).

\paragraph{Acceleration Mechanism of \ours.}
Here we suppose that the stability condition \eqref{quatratic-stab} holds. Since \ours explicitly   increases $\eta$ and $\beta$  in flat directions, we focus our analysis on the regime $\lambda_i < \tilde{\lambda}_1$. In this regime, the dynamics are overdamped, and dominated by the larger root $r_1 = T + \sqrt{T^2 - D}$. We demonstrate that our parameter adjustment modifies $r_1$ to facilitate escape from nonconvex regions and foster convergence in convex regions.

We analyze the sensitivity of $r_1$ to $\beta$ and $\eta$ case by case:

\begin{itemize}
\item \textbf{Case 1: Flat Nonconvex Directions ($\lambda < 0$).}
  The coefficients become $T = 1 - \frac{\alpha}{2} + \frac{\eta |\lambda|(\beta+1)}{2}$ and $D = (1-\alpha)(1 + \eta\beta |\lambda|)$.
Clearly, $T > 1-\alpha$, and both $T$ and $D$ are strictly increasing functions of $\beta$ and $\eta$.
Consider the partial derivatives:
\begin{equation}
    \frac{\partial r_1}{\partial \beta} = \frac{\partial T}{\partial \beta} + \frac{1}{2\sqrt{T^2 - D}} \left( 2T\frac{\partial T}{\partial \beta} - \frac{\partial D}{\partial \beta} \right).
\end{equation}
Since $T > 0$, and noting that:
\begin{equation}
    2T\frac{\partial T}{\partial \beta} - \frac{\partial D}{\partial \beta} = T \eta|\lambda| - (1-\alpha)\eta|\lambda| > 0,
\end{equation}
(and similarly for $\eta$), we conclude that increasing $\beta$ or $\eta$ yields a strictly larger $r_1$ (with $r_1 > 1$). This increases the exponential divergence rate along negative curvature directions, thereby facilitating rapid escape.

\item \textbf{Case 2: Flat Convex Directions ($0 < \lambda < \tilde{\lambda}_1$).}
In this regime, the convergence rate is governed by how quickly $r_1$ decays to zero. To accelerate convergence, we must \textbf{minimize} $r_1$.
The coefficients are $T = 1 - \frac{\alpha + \eta\lambda(\beta+1)}{2}$ and $D = (1-\alpha)(1 - \eta\beta\lambda)$.
To determine the effect of $\beta$ on $r_1$, we differentiate the characteristic equation
$r_1^2 - 2T r_1 + D = 0$ implicitly with respect to $\beta$:
\begin{equation}
    2r_1 \frac{\partial r_1}{\partial \beta} - 2\frac{\partial T}{\partial \beta} r_1 - 2T \frac{\partial r_1}{\partial \beta} + \frac{\partial D}{\partial \beta} = 0,\, implying\quad
    \frac{\partial r_1}{\partial \beta} = \frac{2r_1 \frac{\partial T}{\partial \beta} - \frac{\partial D}{\partial \beta}}{2(r_1 - T)}.
\end{equation}
Substituting $\frac{\partial T}{\partial \beta} = -\frac{\eta\lambda}{2}$ and $\frac{\partial D}{\partial \beta} = -(1-\alpha)\eta\lambda$ gives
\begin{equation}
    \frac{\partial r_1}{\partial \beta} = \frac{\eta\lambda\bigl((1-\alpha) - r_1\bigr)}{2\sqrt{T^2 - D}}.
\end{equation}
Since $\lambda>0$, the sign of $\frac{\partial r_1}{\partial \beta}$ is opposite to the sign of $r_1 - (1-\alpha)$.
Evaluating the characteristic polynomial at $1-\alpha$ yields
$(1-\alpha)^2 - 2T(1-\alpha) + D = (1-\alpha)\eta\lambda > 0$.
Because the leading coefficient is positive, the two roots lie on the same side of $1-\alpha$;
their average is $\frac{r_1+r_2}{2} = T$. Hence $r_1 > 1-\alpha$ if  $T > 1-\alpha$.
The condition $T > 1-\alpha$ reduces to $\eta\lambda(\beta+1) < \alpha$, which is naturally satisfied for sufficiently small $\lambda$ (i.e., in flat directions) under the stability bound \eqref{quatratic-stab}.
Consequently, we have $r_1 > 1-\alpha$, implying $\frac{\partial r_1}{\partial \beta} < 0$.
Thus, increasing $\beta$ strictly decreases the dominant root, directly accelerating convergence along these flat convex directions.

We now verify the analogous property for $\eta$. Similar to $\beta$ we have:
\begin{equation}
    \frac{\partial r_1}{\partial\eta} = \frac{2r_1 \frac{\partial T}{\partial \eta} - \frac{\partial D}{\partial \eta}}{2(r_1 - T)}.
\end{equation}
Using $\frac{\partial T}{\partial \eta} = -\frac{\lambda(\beta+1)}{2}$ and $\frac{\partial D}{\partial \eta} = -(1-\alpha)\beta\lambda$, the numerator simplifies to
$\lambda\bigl[(1-\alpha)\beta - r_1(\beta+1)\bigr]$, yielding
\begin{equation}
    \frac{\partial r_1}{\partial\eta} = \frac{\lambda\bigl[(1-\alpha)\beta - r_1(\beta+1)\bigr]}{2\sqrt{T^2 - D}}.
\end{equation}
 Since $r_1 > 1-\alpha$ and $\beta>0$, we have
\[
r_1 > 1-\alpha > \frac{\beta}{\beta+1}(1-\alpha),
\]
which implies the numerator is strictly negative, hence $\frac{\partial r_1}{\partial\eta} < 0$. Thus, increasing $\beta$ or $\eta$ strictly decreases the dominant root $r_1$, directly accelerating convergence along these flat convex directions.
\end{itemize}

\section{Some Common Adaptive Optimizers for LLM Pre-training}\label{app:adap-opt-example}

Here we list some common optimizers in LLM pre-training. They can be subsumed by discretizing the ODE framework \eqref{2-ode0-d}. We omit gradient clipping and AdamW-type bias correction coefficient $\frac{1-\theta^t}{\sqrt{1-\beta_v^t}}$ for notational brevity. 
\begin{algorithm}[H]
    \caption{AdamW}
    \label{alg:adamw}
    \begin{algorithmic} 
        \REQUIRE Hyper-parameters: $\theta$, $\beta_v$, $\epsilon$, $\{\eta_t\}$, $\lambda$ 
        \STATE Initialize  $m = 0$, $v = 0$, $t = 0$
        \FOR{$t = 1, 2, \dots$}
            \STATE Compute stochastic gradient $g_t$ at current parameters $w_t$.
              
            \STATE  $m_t = \theta   m_{t-1} + (1 - \theta) g_t$
            \STATE  $v_t = \beta_v   v_{t-1} + (1 - \beta_v)   g_t^{\odot 2}$
  
            \STATE   
           $w_{t+1} = w_t - \eta_t \left( \frac{ {m}_t}{\sqrt{ {v}_t} + \epsilon} + \lambda w_t \right)$
        \ENDFOR
    \end{algorithmic}
\end{algorithm}

\begin{algorithm}[H]
\caption{Lion}
\label{alg:lion}
\begin{algorithmic} 
\REQUIRE $\theta, \beta_v, \{\eta_t\}, \epsilon, \lambda $
\STATE Initialize  $m = 0$   
\FOR{$t = 1, 2, \dots$}
    \STATE Compute gradient $g_t$ at current parameters $w_t$
    
    \STATE   $u_t = \theta m_{t-1} + (1 - \theta) g_t$
    \STATE   $m_{t} = \beta_v m_{t-1} + (1 - \beta_v) g_t$
    \STATE  $w_{t+1} = w_t - \eta_t \operatorname{sign}( u_t) - \eta_t \lambda w_t$
\ENDFOR
\end{algorithmic}
\end{algorithm}

\begin{algorithm}[H]
\caption{MARS}
\label{alg:mars}
\begin{algorithmic}
\REQUIRE $\theta, \beta_v, \gamma, \epsilon, \{\eta_t\}, \lambda $
\STATE Initialize $m = 0$, $v = 0$, $g_{0} = 0$  
\FOR{$t = 1, 2, \dots$}
    \STATE Compute stochastic gradient $g_t$ at current parameters $w_t$
    \STATE  $\displaystyle c_t = g_t + \gamma \frac{\theta}{1 - \theta} (g_t - g_{t-1})$
    \STATE   $m_t = \theta m_{t-1} + (1 - \theta)  {c}_t$
    \STATE   $v_t = \beta_v v_{t-1} + (1 - \beta_v)  {c}_t^{\odot 2}$
    \STATE  $\displaystyle w_{t+1} = w_t - \eta_t \frac{ {m}_t}{\sqrt{ {v}_t} + \epsilon} - \eta_t \lambda w_t$
\ENDFOR
\end{algorithmic}
\end{algorithm}

\begin{remark}
 In MARS, the momentum term is updated as: \begin{equation}\label{mars-m-0}
        m_t=\theta m_{t-1}+(1-\theta)g_t+\gamma  \theta(g_t-g_{t-1}).
    \end{equation}
It can be transformed to the formulation in \eqref{2-ode0-d}  as:
\begin{equation}
    \begin{aligned}
      s_t=  \theta s_{t-1}+ g_t, \quad \tilde{m}_t=s_t+\frac{\gamma  }{(1-\gamma)(1-\theta)} g_t,
    \end{aligned}
\end{equation}
with $m_t=(1-\gamma)(1-\theta)\tilde{m}_t$. Note that $\gamma \neq 1$; otherwise,  \eqref{mars-m-0} would reduce to 
    $m_t = \theta m_{t-1} - \theta g_{t-1}$, thereby completely discarding the current gradient information $g_t$.
\end{remark}

\begin{algorithm}[H]
\caption{Muon}
\label{alg:muon}
\begin{algorithmic}
\REQUIRE $\theta_{\operatorname{muon}}, \eta, \epsilon, \theta_{\operatorname{adamw}}, \beta_v,   \{\eta_t\}, \lambda $
\STATE Initialize state $m$
\FOR{$t = 1, 2, \dots$}
    \STATE \textbf{For weights in Output, Embedding, and  Norm layers:}
    \STATE \quad Update using AdamW  with hyper-parameters $\theta_{\operatorname{adamw}}$, $\beta_v$, $\epsilon$, $\eta_t$, $\lambda$.
    
    \STATE \textbf{For weight  matrices in transformer layers:}
    \STATE \quad Compute stochastic gradient $g_t$ at current parameters $w_t$
     
    \STATE \quad $m_t = \theta_{\operatorname{muon}} m_{t-1} + g_t$
    \STATE \quad $u_t = \theta_{\operatorname{muon}} m_t + g_t$

    \STATE \quad  $w_{t+1} = w_t - \operatorname{scale}\cdot\eta_t \operatorname{NS}(u_t) - \eta_t \lambda w_t$, where $\operatorname{scale}=0.2\sqrt{\max\{m,n\}}$ for $w\in \mathbb{R}^{m\times n}$.
\ENDFOR
\end{algorithmic}
\end{algorithm}

\begin{algorithm}[H]
\caption{SOAP}
\label{alg:soap}
\begin{algorithmic} 
\REQUIRE $\theta, \beta_v, \beta_{\mathrm{shampoo}}, k, \epsilon ,\lambda$
\STATE Initialize state: $m = 0$, $v = 0$, $L = 0$, $R = 0$, $Q_l$, $Q_r$

\FOR{$t = 1, 2, \dots$}
\STATE \textbf{For weights in Output, Embedding, and  Norm layers:}
    \STATE \quad Update using AdamW  with hyper-parameters $\beta_1$, $\beta_2$, $\epsilon$, $\eta_t$, $\lambda$.
    
\FOR{\textbf{weight matrices in transformer layers:}}
    \STATE Compute stochastic gradient $G_t$ at current parameters $W_t$
 
    \STATE   $\displaystyle G_t^{\text{rot}} = Q_l^\top G_t Q_r$
    \STATE   $M_t = \theta M_{t-1} + (1 - \theta) G_t$
    \STATE   $V_t = \beta_v V_{t-1} + (1 - \beta_v) (G_t^{\text{rot}})^{\odot 2}$
     
    \STATE   $\displaystyle W_{t+1} = W_t - \eta_t Q_l  \left( \frac{Q_l^\top M_t Q_r}{\sqrt{V_t} + \epsilon} \right) Q_r^{\top}- \eta_t\lambda W_t$
    \STATE   $L_t =  \beta_{\mathrm{shampoo}}  L_{t-1} + (1 - \beta_{\mathrm{shampoo}}) G_t G_t^{\top} $
    \STATE   $R_t =  \beta_{\mathrm{shampoo}}  R_{t-1} + (1 - \beta_{\mathrm{shampoo}}) G_t^{\top} G_t $
    \IF{$t \bmod k = 0$}
        \STATE $Q_l \leftarrow \text{QR}(L_t Q_l)$ \COMMENT{QR decomposition}
        \STATE $Q_r \leftarrow \text{QR}(R_t Q_r)$
    \ENDIF
\ENDFOR
\ENDFOR
\end{algorithmic}
\end{algorithm}

\subsection{Preconditioner-Estimated Riemannian Metrics   for Common Adaptive Optimizers} \label{app:r-exa-pre}
In this section, we analyze the implicit Riemannian metric structure $F$ estimated by the practical preconditioner   for each parameter block in \cref{sec:uni-dis-form}.   Establishing these explicit, deterministic metric forms relies on certain heuristic approximations. To make the analysis tractable, we operate under the assumptions that the stochastic gradients are heavily noise-dominated and that the underlying noise covariance $F(w)$ varies slowly with respect to the parameters $w$.

Let $G \in \mathbb{R}^{m \times n}$ denote the stochastic gradient in matrix form, and let $g = \operatorname{vec}(G) \in \mathbb{R}^{mn}$ be its vectorized counterpart.  A prevalent choice in the literature is to estimate the whitening metric $\hat{F}(w)=(\mathbb{E}[gg^\top])^{1/2}\in\mathbb{R}^{mn\times mn}$ \cite{YANG2008232,frans2025stablewhiteningoptimizerefficient}. This formulation has demonstrated remarkable efficacy in neural network training and exhibits a close theoretical connection to the Gauss-Newton approximation of the Hessian. There are also alternative methods   that explicitly estimate the Hessian matrix.  Building on these premises, we demonstrate below how the practical preconditioners $F_k$ employed in common adaptive optimizers can be viewed as approximations of their corresponding deterministic metric forms $F(w_k)$. Throughout this section, we frequently utilize the Kronecker product identity $\operatorname{vec}(ABC) = (C^\top \otimes A)\operatorname{vec}(B)$ to bridge the matrix and vector formulations.

\paragraph{EMA-based Preconditioners.}
AdamW employs an Exponential Moving Average (EMA) to update its preconditioner $F_k^{-1}$ via $F_k=\operatorname{diag} v_k^{1/2}$ with $v_{k} = \beta_v v_{k-1} + (1-\beta_v) g_k \odot g_k$ .  This expands to $v_k = (1-\beta_v) \sum_{i=0}^k \beta_v^{k-i} g_i \odot g_i$, where $\beta_v$ is typically close to $1$.
Observing that $(1-\beta_v) \sum_{i=k-t}^k \beta_v^{k-i} \approx 1$ for a  large window $t$, and assuming that the stochastic gradients $g_i$ are approximately independent samples while $\hat{F}(w)$ changes slowly with respect to $w$, we can deduce that $v_{k}^{1/2}$ provides a stable and efficient approximation of $F^{\operatorname{adam}}(w) = \operatorname{diag} \hat{F}(w) $.
This approximation logic can directly extend  to other optimizers that accumulate preconditioners via EMA, such as SOAP. Specifically, for SOAP, from \cref{alg:soap} we have
\begin{equation}
    F_k=(Q_r \otimes Q_l) \left( \operatorname{diag} \operatorname{vec}(V_k) \right)^{\frac{1}{2}} (Q_r \otimes Q_l)^\top, \,  \text{with } V_k=\beta_v V_{k-1} +(1-\beta_v) (G_k^{\text{rot}})^{\odot2},
\end{equation}
where $Q_l,Q_r$ are updated in a lazy fashion, and $G_k^{\text{rot}}=Q_l^\top G_k Q_r$. This gives 
\begin{equation}
F^{\operatorname{soap}}(w) = (Q_r \otimes Q_l) \left( \operatorname{diag} \mathbb{E}[g_{\operatorname{rot}} g_{\operatorname{rot}}^\top] \right)^{\frac{1}{2}} (Q_r \otimes Q_l)^\top,
\end{equation}
where $Q_l$ and $Q_r$ are the eigenvectors of $\mathbb{E}[GG^\top]$ and $\mathbb{E}[G^\top G]$, respectively, and $g_{\operatorname{rot}} = \operatorname{vec}(Q_l^\top G Q_r)$. Thus, $F^{\operatorname{soap}}(w)$ approximates $\hat{F}(w)$ via a structured, two-sided Kronecker factorization.

Alternatively, some optimizers construct preconditioners that directly approximate the Hessian rather than the whitening metric. For instance, AdaHessian \cite{yao2021adahessianadaptivesecondorder} and Sophia \cite{liu2024sophia} directly estimate quantities related to the diagonal of the Hessian using techniques such as Hutchinson's method and EMA. Their corresponding theoretical preconditioner takes the form $F(w) = |\operatorname{diag}(\nabla^2 f(w))|^p$, where $p > 0$ is a power parameter, and $|\cdot|$ denotes the element-wise absolute value.

\paragraph{Momentum-based Preconditioners.}
We take Muon as a representative example. At the $k$-th iteration, the momentum is updated as $M_k = \theta_{\operatorname{muon}} M_{k-1} + G_k$, where $G_k \in \mathbb{R}^{m \times n}$ (assuming $m \ge n$). We have $F_k=(M_k^\top M_k)^{1/2}\otimes I_m$ (or replacing $M_k$ with its Nesterov-accelerated  version).  Expanding this recurrence yields $M_k = \sum_{i=0}^k \theta_{\operatorname{muon}}^{k-i} G_i$.
Analyzing the second moment, we have:
\begin{equation}
\mathbb{E}[M_k^\top M_k] = \sum_{i,j \le k} \theta_{\operatorname{muon}}^{2k - (i+j)} \mathbb{E}[G_i^\top G_j] \approx \sum_{i=0}^k \theta_{\operatorname{muon}}^{2(k-i)} \mathbb{E}[G_i^\top G_i].
\end{equation}
Here, we assume that in the middle to late stages of training, the deterministic gradient component is negligible compared to the stochastic noise (or second moment magnitude), and that gradient noise is approximately independent across different steps $i, j$ ($i \neq j$).
Applying the same EMA perspective as used for AdamW and assuming  the  noise structure varies slowly with respect to $w$, $M_k^\top M_k$ serves as a stable estimator of $\frac{1}{1-\theta_{\operatorname{muon}}^2} \mathbb{E}[G^\top G]$. Similarly, for the Nesterov-corrected momentum $\widetilde{M}_k = M_k + G_k / (1-\theta_{\operatorname{muon}})$, we can similarly get that $\widetilde{M}_k^\top \widetilde{M}_k$ approximates $c \mathbb{E}[G^\top G]$ for some constant scaling factor $c$.
Consequently, the implicit metric for Muon can be characterized as $F^{\operatorname{muon}}(w) \propto (\mathbb{E}[G^\top G])^{\frac{1}{2}} \otimes I_m$, representing a single-sided Kronecker-factored approximation of $\hat{F}(w)$.
This derivation extends naturally to element-wise operations on momentum; for instance, Lion's update rule implies a metric $F^{\operatorname{lion}}(w) \propto \operatorname{diag}(\hat{F}(w))$.

\section{Proof}\label{app:proof}

\paragraph{Notations}
A map is said to be smooth if it is infinitely differentiable. Given a matrix-valued function $F: \mathbb{R}^p\to \mathbb{R}^{p\times p}$, its directional derivative is written as $\nabla F(w)[v]=\lim_{\epsilon \to 0} (F(w+\epsilon v)-F(w))/\epsilon$. For notational brevity, we define $F_s(w)=P_s(w)F(w)P_s(w)$, $F_f(w)=P_f(w)F(w)P_f(w)$  and the corresponding pseudo inverse $F_s^{\dagger}(w)=P_s(w)F^{-1}(w)P_s(w)$, $F_f^{\dagger}(w)=P_f(w)F^{-1}(w)P_f(w)$. Denote $\kappa_{F_s}:=\rho/\lambda_{F_s}$, $\kappa_{F}:=\rho/\lambda_{F}$.    The full version of Assumption \ref{ass:rsc}.\ref{ass:spectral} concerning the operator norm bounds on  $\nabla P_s$ and $\nabla F$ is as follows.  We further define $\delta=\sup_{w\in U}\|\nabla P_s(w)\|_{\mathrm{op}}$ and $\delta_F=\sup_{w\in U}\|\nabla F(w)\|_{\mathrm{op}}$. Then $\max\{\delta, \delta_F\} \le \varepsilon$, and for any $u,v,w\in \mathbb{R}^p$, we have (according to Assumption \ref{ass:spectral}) 
\begin{equation}\label{nabla-Ps}
\begin{aligned}
 \|\nabla P_s(w)[u]v\|_{F(w)}&\le \frac{\delta}{G} \|u\|_{F(w)} \|v\|_{F(w)}, 
 \\ \|\nabla F(w)[u]v\|_{F^{-1}(w)}&\le \frac{\delta_F}{G} \|u\|_{F(w)} \|v\|_{F(w)}.  
\end{aligned}  
\end{equation}

\subsection{Landscape Analysis}
In this subsection we mainly focus on properties of $\Phi$ and $\mathcal{R}$.
\begin{lemma}\label{F_s-op-bound}
Suppose that Assumption \ref{ass:rsc} holds. Define $\delta_{F_s}=2\delta+\delta_F$. Then we have 
\begin{equation}
   \max\{ \|\nabla F_s^{{\dagger}}(w)[u]v\|_{F(w)},\|\nabla F_f^{{\dagger}}(w)[u]v\|_{F(w)}\}\le \frac{\delta_{F_s}}{G} \|u\|_{F(w)} \|v\|_{F^{-1}(w)},
\end{equation}
\begin{equation}
        \max\{\|\nabla F_s(w)[u]v\|_{F(w)},\|\nabla F_f(w)[u]v\|_{F(w)}\}\le \frac{\delta_{F_f}}{G} \|u\|_{F(w)} \|v\|_{F(w)}.
\end{equation}  

    \begin{proof}
By \eqref{nabla-Ps}, we have       \begin{equation}\label{p-deltafinvs-u-v}
    \begin{aligned}
        v^\top \nabla(P_sF^{-1}P_s)(w)[u]v=&2  v^\top \nabla P_s(w)[u]F^{-1}(w)P_s(w)v+v^\top P_s(w)\nabla(F^{-1})(w)[u]P_s(w)v
        \\\le&2 \|v\|_{F(w)^{-1}}\|\nabla P_s(w)[u] F(w)^{-1}P_s(w)v\|_{F(w)}\\&+\|P_s(w)F(w)^{-1}v\|_{F(w)}\|\nabla F(w)[u] F(w)^{-1}P_s(w)v\|_{F(w)^{-1
        }}
    \\\le&\frac{\delta_{F_s}}{\|\nabla f(w)\|_{F(w)^{-1}}}\|v\|_{F(w)^{-1}}\|v\|_{F_s^\dagger(w)}\|u\|_{F(w)},
        \end{aligned}
\end{equation}
where the first inequality uses $\nabla(F^{-1})(w)[u]=-F^{-1}(w)\nabla F(w)[u]F^{-1}(w)$. 

Similarly, for any $p$, 
  \begin{equation}\label{p-deltafs-u-v}
    \begin{aligned}
        &p^\top \nabla(P_sFP_s)(w)[u]v
        \\=&  p^\top \nabla P_s(w)[u]F(w)P_s(w)v+  p^\top  P_s(w)F(w) \nabla P_s(w)[u]v+p^\top P_s(w)\nabla F(w)[u]P_s(w)v
        \\\le& \|\nabla P_s(w)[u]v\|_{F(w)}\| F(w)P_s(w)p\|_{F(w)^{-1}}+ \|\nabla P_s(w)[u]p\|_{F(w)}\| F(w)P_s(w)v\|_{F(w)^{-1}}
\\&+\|P_s(w)p\|_{F(w)}\|\nabla F(w)[u] P_s(w)v\|_{F(w)^{-1}}
    \\\le&\frac{\delta(\|v\|_{F_s(w)}\|p\|_{F(w)}+\|v\|_{F(w)}\|p\|_{F_s(w)})+\delta_F\|v\|_{F_s(w)}\|p\|_{F_s(w)}}{\|\nabla f(w)\|_{F(w)^{-1}}}\|u\|_{F(w)}.
        \end{aligned}
\end{equation}
    \end{proof}
\end{lemma}

Using $\nabla P_f=-\nabla P_s$, we can similarly get
\begin{equation}\label{p-deltafinvf-u-v}
    \begin{aligned}
        v^\top \nabla(P_fF^{-1}P_f)(w)[u]v\le&\frac{\delta_{F_s}}{\|\nabla f(w)\|_{F(w)^{-1}}}\|v\|_{F(w)^{-1}}\|v\|_{F_f^\dagger(w)}\|u\|_{F(w)},
        \end{aligned}
\end{equation}
and
\begin{equation}\label{p-deltaff-u-v}
    \begin{aligned}
        v^\top \nabla(P_fFP_f)(w)[u]v\le&\frac{\delta_{F_s}}{\|\nabla f(w)\|_{F(w)^{-1}}}\|v\|_{F(w)}\|v\|_{F_f(w)}\|u\|_{F(w)}.
        \end{aligned}
\end{equation}

\begin{lemma}
Suppose Assumptions \ref{ass:river} and \ref{ass:rsc} hold. For any $w\in U$, we have
\begin{equation}\label{nabla-phi-Fs-nabla-0}
        \nabla \Phi(w) P_s(w) F^{-1}(w) \nabla f(w) = 0,
    \end{equation} 
    and $\operatorname{Range } \nabla \Phi(w) \subset T_{\Phi(w)}\mathcal{R}$. Especially, if we further assume that $ \max\{\delta_{F},\delta\}<\frac{1}{3}\lambda_{H_{F,s}}$, then for any $z \in \mathcal{R}$, $\nabla \Phi(z)$  is the oblique projection onto $T_z\mathcal{R}$ in the direct sum decomposition $\mathbb{R}^p = T_z\mathcal{R} \oplus \operatorname{Range } P_s(z)$.

\begin{proof}
First, note that $\Phi(w)=\lim_{n\to \infty} \phi_1^{(n)}(w)$, where $\phi_1(w):=\psi_t(w)|_{t=1}$ denotes the time-1 map of the flow. Since $F
$ and $f$ are smooth, $\phi_1$ is smooth. Then the smoothness of $\Phi$ follows immediately from  Theorem 5.1 in \cite{https://doi.org/10.1112/jlms/s2-27.2.356}. For any $w\in U$, noting that \begin{equation}
  -\nabla \Phi(w) P_s(w) F^{-1}(w) \nabla f(w)=\frac{d}{dt } \Phi(\psi_t(w))|_{t=0}=\frac{d}{dt } \Phi(w)|_{t=0}=0,      
 \end{equation}
 we get \eqref{nabla-phi-Fs-nabla-0}. Besides, $\operatorname{Range } \nabla \Phi(w) \subset T_{\Phi(w)}\mathcal{R}$ is due to $\Phi(w)\in \mathcal{R}$. 

It follows that for any $ z\in \mathcal{R}$ and any $v \in \mathbb{R}^{p}$, we have
\begin{equation}\label{0-eq-nablaphiz}
    \begin{aligned}
     0 =& \nabla\left( \nabla\Phi F_{s}^\dagger \nabla f\right)(z)[v]  
     \\=&\nabla^2\Phi(z)[v]\underbrace{F_{s}^\dagger(z) \nabla f(z)}_{=0}+\nabla\Phi(z)\nabla F_{s}^\dagger(z)[v] \nabla f(z)+\nabla \Phi(z) F_{s}^\dagger(z)  \nabla^2 f(z)[v]
     \\=&\nabla\Phi(z)\left(\nabla F_{s}^\dagger(z)[v] \nabla f(z)+  F_{s}^\dagger(z)  \nabla^2 f(z)[v]\right).
    \end{aligned}
\end{equation}

Define the linear map $T_z$ as $T_z(v)=\nabla F_{s}^\dagger (z)[v] \nabla f(z)+  F_{s}^\dagger (z)  \nabla^2 f(z)v$. Noting that 
\begin{equation}
    \begin{aligned}
     \nabla F_s^\dagger(z)[v]\nabla f(z)=&\nabla P_s(z)[v]F^{-1}(z)P_s(z)\nabla f(z)+ P_s(z)\nabla F^{-1}(z)[v]P_s(z)\nabla f(z)
     \\&+ P_s(z)F^{-1}(z)\nabla P_s(z)[v]\nabla f(z)
     \\=&P_s(z)F^{-1}(z)\nabla P_s(z)[v]\nabla f(z),  
    \end{aligned}
\end{equation} 
the range  of $T_z$ is a subspace of $\operatorname{Range} P_s(z)$. By Lemma \ref{F_s-op-bound} and Assumption \ref{ass:rsc}, for $v$ satisfying $P_s(z)v=v$, 
\begin{equation}
   \|T_z(v) \|_{F(z)^{-1}}\ge \lambda_{H_{F,s}}\|v\|_{F(z)}-\frac{\delta_{F_s}}{G}\|\nabla f(z)\|_{F^{-1}(z)}\|v\|_{F(z)}\ge (\lambda_{H_{F,s}}-\delta_{F_s})\|v\|_{F(z)}.
\end{equation}
This implies that $T_z|_{\operatorname{Range} P_s(z)}:\operatorname{Range} P_s(z)\to \operatorname{Range} P_s(z)$ is a linear isomorphism. Therefore, by \eqref{0-eq-nablaphiz}, we get that $\operatorname{Range} P_s(z)\subset \operatorname{Null}\nabla\Phi(z) $, and $\dim (\operatorname{Null}\nabla\Phi(z))\ge \dim(\operatorname{Range} P_s(z))$. 

 On the other hand, for any smoothed curve $\tilde{z}_t$ on $\mathcal{R}$, $\nabla \Phi(\tilde{z}_t)\dot{\tilde{z}}_t=\frac{d}{dt}\Phi(\tilde{z}_t)=\dot{\tilde{z}}_t$. Thus $\nabla \Phi(z)|_{T_z\mathcal{R}}=id$ for any $z\in \mathcal{R}$. It yields $p=\dim (\operatorname{Range} \nabla \Phi(z))+\dim (\operatorname{Null} \nabla \Phi(z))\ge \dim (T_z\mathcal{R})+\dim(\operatorname{Range} P_s(z))\ge p$, implying that $\operatorname{Range} P_s(z)= \operatorname{Null}\nabla\Phi(z) $. We get the conclusion.
\end{proof}

\end{lemma}

By definition,  $f$ exhibits strong convexity  along the sharp directions, and consequently inherits certain properties of strongly convex functions. We  first present this characteristic (as well as $L$-smoothness) in Lemmas \ref{exp-decay-proj-to-river} and \ref{pl-condition-lemma}.

\begin{lemma}\label{exp-decay-proj-to-river}
Suppose Assumptions \ref{ass:river} and \ref{ass:rsc} hold. 
Denote $\gamma=\lambda_{H_F,s}-\frac{\delta_{F_s}}{2}$, $\ell=L_{H_F,s}+\frac{\delta_{F_s}}{2}$.  Then for any $w\in U$ the following inequality hold
\begin{equation}\label{exp-decay-ps-nablaf}
    \begin{aligned}
      \|P_s(\psi_t(w))F(\psi_t(w))^{-1}\nabla f(\psi_t(w))\|_{F(\psi_t(w))} \le\exp(- \gamma t)\|P_s(w) F(w)^{-1}\nabla f(w)\|_{F(w)} ,
    \end{aligned}
\end{equation}
and
\begin{equation}\label{exp-l-decay-ps-nablaf}
    \begin{aligned}
      \|P_s(\psi_t(w))F(\psi_t(w))^{-1}\nabla f(\psi_t(w))\|_{F(\psi_t(w))} \ge\exp(- \ell t)\|P_s(w) F(w)^{-1}\nabla f(w)\|_{F(w)} .
    \end{aligned}
\end{equation}

    \begin{proof}
     \begin{equation}\label{dt-exp-decay-proj-flow}
         \begin{aligned}
            \frac{d}{dt}\|P_sF^{-1}\nabla f\|_{F}^2=&2\partial_t\psi_t(w)^\top\nabla^2f P_s F^{-1} \nabla f+\nabla f^\top \partial_t(P_sF^{-1}P_s)\nabla f
            \\=&-2\nabla f^\top F^{-1}P_s\nabla^2f P_s F^{-1} \nabla f +\nabla f^\top \nabla (F_s^\dagger)[-F_s^\dagger \nabla f] \nabla f
             \\ \overset{\eqref{p-deltafinvs-u-v}}{\le} &-2\lambda_{H_F,s}       \|P_sF^{-1}\nabla f\|_{F}^2+\delta_{F_s}\|P_sF^{-1}\nabla f\|_{F}^2.
         \end{aligned}
     \end{equation}

Using Gronwall's inequality (Lemma \ref{f-Gronwall}) gives
\begin{equation}
    \begin{aligned}
      \|P_s(\psi_t(w))F(\psi_t(w))^{-1}\nabla f(\psi_t(w))\|_{F(\psi_t(w))}^2\le\exp(- 2\gamma t)\|P_s(w) F(w)^{-1}\nabla f(w)\|_{F(w)}^2.
    \end{aligned}
\end{equation}

On the other hand, we have
\begin{equation}\label{dt-L-exp-decay-proj-flow}
         \begin{aligned}
            \frac{d}{dt}\|P_sF^{-1}\nabla f\|_{F}^2
            =&-2\nabla f^\top F^{-1}P_s\nabla^2f P_s F^{-1} \nabla f +\nabla f^\top \nabla (F_s^\dagger)[-F_s^\dagger \nabla f] \nabla f
             \\ \overset{\eqref{p-deltafinvs-u-v}}{\ge} &-2L_{H_F,s}       \|P_sF^{-1}\nabla f\|_{F}^2-\delta_{F_s}\|P_sF^{-1}\nabla f\|_{F}^2.
         \end{aligned}
     \end{equation}

Similarly, we get
\begin{equation}
    \begin{aligned}
      \|P_s(\psi_t(w))F(\psi_t(w))^{-1}\nabla f(\psi_t(w))\|_{F(\psi_t(w))}^2\ge\exp(- 2\ell t)\|P_s(w) F(w)^{-1}\nabla f(w)\|_{F(w)}^2.
    \end{aligned}
\end{equation}

\end{proof}
\end{lemma}

\begin{lemma} \label{pl-condition-lemma}
Suppose Assumptions \ref{ass:river} and \ref{ass:rsc} hold. Then for any $w\in U$ we have 
\begin{equation}\label{w-Phiw-f-inf-f}
  \frac{\gamma\lambda_{F_s}}{2}\|w-\Phi(w)\|_2^2\le f(w)-f(\Phi(w)),  
\end{equation}
and
\begin{equation}\label{f-inf-f-grad-norm}
     \frac{1}{2\ell}\|P_s(w) F(w)^{-1}\nabla f(w)\|_{F(w)}^2\le f(w)-f(\Phi(w))\le\frac{1}{2\gamma}\|P_s(w) F(w)^{-1}\nabla f(w)\|_{F(w)}^2. 
\end{equation}

    \begin{proof}
Given $w$, define $a_t=\|P_s(\psi_t(w))F(\psi_t(w))^{-1}\nabla f(\psi_t(w))\|_{F(\psi_t(w))}$. Then Lemma \ref{exp-decay-proj-to-river} implies that for any $t,s>0$, $a_{t+s}\le e^{-\gamma s} a_t$. 
Note that $f(w)-f(\Phi(w))=\int_0^\infty \nabla f(\psi_t(w))^\top P_s(\psi_t(w))F(\psi_t(w))^{-1}\nabla f(\psi_t(w))dt=\int_0^\infty a_t^2 dt$, $\|w-\Phi(w)\|_2=\|\int_0^\infty P_s(\psi_t(w))F(\psi_t(w))^{-1}\nabla f(\psi_t(w))dt\|_2\le \lambda_{F_s}^{-1/2}\int_0^\infty a_t dt$. We have

\begin{equation}
\begin{aligned}
    \lambda_{F_s}\|w-\Phi(w)\|_2^2\le&  \left(\int_0^\infty a_t dt\right)^2=\int_0^\infty  \int_0^\infty a_sa_t dsdt   =2 \int_0^\infty dt \int_t^\infty a_s a_t  ds
    \\ \le &2\int_0^\infty a_t^2 \left(\int_t^\infty e^{-\gamma (s-t)}ds\right) dt=\frac{2}{\gamma}\int_0^\infty a_t^2 dt.
\end{aligned}
\end{equation}

On the other hand, by Lemma \ref{exp-decay-proj-to-river}, it holds that  
\begin{equation}
\begin{aligned}
    f(w)-f(\Phi(w))=&\int_0^\infty\|P_s(\psi_t(w))F(\psi_t(w))^{-1}\nabla f(\psi_t(w))\|_{F(\psi_t(w))}^2  dt
    \\\le& \int_0^\infty \exp(- 2\gamma t)\|P_s(w) F(w)^{-1}\nabla f(w)\|_{F(w)}^2 dt
    \\=& \frac{1}{2\gamma}\|P_s(w) F(w)^{-1}\nabla f(w)\|_{F(w)}^2,
    \end{aligned}
\end{equation}
and
\begin{equation}
\begin{aligned}
    f(w)-f(\Phi(w))=&\int_0^\infty\|P_s(\psi_t(w))F(\psi_t(w))^{-1}\nabla f(\psi_t(w))\|_{F(\psi_t(w))}^2  dt
    \\\ge& \int_0^\infty \exp(- 2\ell t)\|P_s(w) F(w)^{-1}\nabla f(w)\|_{F(w)}^2 dt
    \\=& \frac{1}{2\ell}\|P_s(w) F(w)^{-1}\nabla f(w)\|_{F(w)}^2.
    \end{aligned}
\end{equation}

\end{proof}
\end{lemma}

The following lemma shows the property of $\nabla\Phi$: it approximates the projection to the flat subspace.
\begin{lemma}
Suppose Assumptions \ref{ass:river} and \ref{ass:rsc} hold, and 
\begin{equation}
\mu_1:=\lambda_{H_F,s}-\frac{3}{2}\delta_{F_s}\ge0,\quad 2\delta\lambda_{F_s}^{-1}\frac{1}{\gamma}\left(1+\frac{\delta_{F_s}}{2\mu_1}\right)\le\frac{1}{2\rho}.
\end{equation} Define
\begin{equation}
 \epsilon_\Phi=   \frac{4\delta\kappa_{F_s} }{\gamma+\mu_1 }   +\frac{4 \kappa_{F_s}^2\delta }{\gamma  }\left(1+\frac{\delta_{F_s}}{2\mu_1}\right) \left(1+ \frac{2\delta}{\gamma+\mu_1 }    \right)+\frac{2\delta_{F_s}\kappa_{F_s}}{\mu_1}\left(1+ \frac{2\delta}{\gamma+\mu_1 }    \right).  
\end{equation}   Then for any $w\in U$ and any vector $v\in \mathbb{R}^{p}$,  we have
\begin{equation}
\|P_s(\Phi(w))\nabla\Phi(w)v\|_{F(\Phi(w))}\le  \frac{\epsilon_\Phi}{2}   \|v\|_{F(\Phi(w))},
\end{equation}
\begin{equation}
    \|P_f(\Phi(w))\nabla\Phi(w)v-P_f(w)v\|_{F(\Phi(w))}\le \frac{\epsilon_\Phi}{2}\|v\|_{F(\Phi(w))}.
\end{equation}

Thus 
\begin{equation}\label{nabla-Phi-P-f-diff}
    \|\nabla\Phi(w)v-P_f(w)v\|_{F(\Phi(w))}\le   \epsilon_\Phi  \|v\|_{F(\Phi(w))}.
\end{equation}

    \begin{proof}
Define
\begin{equation}
h_t=P_s(\psi_t(w))\partial_w\psi_t(w)[v],\quad r_t=P_f(\psi_t(w))\partial_w\psi_t(w)[v].
\end{equation}
It implies that $h_t+r_t=\partial_w\psi_t(w)[v]$, $h_0+r_0=v$ and  $\lim_{t\to \infty}  h_t+\lim_{t\to \infty} r_t =\nabla \Phi(w)[v]$.

We compute the time derivative of $\|h_t\|_F^2$:
\begin{equation}\label{dht-compute}
\begin{aligned}
    &\frac{d}{dt}\|h_t\|^2_{F(\psi_t(w) )} \\=&   \frac{d}{dt} \left(v^\top \partial_w\psi_t(w)^\top F_s(\psi_t(w))\partial_w\psi_t(w)v\right)
    \\=&\underbrace{v^\top \partial_w\psi_t(w)^\top \nabla F_s(\psi_t(w))[\partial_t\psi_t(w)]\partial_w\psi_t(w)v}_A+2\underbrace{v^\top \partial_w\psi_t(w)^\top F_s(\psi_t(w))\partial_t\partial_w\psi_t(w)v}_B.
 \end{aligned}
\end{equation}

\begin{equation}\label{A-dht}
\begin{aligned}
  A =& - v^\top \partial_w\psi_t(w)^\top \nabla F_s(\psi_t(w))[F_s^{{\dagger}}(\psi_t(w))\nabla f(\psi_t(w))]\partial_w\psi_t(w)v
    \\ \overset{\eqref{p-deltafs-u-v} }{\le}&\delta_{F_s}\|\partial_w\psi_t(w)v\|_{F_s}\|\partial_w\psi_t(w)v\|_{F}.
 \end{aligned}
\end{equation}

\begin{equation}\label{B-dht}
\begin{aligned}
    B =& - h_t^\top  F_s(\psi_t(w)) \partial_w(F_s^{{\dagger}}(\psi_t(w))\nabla f(\psi_t(w)))v
    \\=&- h_t^\top  F_s(\psi_t(w)) \partial_w(F_s^{{\dagger}}(\psi_t(w)))[v]\nabla f(\psi_t(w))
    -   h_t^\top   \nabla^2 f(\psi_t(w))h_t
    \\=&-h_t^\top F_s  \nabla F_s^\dagger [\partial_w\psi_t(w)[v]]  \nabla f-   h_t^\top  P_s  \nabla^2 f P_s h_t
    \\\overset{\eqref{p-deltafinvs-u-v} }{\le}& \delta_{F_s}\|F_s h_t\|_{F^{-1}}\|\partial_w\psi_t(w)[v]\|_{F}-\lambda_{H_F,s}\|h_t\|^2_F,
 \end{aligned}
\end{equation}

Substituting \eqref{A-dht} and \eqref{B-dht} into \eqref{dht-compute} yields
\begin{equation}\label{dot-ht-l2}
\begin{aligned}
    \frac{d}{dt}\|h_t\|
_{F(\psi_t(w))}^2&\le3\delta_{F_s}\|  h_t\|_{F(\psi_t(w))}\|h_t+r_t\|_{F(\psi_t(w))}-2\lambda_{H_F,s}\|h_t\|^2_{F(\psi_t(w))} 
\\&\le-2\underbrace{\left(\lambda_{H_F,s}-\frac{3}{2}\delta_{F_s}\right)}_{:=\mu_1}\|h_t\|^2_{F(\psi_t(w))}+\delta_{F_s}\|h_t\| _{F(\psi_t(w))}\|r_t\| _{F(\psi_t(w))}.
\end{aligned}
\end{equation}

Next we consider the time derivative for $r_t$.  Directly computing $\frac{d}{dt}r_t$ gives:

\begin{equation}\label{Direct-dot-rt}
    \begin{aligned}
        \frac{d}{dt}r_t 
    =& \underbrace{ \nabla P_f(\psi_t(w))[\partial_t\psi_t(w)]\partial_w\psi_t(w)v}_c+\underbrace{P_f(\psi_t(w))\partial_t\partial_w\psi_t(w)v}_d.
    \end{aligned}
\end{equation}
For the term $c$:
\begin{equation}\label{C-drt}
    \begin{aligned}
      & \|c\|_{F(\psi_t(w))} \\=& \| \nabla P_f(\psi_t(w))[F_s^{{\dagger}}(\psi_t(w))\nabla f(\psi_t(w))]\partial_w\psi_t(w)v\|_{F(\psi_t(w))}
        \\ \overset{\eqref{nabla-Ps} }{\le}&\frac{\delta }{\|F_s^{{\dagger}}(\psi_0(w))\nabla f(\psi_0(w))\|_{F(\psi_0(w))}} \|\partial_w\psi_t(w)v\|_{F(\psi_t(w))}\|F_s^{{\dagger}}(\psi_t(w))\nabla f(\psi_t(w))\|_{F(\psi_t(w))}
    \\ \overset{\eqref{exp-decay-proj-to-river} }{\le}&\delta e^{-\gamma t}\|\partial_w\psi_t(w)v\|_{F(\psi_t(w))}.  
    \end{aligned}
\end{equation}
where the first inequality also uses $\nabla P_s+\nabla P_f=\nabla I=0$.

Note that for any vector $p$,

\begin{equation}\label{D-drt-Ff-nabla-Finvs}
    \begin{aligned}
     P_f(w) \nabla (F_s^\dagger) (w)[p]&=P_f(w) \left[\nabla P_s(w)[p]F^{-1}P_s+ P_s\nabla(F^{-1})(w)[p]P_s+P_sF^{-1}\nabla P_s[p]\right]
     \\&=P_f(w) \nabla (P_s)(w)[p]F^{-1}(w)P_s(w).
    \end{aligned}
\end{equation}

Thus 
\begin{equation}\label{D-drt}
\begin{aligned}
    &\|d\|_{F(\psi_t(w))} 
    \\=& \|  P_f(\psi_t(w)) \partial_w(F_s^{{\dagger}}(\psi_t(w))\nabla f(\psi_t(w)))v\|_{F(\psi_t(w))}
    \\\le&\|P_f(\psi_t(w)) \partial_w(F_s^{{\dagger}}(\psi_t(w)))[v]\nabla f(\psi_t(w))\|_{F(\psi_t(w))}
    \\&+\|P_f(\psi_t(w)) F_s^{{\dagger}}(\psi_t(w))\nabla^2 f(\psi_t(w))\partial_w \psi_t(w)v\|_{F(\psi_t(w))}
    \\\overset{\eqref{D-drt-Ff-nabla-Finvs} }{
    =}&\| P_f(\psi_t(w))  \nabla P_s(\psi_t(w))  [\partial_w\psi_t(w)v]  F_s^\dagger(\psi_t(w)) \nabla f(\psi_t(w)) \|_{F(\psi_t(w))}
    \\\overset{\eqref{nabla-Ps}}{\le}& \frac{\delta}{\|F_s^\dagger(\psi_0(w)) \nabla f(\psi_0(w))\|_{F(\psi_0(w))}}   \|\partial_w\psi_t(w)[v]\|_{F(\psi_t(w))}   \|F_s^\dagger(\psi_t(w)) \nabla f(\psi_t(w))\|_{F(\psi_t(w))}
    \\\overset{ \eqref{exp-decay-ps-nablaf} }{\le}& \delta e^{-\gamma t} \|\partial_w\psi_t(w)[v]\|_{F(\psi_t(w))}.
 \end{aligned}
\end{equation}

 Substituting \eqref{C-drt} and \eqref{D-drt} into \eqref{Direct-dot-rt} yields

 \begin{equation}\label{dot-t-l2}
     \left\|\frac{d}{dt }r_t\right\|_{F(\psi_t(w))}\le 2\delta e^{-\gamma t} \|h_t+r_t\|_{F(\psi_t(w))}.
 \end{equation}

Now we combine \eqref{dot-ht-l2} and \eqref{dot-t-l2} to get respective bounds. Denote $M_t=\sup_{0\le s\le t}\|r_s\|_{F(\psi_s(w))}$. Using Lemma \ref{tight-bound-Gronwall} to \eqref{dot-ht-l2} gives

\begin{align}
\|h_t\|_{F(\psi_t(w))}&\le e^{-\mu_1 t}\|h_0\|_{F(\psi_0(w))} + \frac{1}{2}  \delta_{F_s}e^{-\mu_1 t}\int_{0}^t e^{ \mu_1 s}\|r_s\|_{F(\psi_s(w))} ds \label{ht-bound-1}
\\&\le e^{- \mu_1 t}\|h_0\|_{F(\psi_0(w))}+\frac{\delta_{F_s}}{2\mu_1} M_t   \label{ht-bound-2}.
    \end{align}



Thus for any $t\ge0$,

\begin{align}
\|r_t-r_0\|_2\le& \int_0^t \left\|\frac{d}{dr} r_s\right\|_2 ds\le \lambda_{F_s}^{-1}\int_0^t \left\|\frac{d}{dr} r_s\right\|_{F(\psi_s(w))} ds
\\\overset{\eqref{dot-t-l2}}{\le} &\lambda_{F_s}^{-1}\int_0^t 2\delta e^{-\gamma s} (\|h_s\|_{F(\psi_s(w))}+\|r_s\|_{F(\psi_s(w))} )ds \label{rt-r0-bound1}
\\\overset{\eqref{ht-bound-2}}{\le} &2\delta\lambda_{F_s}^{-1}\int_0^t  e^{-\gamma s} \left(e^{-\mu_1s}\|h_0\|_{F(w)}+\left(1+\frac{\delta_{F_s}}{2\mu_1}\right) M_t\right)ds  
\\ {\le}& 2\delta\lambda_{F_s}^{-1} \left(\frac{\|h_0\|_{F(w)}}{\gamma+\mu_1 }   +\frac{1}{\gamma}\left(1+\frac{\delta_{F_s}}{2\mu_1}\right) M_t\right) \label{rt-r0-bound2}
.
    \end{align}

Using $M_t\le \rho(\|r_0\|_2+ \sup_{0\le s\le t}\|r_s-r_0\|_2)$ and $2\delta\lambda_{F_s}^{-1}\frac{1}{\gamma}\left(1+\frac{\delta_{F_s}}{2\mu_1}\right)\le\frac{1}{2\rho}$, we get for any $t\ge 0$, 
\begin{equation}
    M_t\le 2\kappa_{F_s}\left(\|r_0\|_{F(w)}+2\delta \frac{\|h_0\|_{F(w)}}{\gamma+\mu_1 }    \right).
\end{equation}

 Substituting it into \eqref{ht-bound-2}  and \eqref{rt-r0-bound2} respectively, we get

\begin{equation}
\begin{aligned}
    \|h_t\|_{F(\psi_t(w))} \le& e^{-\mu_1 t}\|h_0\|_{F(\psi_0(w))} +\frac{\delta_{F_s}}{\mu_1}\kappa_{F_s}\left(\|r_0\|_{F(w)}+2\delta \frac{\|h_0\|_{F(w)}}{\gamma+\mu_1 }    \right)
    \\ \le &e^{-\mu_1 t}\|h_0\|_{F(\psi_0(w))} +\frac{1}{2}\epsilon_\Phi  \|v\|_{F(w)} ,
    \end{aligned}
\end{equation}

and 

\begin{equation}
\begin{aligned}
    \|r_t-r_0\|_2\le& 2\delta\lambda_{F_s}^{-1}  \frac{\|h_0\|_{F(w)}}{\gamma+\mu_1 }   +\frac{4 \kappa_{F_s}^2\delta }{\gamma \rho}\left(1+\frac{\delta_{F_s}}{2\mu_1}\right) \left(\|r_0\|_{F(w)}+2\delta \frac{\|h_0\|_{F(w)}}{\gamma+\mu_1 }    \right)
    \\\le&\frac{1}{2}\epsilon_\Phi\frac{  \|v\|_{F(w)}}{\rho}.
    \end{aligned}
\end{equation}

Noting that $\lim_{t\to \infty}h_t= P_s(\Phi(w))\nabla\Phi(w)v $ and $\lim_{t\to \infty}(r_t-r_0)= P_f(\Phi(w))\nabla\Phi(w)v-P_f(w) v$, we get the conclusion.

\end{proof}
\end{lemma}

 \subsection{Dynamics Analysis}   
\begin{lemma}\label{app:thm1-0}
Suppose Assumptions \ref{ass:river} and \ref{ass:rsc} hold, and  
\begin{equation}\label{conditions-exp-decay-sharp}
    \epsilon_\Phi\le \frac{1}{10}\min\{1,\alpha\beta_1\},\, \delta_{F_s}\le \frac{1}{10}\lambda_{H_{F,s}}, \,\eta_0\delta_{F_s}\chi \left(   \frac{\kappa_{F}^{1/2}}{\alpha }+\beta_2\right)\le \frac{1}{5}\alpha.
\end{equation} Let $\iota_t=\min\left\{\frac{1}{2} \alpha ,\lambda_{H_F,s}\beta_1 \eta_t\right\}$.  We have
\begin{equation}\label{exp-sharp-grad-decay}
 \begin{aligned}
     \|w_t-\Phi(w_t)\|_2^2+\eta_t\|m_t\|_{F_s^\dagger(w_t)}^2\lesssim &\exp\left(-{ \int_0^t\iota_s ds}\right)
     \\+ \varepsilon& \int_0^t \exp\left({ -\int_s^t\iota_\tau d\tau}\right)   (\|\nabla f(w_s)\|_{F_f^{\dagger}(w_s)}^2+ \|m_s\|_{F_f^{\dagger}(w_s)}^2) ds 
      \end{aligned}
 \end{equation}

Especially,   taking $\alpha= 2\eta_0\lambda_{H_{F,s}}^{\frac{1}{2}}$ and $\beta_1=\lambda_{H_{F,s}}^{-\frac{1}{2}}$ gives

 \begin{equation}
 \begin{aligned}
&\|w_t-\Phi(w_t)\|_2^2+\eta_t \|m_t\|_{F_s^{\dagger}(w_t)}^2
\\\lesssim & \exp\left({-\lambda_{H_{F,s}}^{\frac{1}{2}}\int_0^t\eta_s ds}\right)
+ \varepsilon \int_0^t \exp\left({-\lambda_{H_{F,s}}^{\frac{1}{2}}\int_s^t\eta_\tau d\tau}\right)   (\|\nabla f(w_s)\|_{F_f^{\dagger}(w_s)}^2+ \|m_s\|_{F_f^{\dagger}(w_s)}^2) ds.
      \end{aligned}
 \end{equation}

\begin{remark}
    Such choices of $\alpha$ and $\beta_1$ yield  fast convergence in the (strongly convex) sharp directions of the ill-conditioned landscape. The rate $\exp\left({-\lambda_{H_{F,s}}^{\frac{1}{2}}\int_0^t\eta_s ds}\right)$ matches the results in \cite{Attouch2022}.
\end{remark}

    \begin{proof}
    We first consider general hyper-parameters $\alpha$, $\beta_1$ and $\eta_t$. 
Recall that $\dot{w}_t=-\eta_t F_s(w_t)^{-1}(m_t+\beta_1 \nabla f(w_t))- \eta_t \chi F_f(w_t)^{-1}(m_t+\beta_2\nabla f(w_t))$. 
We begin by bounding  $\|\dot{w}_t\|_{F(w_t)}$. Using $\dot{m}_t=-\alpha m_t+\nabla f(w_t)$, we have $m_t=e^{-\alpha t}\int_0^t e^{\alpha s} \nabla f(w_s)ds$. It yields 
\begin{equation}\label{dot-wt-bound}
\begin{aligned}
     \|\dot{w}_t\|_{F(w_t)}\le& \eta_t\left( \chi\|m_t\|_{F^{-1}(w_t)}+\beta_2\|\nabla f(w_t)\|_{F^{-1}(w_t)} \right)   
     \\\le&\eta_t\chi \left( \lambda_{F}^{-\frac{1}{2}}\int _0^t e^{ -\alpha(t-s)}\|\nabla f(w_s)\|_{2} ds+\beta_2\|\nabla f(w_t)\|_{F^{-1}(w_t)} \right)  
     \\\le &\eta_t\underbrace{\chi \left(   \frac{\kappa_{F}^{1/2}}{\alpha }+\beta_2\right)}_{:=C_w} G.
\end{aligned}
\end{equation}

Now consider  the Lyapunov function:
\begin{equation}
    \begin{aligned}
        V_t=& \left(f(w_t)-f(\Phi(w_t))\right)+\frac{\eta_t}{2}\|m_t\|_{F_s^{\dagger}(w_t)}^2.
    \end{aligned}
\end{equation}

For notational brevity, we define $\widetilde{m}_t=m_t+\beta_1 \nabla f(w_t)$, and 
\begin{equation}
    \begin{aligned}
        \Delta_1=&\left\langle  \nabla f(w_t), (P_f(w_t)-\nabla \Phi(w_t) )\dot{w}_t\right\rangle, \quad \Delta_2=\frac{\dot{\eta}_t}{2}\|m_t\|_{F_s^{\dagger}(w_t)}^2,
        \\\Delta_{3}=& m_t^\top \nabla F_s^\dagger (w_t)[\dot{w}_t]   m_t.
    \end{aligned}
\end{equation}

For $\Delta_1$ We have
\begin{equation}\label{delta1_bound}
    \begin{aligned}
        \Delta_1
        =&\left\langle  \nabla f(w_t), (P_f(w_t)-\nabla \Phi(w_t) )P_f(w_t)\dot{w}_t\right\rangle+\left\langle  \nabla f(w_t), (P_f(w_t)-\nabla \Phi(w_t) )P_s(w_t)\dot{w}_t\right\rangle
        \\\overset{\eqref{nabla-Phi-P-f-diff}}{\le}& \epsilon_\Phi \eta_t\|\nabla f(w_t)\|_{F^{-1}(w_t)}\left(\|m_t+\beta_1 \nabla f(w_t)\|_{F_s^{\dagger}(w_t)}+\chi\|m_t+ \beta_2 \nabla f(w_t)\|_{F_f^{\dagger}(w_t)}\right)
        \\ \le & \epsilon_\Phi \eta_t (\|\nabla f(w_t)\|_{F_s^{\dagger}(w_t)}+\|\nabla f(w_t)\|_{F_f^{\dagger}(w_t)})\left(\|m_t\|_{F_s^{\dagger}(w_t)}+\beta_1 \|\nabla f(w_t)\|_{F_s^{\dagger}(w_t)}\right)
       \\&+ \epsilon_\Phi \eta_t (\|\nabla f(w_t)\|_{F_s^{\dagger}(w_t)}+\|\nabla f(w_t)\|_{F_f^{\dagger}(w_t)})\left(\chi\|m_t\|_{F_f^{\dagger}(w_t)}+ \chi\beta_2 \|\nabla f(w_t)\|_{F_f^{\dagger}(w_t)}\right)\\\le&2\epsilon_\Phi(\frac{1}{\alpha} +\beta_1)\eta_t\|\nabla f(w_t)\|_{F_s^{\dagger}(w_t)}^2+\epsilon_\Phi(\frac{1}{\alpha}+\frac{\beta_1}{2}+\chi \beta_2+\frac{\chi^2\beta_2^2\alpha}{2})\eta_t\|\nabla f(w_t)\|_{F_f^{\dagger}(w_t)}^2
       \\&+\frac{\epsilon_\Phi \alpha}{2}\eta_t\|m_t\|_{F_s^{\dagger}(w_t)}^2
        +\frac{\epsilon_\Phi \alpha}{2} \eta_t\chi^2\|m_t\|_{F_f^{\dagger}(w_t)}^2 .
    \end{aligned}
\end{equation}
where the last inequality uses AM-GM inequality.

For  $\Delta_2$, 
\begin{equation}\label{delta3_bound}
    \begin{aligned}
        |\Delta_3|
        \overset{\eqref{p-deltafinvs-u-v}}{\le}&  \eta_t\delta_{F_s} C_w\| {m}_t\|_{F_s^{\dagger}(w_t)}\|{m}_t\|_{F^{-1}(w_t)}
        \\\le&\eta_t\delta_{F_s} C_w\| {m}_t\|_{F_s^{\dagger}(w_t)}^2+\eta_t\delta_{F_s} C_w\| {m}_t\|_{F_s^{\dagger}(w_t)}\|{m}_t\|_{F_f^{\dagger}(w_t)}
\\\le& \frac{3}{2}\eta_t\delta_{F_s} C_w \| {m}_t\|_{F_s^{\dagger}(w_t)}^2+ \frac{1}{2}\eta_t\delta_{F_s} C_w  \| {m}_t\|_{F_f^{\dagger}(w_t)}^2,
    \end{aligned}
\end{equation}
where the second inequality uses $\| {m}_t\|_{F^{-1}(w_t)}\le \| {m}_t\|_{F_s^{\dagger}(w_t)}+\| {m}_t\|_{F_f^{\dagger}(w_t)}$ and the
last inequality uses Cauchy inequality. 

Directly computing  the time derivative yields
\begin{equation}
    \begin{aligned}
        \frac{d}{dt}V_t+ \iota_t V_t=&   \Delta_1+\Delta_2 +\frac{\eta_t}{2}\Delta_3
        -\alpha \eta_t\|m_t\|_{F_s^{\dagger}(w_t)}^2-\beta_1\eta_t \| \nabla f(w_t)\|_{F_s^{\dagger}(w_t)}^2
         \\&+ \iota_t\left(f(w_t)-f(\Phi(w_t))\right)
         +\frac{\eta_t \iota_t}{2}\|m_t\|_{F_s^{\dagger}(w_t)}^2
        \\\le& \left(-\beta_1 \eta_t + \frac{\iota_t}{2\gamma } +2\epsilon_\Phi(\frac{1}{\alpha} +\beta_1)\eta_t\right)\| \nabla f(w_t)\|_{F_s^{\dagger}(w_t)}^2
\\&+\left(-\alpha\eta_t + \frac{\eta_t\iota_t}{2} +  \frac{1}{2}\epsilon_\Phi\alpha\eta_t  +\eta_t^2\delta_{F_s}C_w\right)\|m_t \|_{F_s^{\dagger}(w_t)}^2
\\&+ \epsilon_\Phi(\frac{1}{\alpha}+\frac{\beta_1}{2}+\chi \beta_2+\frac{\chi^2\beta_2^2\alpha}{2})\eta_t\|\nabla f(w_t)\|_{F_f^{\dagger}(w_t)}^2
\\&+(\frac{1}{4}  \eta_t^2\delta_{F_s}  C_w +\frac{\epsilon_\Phi \alpha}{2} \eta_t\chi^2) \| {m}_t\|_{F_f^{\dagger}(w_t)}^2 ,
    \end{aligned}
\end{equation}

where the inequality uses $\dot{\eta}_t\le 0$, \eqref{f-inf-f-grad-norm}, \eqref{delta1_bound}  and \eqref{delta3_bound}. 

Now set $\epsilon_u=\epsilon_\Phi(\frac{1}{\alpha}+\frac{\beta_1}{2}+\chi \beta_2+\frac{\chi^2\beta_2^2\alpha}{2})\eta_t +(\frac{1}{4}  \eta_t^2\delta_{F_s}  C_w +\frac{\epsilon_\Phi \alpha}{2} \eta_t\chi^2)  $ and $U_t= \|\nabla f(w_t)\|_{F_f^{\dagger}(w_t)}^2+ \|m_t\|_{F_f^{\dagger}(w_t)}^2$.   

Then by conditions in \eqref{conditions-exp-decay-sharp}, 
we have 
\begin{equation}\label{dv-u-1}
    \frac{d}{dt}V_t+\iota_t V_t\le  \epsilon_u U_t.
\end{equation}

Integrating $\frac{d}{dt}(e^{\int_0^t\iota_sds}V_t)$ yields
\begin{equation}
e^{\int_0^t\iota_sds}V_t\le  V_0+ \epsilon_u\int_0^t e^{\int_0^s\iota_\tau d\tau}  U_s ds.
\end{equation}

By \eqref{w-Phiw-f-inf-f} in Lemma \ref{pl-condition-lemma}, we obtain

 \begin{equation}
 \begin{aligned}
     &\|w_t-\Phi(w_t)\|_2^2+ \eta_t\|m_t\|_{F_s^\dagger(w_t)}^2\\ \lesssim & \exp\left(-{ \int_0^t\iota_s ds}\right)
      + \varepsilon\int_0^t \exp\left({ -\int_s^t\iota_\tau d\tau}\right)   (\|\nabla f(w_s)\|_{F_f^{\dagger}(w_s)}^2+ \|m_s\|_{F_f^{\dagger}(w_s)}^2) ds .
      \end{aligned}
 \end{equation}

If we further take $\alpha=2\eta_0\lambda_{H_{F,s}}^{\frac{1}{2}}$ and $\beta_1=\lambda_{H_{F,s}}^{-\frac{1}{2}}$, then $\iota_t=\eta_t\lambda_{H_{F,s}}^{\frac{1}{2}}$, and we finally get

 \begin{equation}
 \begin{aligned}
     &\|w_t-\Phi(w_t)\|_2^2+\eta_t\|m_t\|_{F_s^\dagger(w_t)}^2
     \\\lesssim& \exp\left({-\lambda_{H_{F,s}}^{\frac{1}{2}}\int_0^t\eta_s ds}\right)
     + \varepsilon  \int_0^t \exp\left({-\lambda_{H_{F,s}}^{\frac{1}{2}}\int_s^t\eta_\tau d\tau}\right)   (\|\nabla f(w_s)\|_{F_f^{\dagger}(w_s)}^2+ \|m_s\|_{F_f^{\dagger}(w_s)}^2) ds .
      \end{aligned}
 \end{equation}

\end{proof}
\end{lemma}

\begin{lemma}\label{Pm-Pf-diff}
Suppose Assumptions \ref{ass:river} and \ref{ass:rsc} hold. For any $z\in \mathcal{R}$, we have
\begin{equation} 
    \|P_\mathcal{R}(z) v-P_f(z) v\|_{F(z)}\le \frac{\delta_{F_s}}{\lambda_{H_F,s}} \|v\|_{F(z)} .
\end{equation}
    \begin{proof}
        We first compute the closed form of $P_\mathcal{R}$. Let $\mathbf{0}^p$ denote the $p$ dimensional all zero vector. Note that the $p-k$ dimensional manifold $\mathcal{R}$ is defined by $\mathcal{R}=\{w:P_s(w)  \nabla f(w)=\mathbf{0}^p\}=\{w:F_s^\dagger(w)  \nabla f(w)=\mathbf{0}^p\}$. Define  $N(w)=F_s^\dagger(w) \nabla f(w)$. Thus the tangent space is $T_z \mathcal{R}=\{v:\nabla N(z)v=\mathbf{0}^p\}$.  
The projection to $T_zM$ is given by 
        \begin{equation}
            P_\mathcal{R}(z) v=\operatorname{argmin}_{u\in \mathbb{R}^p} \|v-u\|^2_{F(z)},\quad \text{s.t. } \nabla N(z)u=\mathbf{0}^p.
        \end{equation}
Directly computing $\nabla N$ gives $\nabla N(z)u=\nabla F_s^\dagger(w)[u] \nabla f(w) +F_s^\dagger(w) \nabla f^2(w)u$.

Noting that $\|v-u\|^2_{F(z)}=\|F(z)^{ \frac{1}{2}}v-F(z)^{ \frac{1}{2}}u\|^2_{2}$ and $\nabla N(z)u=\nabla N(z)F(z)^{-\frac{1}{2}}F(z)^{\frac{1}{2}}u$,  We have
\begin{equation}\label{express-PM}
  F(z)^{\frac{1}{2}}  P_\mathcal{R}(z) v= (\underbrace{I-F(z)^{-\frac{1}{2}} \nabla N(z)^\top (\nabla N(z)F(z)^{-1} \nabla N(z)^\top)^{\dagger}\nabla N(z)F(z)^{-\frac{1}{2}}  }_{K(z)})F(z)^{\frac{1}{2}} v.
\end{equation}

For notational brevity, we define $J =F(z)^{-\frac{1}{2}}P_s(z) \nabla^2 f(z) P_s(z)F(z)^{-\frac{1}{2}}$, and $\Delta J\in \mathbb{R}^{k\times p}$ satisfying $\Delta J q=F(z)^{\frac{1}{2}}\nabla F_s^\dagger(w)[F(z)^{-\frac{1}{2}}q]^\top \nabla f(z)$ for any $q\in \mathbb{R}^{p}$.  

Note that $K(z)$ is the Euclidean projection to the kernel subspace:
\begin{equation}
\begin{aligned}
\text{ker} \left(\nabla N(z)F(z)^{-\frac{1}{2}}\right)=\text{ker} \left(F(z)^{\frac{1}{2}}\nabla N(z)F(z)^{-\frac{1}{2}}\right)
    =\text{ker} \left( J  + \Delta J \right),
    \end{aligned}
\end{equation} 
and $P_f(z)$ is the Euclidean projection to   $\text{ker} J$. Using Wedin  $\sin(\Theta)$ Theorem (Lemma \ref{Wedin}) to $J$ and $J+ \Delta J  $, we get 
\begin{equation}
    \|K(z)-P_f(z)\|_2\le \frac{ \|\Delta J  \|_2}{\lambda_{H_F,s}}.
\end{equation}
    It follows that
\begin{equation}\label{F-PM-v-F-Pf-v}
\begin{aligned}
     \|F(z)^{\frac{1}{2}} P_\mathcal{R}(z)v-F(z)^{\frac{1}{2}}P_f(z)v\|_2=&\|K(z)F(z)^{\frac{1}{2}}v-P_f(z)F(z)^{\frac{1}{2}}v\|_2 
     \\\le& \frac{ \|\Delta J  \|_2}{\lambda_{H_F,s}}\|F(z)^{\frac{1}{2}}v\|_{2}.
    \end{aligned}
\end{equation}

Note that
\begin{equation}
    \begin{aligned}
        \|\Delta J q\|_2=\| \nabla F_s^\dagger(w)[F(z)^{-\frac{1}{2}}q]^\top \nabla f(z)\|_{F(z)}
        \overset{\eqref{p-deltafinvs-u-v}}{\le} \delta_{F_s}\|q\|_2.
    \end{aligned}
\end{equation}

Substituting it into \eqref{F-PM-v-F-Pf-v} yields
\begin{equation}
    \| P_\mathcal{R}(z)v- P_f(z)v\|_{F(z)}=\|F(z)^{\frac{1}{2}} P_\mathcal{R}(z)v-F(z)^{\frac{1}{2}}P_f(z)v\|_2 \le \frac{\delta_{F_s}}{\lambda_{H_F,s}} \|v\|_{F(z)}.
\end{equation}

    \end{proof}
\end{lemma}

\begin{lemma}
Suppose Assumptions \ref{ass:river} and \ref{ass:rsc} hold. For any $x,y\in U$ and $v\in\mathbb{R}^p$, we have 
\begin{equation}\label{d-Fs-2-norm}
    \|F_s^\dagger(x)v-F_s^\dagger(y)v\|_2\le \frac{\delta_{F_s}}{G} \kappa_F^{\frac{1}{2}} \lambda_{F}^{-\frac{1}{2}}\|x-y\|_2 \|v\|_2,
\end{equation}

\begin{equation}\label{d-Ff-2-norm}
    \|F_f^\dagger(x)v-F_f^\dagger(y)v\|_2\le \frac{\delta_{F_s}}{G} \kappa_F^{\frac{1}{2}} \lambda_{F}^{-\frac{1}{2}}\|x-y\|_2 \|v\|_2,
\end{equation}
\begin{equation}\label{Pf-lip}
    \begin{aligned}
 \|P_f(x)v-P_f(y)v\|_2\le  \frac{\delta}{G} \kappa_F^{\frac{1}{2}} \lambda_{F}^{-\frac{1}{2}} \|x-y\|_2\|v\|_2   .    
    \end{aligned}
\end{equation}

    \begin{proof}
    By Assumption \ref{ass:rsc},  
\begin{equation}
    \begin{aligned}
 \|F_s^\dagger(x)v-F_s^\dagger(y)v\|_2=\|\int_0^1 \nabla F_s^\dagger(x+(y-x)t)[x-y]vdt \|_2\overset{\eqref{p-deltafinvs-u-v}}{\le} \frac{\delta_{F_s}}{G} \kappa_F^{\frac{1}{2}} \lambda_{F}^{-\frac{1}{2}} \|x-y\|_2\|v\|_2       
    \end{aligned}
\end{equation}
We can similarly prove \eqref{d-Ff-2-norm} by \eqref{p-deltaff-u-v}. On the other hand, using $\nabla P_f+\nabla P_s=0$, we get \begin{equation}
    \begin{aligned}
 \|P_f(x)v-P_f(y)v\|_2=\|\int_0^1 \nabla P_s(x+(y-x)t)[x-y]vdt \|_2 \le  \frac{\delta}{G} \kappa_F^{\frac{1}{2}} \lambda_{F}^{-\frac{1}{2}} \|x-y\|_2\|v\|_2   .    
    \end{aligned}
\end{equation}

    \end{proof}
\end{lemma}

\begin{lemma}
Suppose Assumptions \ref{ass:river} and \ref{ass:rsc} hold. 
The dynamics of $z_t=\Phi(w_t)$ satisfies
\begin{align}\label{dyn-z-lemma}
\begin{cases}
    &\dot{z}_t =- \eta_t \chi P_\mathcal{R}(z_t)F(z_t)^{-1} (s_t+\beta_2\nabla f(z_t))+\eta_t \epsilon_{z,t} , \\
    &\dot{s}_t=-\alpha s_t +\nabla f(z_t), 
\end{cases}
\end{align}
where $s_0=0$ and 
\begin{equation}\label{loose-bound-eps-z-t}
    \|\epsilon_{z,t}\|_2\lesssim\int_0^t e^{-\alpha(t-s)} \|w_s-\Phi(w_s)\|_2ds+\|w_t-\Phi(w_t)\|_2+\varepsilon \|m_t\|_{F_s^\dagger(w_t)}+\varepsilon \| F_f^\dagger(z_t)(s_t+\beta_2\nabla f(z_t))\|_2 
\end{equation}

    \begin{proof}
By \eqref{nabla-phi-Fs-nabla-0}, the dynamics of $z_t=\Phi(w_t)$ satisfies

\begin{equation}
\begin{aligned}\label{dot-z-1}
        \dot{z}_t=-\eta_t\nabla \Phi(w_t) \left(F_s^\dagger(w_t) m_t +\chi F_f^\dagger(w_t)(m_t+\beta_2\nabla f(w_t))\right).
    \end{aligned}
\end{equation}

Define $s_t$ satisfying $\dot{s}_t=-\alpha s_t+\nabla f(z_t)$ and $s_0=0$. Then 
\begin{equation}\label{st-mt-bound}
    \|s_t-m_t\|_2=\left\|\int_0^t e^{-\alpha(t-s)} (\nabla f(z_s)-\nabla f(w_s))ds\right\|_2\le L\int_0^t e^{-\alpha(t-s)} \|w_s-\Phi(w_s)\|_2ds
\end{equation}
and
\begin{equation}\label{s-m-2-bound}
\begin{aligned}
    \max\{\|s_t\|_2, \|m_t\|_2\}\le&\int_0^t e^{-\alpha(t-s)}\max\{\|\nabla f(w_s)\|_2,\|\nabla f(z_s)\|_2\} ds
    \\\le&\int_0^t e^{-\alpha(t-s)}\lambda_F^{-\frac{1}{2}} G ds\le \frac{\lambda_F^{-\frac{1}{2}}}{\alpha} G.
    \end{aligned}
\end{equation}
On the other hand,
\begin{equation}\label{s-Fm-Fs}
    \begin{aligned}
\|F_s^\dagger(w_t) m_t-F_s^\dagger(z_t) s_t\|_2\le& \|F_s^\dagger(w_t) (m_t-  s_t)\|_2+\|F_s^\dagger(w_t) s_t-F_s^\dagger(z_t) s_t\|_2
        \\\overset{\eqref{d-Fs-2-norm}}{\le}& \lambda_{F_s}^{-1}\|m_t-s_t\|_2+\frac{\delta_{F_s}}{G} \kappa_F^{\frac{1}{2}} \lambda_{F}^{-\frac{1}{2}}\|s_t\|_2\|w_t-\Phi(w_t)\|_2
        \\\overset{\eqref{s-m-2-bound}}{\le}& \lambda_{F_s}^{-1}\|m_t-s_t\|_2+\frac{ \delta_{ F_s}}{\alpha} \kappa_F^{\frac{1}{2}} \lambda_{F}^{-1}\|w_t-\Phi(w_t)\|_2
    \end{aligned}
\end{equation}
Similarly
\begin{equation}\label{f-Fm-Fs}
    \begin{aligned}
       & \|F_f^\dagger(w_t) m_t-F_f^\dagger(z_t) s_t\|_2
       \\\le& \|F_f^\dagger(w_t) (m_t-  s_t)\|_2+\|F_f^\dagger(w_t) s_t-F_f^\dagger(z_t) s_t\|_2
        \\\overset{\eqref{d-Ff-2-norm}}{\le}& \lambda_{F}^{-1}\|m_t-s_t\|_2+\frac{\delta_{F_s}}{G} \kappa_F^{\frac{1}{2}} \lambda_{F}^{-\frac{1}{2}}\|s_t\|_2\|w_t-\Phi(w_t)\|_2
        \\\overset{\eqref{s-m-2-bound}}{\le}& \lambda_{F_s}^{-1}\|m_t-s_t\|_2+\frac{ \delta_{ F_s}}{\alpha} \kappa_F^{\frac{1}{2}} \lambda_{F}^{-1}\|w_t-\Phi(w_t)\|_2,
    \end{aligned}
\end{equation}
and
\begin{equation}\label{f-Fm-Fs-g}
    \begin{aligned}
       & \|F_f^\dagger(w_t) \nabla
 f(w_t)-F_f^\dagger(z_t) \nabla
 f(z_t)\|_2 \\ \le & \|F_f^\dagger(w_t) (\nabla
 f(w_t)-  \nabla
 f(z_t))\|_2+\|F_f^\dagger(w_t) \nabla
 f(z_t)-F_f^\dagger(z_t) \nabla
 f(z_t)\|_2
        \\\overset{\eqref{d-Ff-2-norm}}{\le}& \lambda_{F}^{-1}L\|w_t-\Phi(w_t)\|_2+ \delta_{F_s}   \kappa_F^{\frac{1}{2}} \lambda_{F}^{-1} \|w_t-\Phi(w_t)\|_2.
    \end{aligned}
\end{equation}

  Now we rewrite  \eqref{dot-z-1} as

 \begin{equation}
     \begin{aligned}
\dot{z}_t=&-\eta_tP_\mathcal{R}(z_t) \left(F_s^\dagger(z_t)s_t+\chi F_f^\dagger(z_t)(s_t+\beta_2\nabla f(z_t))\right)+\eta_t\epsilon_{z,t},
     \end{aligned}
 \end{equation}
where 
\begin{equation}
    \begin{aligned}
        \epsilon_{z,t}=&- \underbrace{P_\mathcal{R}(z_t) \left(F_s^\dagger(w_t) m_t-F_s^\dagger(z_t) s_t  \right)}_a
\\&- \underbrace{(\nabla \Phi(w_t) -P_\mathcal{R}(z_t))\left(F_s^\dagger(w_t)m_t+\chi F_f^\dagger(z_t)(s_t+\beta_2\nabla f(z_t))\right)}_b
\\&-   \underbrace{\chi\nabla \Phi(w_t)  \left(  F_f^\dagger(w_t)(m_t+\beta_2\nabla f(w_t))-F_f^\dagger(z_t)(s_t+\beta_2\nabla f(z_t))\right)}_c.
    \end{aligned}
\end{equation}

By \eqref{express-PM}, we have 
\begin{equation}\label{a-bound-flat-dyn}
\begin{aligned}
    \|a\|_2\le& \lambda_F^{-\frac{1}{2}}\|a\|_{F(z_t)}\le \lambda_F^{-\frac{1}{2}}\|F_s^\dagger(w_t) m_t-F_s^\dagger(z_t) s_t\|_{F(z_t)}\le \kappa_F^{\frac{1}{2}}\|F_s^\dagger(w_t) m_t-F_s^\dagger(z_t) s_t\|_2
    \\\overset{\eqref{s-Fm-Fs}}{\le}&\kappa_F^{\frac{1}{2}}\lambda_{F_s}^{-1}\|m_t-s_t\|_2+\kappa_F^{\frac{1}{2}}\frac{\delta_{F_s}}{\alpha} \kappa_F^{\frac{1}{2}} \lambda_{F}^{-1}\|w_t-\Phi(w_t)\|_2.
    \end{aligned}
\end{equation}

For any $v$, 
\begin{equation}
\begin{aligned}
    &\|\nabla \Phi(w_t) v-P_\mathcal{R}(z_t) v\|_2 
    \\\le&  \|\nabla \Phi(w_t) v-P_f(w_t) v\|_2+\| P_f(w_t) v-P_f(z_t) v\|_2+\| P_f(z_t) v-P_\mathcal{R}(z_t) v\|_2
   \\ \le & \lambda_F^{-\frac{1}{2}}\|\nabla \Phi(w_t) v-P_f(w_t) v\|_{F(w_t)}+\| P_f(w_t) v-P_f(z_t) v\|_2+\lambda_F^{-\frac{1}{2}}\| P_f(z_t) v-P_\mathcal{R}(z_t) v\|_{F(w_t)}
   \\ \le & \rho^{\frac{1}{2}}\lambda_F^{-\frac{1}{2}}(\epsilon_\Phi+\frac{\delta_{F_s}}{\lambda_{H_{F,s}}})\|v\|_2+\frac{\delta}{G} \kappa_F^{\frac{1}{2}} \lambda_{F}^{-\frac{1}{2}} \|w_t-\Phi(w_t)\|_2\|v\|_2,
   \end{aligned}
\end{equation}
where the last inequality uses \eqref{Pf-lip}, Lemma \eqref{Pm-Pf-diff} and \eqref{nabla-Phi-P-f-diff}. It follows that

\begin{equation}\label{b-bound-flat-dyn}
    \begin{aligned}
        \|b\|_2\le&\kappa_F^{\frac{1}{2}} (\epsilon_\Phi+\frac{\delta_{F_s}}{\lambda_{H_{F,s}}})\|F_s^\dagger(w_t)m_t+\chi F_f^\dagger(z_t)(s_t+\beta_2\nabla f(z_t))\|_2
        \\&+\frac{\delta}{G} \kappa_F^{\frac{1}{2}} \lambda_{F}^{-\frac{1}{2}} \|w_t-\Phi(w_t)\|_2\|F_s^\dagger(w_t)m_t+\chi F_f^\dagger(z_t)(s_t+\beta_2\nabla f(z_t))\|_2
        \\\overset{\eqref{s-m-2-bound}}{\le}&\kappa_F^{\frac{1}{2}} (\epsilon_\Phi+\frac{\delta_{F_s}}{\lambda_{H_{F,s}}})\|F_s^\dagger(w_t)m_t\|_2+\chi\kappa_F^{\frac{1}{2}} (\epsilon_\Phi+\frac{\delta_{F_s}}{\lambda_{H_{F,s}}})\| F_f^\dagger(z_t)(s_t+\beta_2\nabla f(z_t))\|_2
        \\&+ 2\chi \delta  \kappa_F^{\frac{1}{2}} \lambda_{F}^{-1}(\frac{1}{\alpha}+\beta_2) \|w_t-\Phi(w_t)\|_2.
    \end{aligned}
\end{equation}
By \eqref{nabla-Phi-P-f-diff} we have $\|\nabla \Phi(w_t)\|_2\le \|\nabla \Phi(w_t)-P_f(w_t)\|_2+\| P_f(w_t)\|_2\le 1+\epsilon_\Phi\kappa_F$. Thus

\begin{equation}\label{c-bound-flat-dyn}
    \begin{aligned}
  \|c\|_2\le& \chi (1+\epsilon_\Phi\kappa_F)\|F_f^\dagger(w_t)(m_t+\beta_2\nabla f(w_t))-F_f^\dagger(z_t)(s_t+\beta_2\nabla f(z_t))\|_2 
\\\overset{\eqref{f-Fm-Fs},\eqref{f-Fm-Fs-g}}{\le} & \chi (1+\epsilon_\Phi\kappa_F)\lambda_{F_s}^{-1}\|m_t-s_t\|_2
\\&+\chi (1+\epsilon_\Phi\kappa_F)(\frac{ \delta_{ F_s}}{\alpha} \kappa_F^{\frac{1}{2}} \lambda_{F}^{-1}+\beta_2\lambda_{F}^{-1}L+\beta_2 \delta_{F_s}   \kappa_F^{\frac{1}{2}} \lambda_{F}^{-1}) \|w_t-\Phi(w_t)\|_2.    \end{aligned}
\end{equation}

Combining \eqref{a-bound-flat-dyn}, \eqref{b-bound-flat-dyn} and \eqref{c-bound-flat-dyn}, we get 
\begin{equation}
    \|a+b+c\|_2=\mathcal{O}(\|m_t-s_t\|_2+\|w_t-\Phi(w_t)\|_2+\varepsilon \|m_t\|_{F_s^\dagger(w_t)}+\varepsilon \| F_f^\dagger(z_t)(s_t+\beta_2\nabla f(z_t))\|_2)
\end{equation}
Then using \eqref{st-mt-bound} gives the conclusion.
    \end{proof}
\end{lemma}

\begin{lemma}\label{app:thm1-1}
Suppose Assumptions \ref{ass:river} and \ref{ass:rsc} hold. 
 Then there exist a constant $\varepsilon_0>0$ depending on $f$ and $\alpha,\beta,
\{\eta_t\}_t,\chi,\beta_2$ such that for any  $\varepsilon<\varepsilon_0$, we have
\begin{equation}\label{sharp-bound-eps-z-t}
    \|\epsilon_{z,t}\|_2^2\lesssim e^{- \Pi_t}  
     + \varepsilon  \int_0^t e^ {  \Pi_\tau- \Pi_t} \left(\| \nabla f(z_\tau)\|_{F_f^\dagger(z_t)}^2+ \| s_\tau\|_{F_f^\dagger(z_t)}^2\right) d\tau  +\varepsilon^2 \| s_t+\beta_2\nabla f(z_t)\|_{F_f^\dagger(z_t)}^2.
\end{equation} 
    \begin{proof}
  Define $\Pi_s=\int_0^s\iota_\tau d\tau$, $D_t=\|w_t-\Phi(w_t)\|_2+\eta_t^{1/2}\|m_t\|_{F_s^\dagger(w_t)}$ and $E_t=\|\nabla f(z_t)\|_{F_f^{\dagger}(z_t)}^2+ \|s_t\|_{F_f^{\dagger}(z_t)}^2$  for notational brevity.  By \eqref{exp-sharp-grad-decay}  we have 

\begin{equation} 
 \begin{aligned}
     D_t^2 \lesssim   &\exp\left(-\Pi_t\right)
      + \varepsilon  \int_0^t \exp\left(\Pi_s-\Pi_t\right)   (\|\nabla f(w_s)-\nabla f(z_s)\|_{F_f^{\dagger}(z_t)}^2+ \|m_s-z_s\|_{F_f^{\dagger}(z_t)}^2+E_s) ds 
      \\ \overset{\eqref{st-mt-bound}}{\lesssim  } &  \exp\left(-\Pi_t\right)
      + \varepsilon   \int_0^t \exp\left(\Pi_s-\Pi_t\right)   (D_s^2+ \int _0^s e^{-\alpha(s-\tau)} D_\tau^2 d\tau+E_s) ds
      \\   \le    &   \exp\left(-\Pi_t\right)
      + \varepsilon   \int_0^t \exp\left(\Pi_s-\Pi_t\right)   (D_s^2+E_s) ds,
      \end{aligned}
 \end{equation}

where we also use $(\int _0^s e^{-\alpha(s-\tau)} D_\tau  d\tau )^2\le\int _0^s e^{-\alpha(s-\tau)} D_\tau^2 d\tau\int _0^s e^{-\alpha(s-\tau)}  d\tau\lesssim \int _0^s e^{-\alpha(s-\tau)} D_\tau^2 d\tau$ in the second inequality, and 
\begin{equation}
    \begin{aligned}
\int_0^t \exp\left(\Pi_s-\Pi_t\right)  \int _0^s e^{-\alpha(s-\tau)} D_\tau^2 d\tau ds=&    \int_0^t \exp \left(\Pi_\tau-\Pi_t\right) D_\tau^2 d\tau \int _\tau^t e^{\alpha \tau -\alpha s}    \exp\underbrace{\left(\Pi_s-\Pi_\tau\right)}_{\le (s-\tau)\iota_\tau}ds 
\\\le & \frac{1}{\alpha-\eta_t}\int_0^t  \exp \left(\Pi_\tau-\Pi_t\right) D_\tau^2 d\tau  .     
    \end{aligned}
\end{equation} in the third inequality. Then by using Gronwall's inequality in Lemma \ref{f-Gronwall} to $e^{\Pi_t}D_t^2$, we get

\begin{equation}\label{Dt-bound} 
 \begin{aligned}
     \exp\left(\Pi_t \right) D_t^2   \lesssim   &   \exp(\mathcal{O}(\varepsilon t))\left(1
      + \varepsilon   \int_0^t   \exp\left(\Pi_s \right)   E_s ds\right).
      \end{aligned}
 \end{equation}

Substituting it into \eqref{st-mt-bound} yields

     \begin{equation}\label{squ-e-alpha-d}
         \begin{aligned}
             &\left(\int_0^t e^{-\alpha(t-s)} \|w_s-\Phi(w_s)\|_2ds\right)^2\\\le &\int_0^t e^{-\alpha(t-s)} ds \int_0^t e^{-\alpha(t-s)} \|w_s-\Phi(w_s)\|_2^2ds
             \\\lesssim&  \int_0^t e^{-\alpha(t-s)}\exp\left(-\Pi_s\right)ds
     + \varepsilon \int_0^t e^{-\alpha(t-s)} ds\int_0^s \exp\left({ \Pi_\tau-\Pi_s}\right)    E_\tau   d\tau.  
         \end{aligned}
     \end{equation}
To bound the right hand side, we first note that for any fixed $0\le\tau<t$, $\frac{\alpha}{2}\ge\frac{\Pi_s-\Pi_\tau}{s-\tau}\ge\frac{\Pi_t-\Pi_\tau}{t-\tau}:=\bar{\iota}$ since  $  \iota_t  $ is non-increasing. This gives
\begin{equation}\label{t-s-pi-tau-pi-s}
  \int_\tau^t e^{-\alpha(t-s)}\exp\left({  \Pi_\tau- \Pi_s}\right) ds \le  \int_\tau^t e^{-\alpha(t-s)}e^{-\bar{\iota}(s-\tau)} ds\le \frac{1}{\alpha-\bar{\iota}} e^{-\bar{\iota}(t-\tau)}\le\frac{2}{\alpha } e^{\Pi_\tau-\Pi_t}.
\end{equation}
    
Thus for the first term in \eqref{squ-e-alpha-d}, we have 
\begin{equation}\label{e-eps-w-Phiw-1}
    \int_0^t e^{-\alpha(t-s)}\exp\left(- \Pi_s\right)ds\le  \frac{2}{\alpha }e^{- \Pi_t}.
\end{equation}
For the second term in \eqref{squ-e-alpha-d},  
\begin{equation}\label{e-eps-w-Phiw-2}
    \begin{aligned}
         \int_0^t e^{-\alpha(t-s)} ds\int_0^s \exp\left({  \Pi_\tau- \Pi_s}\right)    E_\tau  d\tau
        =& \int_0^t E_\tau d\tau \int_\tau^t e^{-\alpha(t-s)}\exp\left({  \Pi_\tau- \Pi_s}\right)      ds
        \\\overset{\eqref{t-s-pi-tau-pi-s}}{\le}&\frac{2}{\alpha}\int_0^t \exp\left({  \Pi_\tau- \Pi_t}\right)E_\tau d\tau  .
    \end{aligned}
\end{equation}
 
Substituting  \eqref{e-eps-w-Phiw-1} and  \eqref{e-eps-w-Phiw-2} into \eqref{squ-e-alpha-d},  we get
\begin{equation}
         \begin{aligned}
            \left(\int_0^t e^{-\alpha(t-s)} \|w_s-\Phi(w_s)\|_2ds\right)^2\lesssim  e^{- \Pi_t}  
     + \varepsilon  \int_0^t \exp\left({  \Pi_\tau- \Pi_t}\right)E_\tau d\tau  .  
         \end{aligned}
     \end{equation}

Substituting it and \eqref{Dt-bound}   into
\eqref{loose-bound-eps-z-t} gives 
 \begin{equation} 
    \|\epsilon_{z,t}\|_2^2\lesssim e^{- \Pi_t}  
     + \varepsilon  \int_0^t e^{  \Pi_\tau- \Pi_t} \left(\| \nabla f(z_\tau)\|_{F_f^\dagger(z_t)}^2+ \| s_\tau\|_{F_f^\dagger(z_t)}^2\right) d\tau  +\varepsilon^2 \| s_t+\beta_2\nabla f(z_t)\|_{F_f^\dagger(z_t)}^2.
\end{equation}

Note that a sufficiently small $\varepsilon$ can satisfy the constraints on $\delta,\delta_F$ in all preceding lemmas. Then, by applying  Lemma \ref{Pm-Pf-diff}, we get the conclusion.

    \end{proof}
\end{lemma}

\begin{lemma}\label{app:thm2-0}
Suppose Assumptions \ref{ass:river} and \ref{ass:rsc} hold, and  

\begin{equation}\label{condition-deltafs-river-flow}
  \delta_{F_s}\le \min\left\{\frac{\alpha\beta_2\lambda_{H_{F,s}}}{2},\frac{\lambda_{H_{F,s}}}{4},\frac{\alpha}{2\chi\eta_{0}} \left(   \frac{\kappa_{F}^{1/2}}{\alpha }+\beta_2\right)^{-1} \right\}.
\end{equation} 
Then the trajectory $z_t$ in 
\begin{align}
\begin{cases}
    &\dot{z}_t =- \eta_t \chi P_\mathcal{R}(z_t)F(z_t)^{-1} (s_t+\beta_2\nabla f(z_t))  , \\
    &\dot{s}_t=-\alpha s_t +\nabla f(z_t), 
\end{cases}
\end{align}
starting from $z_t|_{t=0}=z_0$ and $s_0=0$ satisfies 

\begin{equation}
    \int_0^t\eta_t \|P_\mathcal{R}(z_t) F(z_t)^{-1}\nabla f(z_t)\|_{F(z_t)}^2 \le \frac{2(f(z_0)-\inf f)}{\chi\beta_2},\quad \int_0^t\eta_t \|s_t\|_{F(z_t)}^2 \le \frac{2(f(z_0)-\inf f)}{\chi\alpha}.
\end{equation}
    \begin{proof}
        Consider the Lyapunov function $W_t=f(z_t)-\inf f+\frac{1}{2}\eta_t\chi\|s_t\|_{F^{-1}(z_t)}^2  $. Then taking the time derivative   gives
\begin{equation}\label{dWt-1}
\begin{aligned}
\frac{d}{dt}W_t=&\underbrace{-\eta_t\chi \left\langle   (P_\mathcal{R}(z_t)-I)F(z_t)^{-1}\nabla f(z_t),   s_t\right\rangle}_A-\eta_t\chi \beta_2\|P_\mathcal{R}(z_t) F(z_t)^{-1}\nabla f(z_t)\|_{F(z_t)}^2
\\&-\eta_t\chi\alpha \|s_t\|_{F^{-1}(z_t)}^2+\frac{1}{2}\eta_t\chi s_t^\top \nabla F^{-1}(z_t)[\dot{z}_t]s_t+\frac{1}{2}\dot{\eta}_t\chi\|s_t\|_{F^{-1}(z_t)}^2.
\end{aligned} 
\end{equation}

By using $F(z_t)^{-1}\nabla f(z_t)=P_f(z_t)F(z_t)^{-1}\nabla f(z_t)$ and Lemma \ref{Pm-Pf-diff}, we have
\begin{equation}
    \begin{aligned}
 \| (P_\mathcal{R}(z_t)-I)F(z_t)^{-1}\nabla f(z_t) \|_{F(z_t)}^2   
 =&  \|(P_\mathcal{R}(z_t)-P_f(z_t))F(z_t)^{-1}\nabla f(z_t) \|_{F(z_t)} ^2
 \\\le&  \frac{\delta_{F_s}^2}{\lambda_{H_{F,s}}^2} \| F(z_t)^{-1}\nabla f(z_t) \|_{F(z_t)} ^2.
    \end{aligned}
\end{equation}
Since $\| (P_\mathcal{R}(z_t)-I)F(z_t)^{-1}\nabla f(z_t) \|_{F(z_t)}^2+\|  P_\mathcal{R}  F(z_t)^{-1}\nabla f(z_t) \|_{F(z_t)}^2=\|  F(z_t)^{-1}\nabla f(z_t) \|_{F(z_t)}^2$, we have

\begin{equation}
    \begin{aligned}
 \| (P_\mathcal{R} (z_t)-I)F(z_t)^{-1}\nabla f(z_t) \|_{F(z_t)}^2\le \frac{\frac{\delta_{F_s}^2}{\lambda_{H_{F,s}}^2}}{1-\frac{\delta_{F_s}^2}{\lambda_{H_{F,s}}^2}}\|P_\mathcal{R}(z_t)F(z_t)^{-1}\nabla f(z_t) \|_{F(z_t)} ^2.  
    \end{aligned}
\end{equation}
Thus
\begin{equation}\label{A-bound-Pm-I}
    \begin{aligned}
        A\overset{\eqref{condition-deltafs-river-flow}}{\le}& 2\frac{\delta_{F_s}}{\lambda_{H_{F,s}}}\eta_t\chi  \|P_\mathcal{R}(z_t)F(z_t)^{-1}\nabla f(z_t) \|_{F(z_t)}\|s_t\|_{F^{-1}(z_t)}
        \\\le&  \frac{\delta_{F_s}}{\alpha\lambda_{H_{F,s}}}\eta_t\chi  \|P_\mathcal{R}(z_t)F(z_t)^{-1}\nabla f(z_t) \|_{F(z_t)}^2+\frac{\delta_{F_s}\alpha}{\lambda_{H_{F,s}}}\eta_t\chi\|s_t\|_{F^{-1}(z_t)}^2.
    \end{aligned}
\end{equation}

On the other hand, similar to \eqref{dot-wt-bound}, we have
\begin{equation}\label{dot-zt-bound}
\begin{aligned}
     \|\dot{z}_t\|_{F(z_t)}\le& \eta_t\chi\left( \|s_t\|_{F^{-1}(w_t)}+\beta_2\|\nabla f(z_t)\|_{F^{-1}(z_t)} \right)  \\\le&\eta_t\chi \left( \lambda_{F}^{-\frac{1}{2}}\int _0^t e^{\alpha -(t-s)}\|\nabla f(z_s)\|_{2} ds+\beta_2\|\nabla f(z_t)\|_{F^{-1}(z_t)} \right)  
     \\\le &\eta_t\underbrace{\chi \left(   \frac{\kappa_{F}^{1/2}}{\alpha }+\beta_2\right)}_{:=C_w} G.
\end{aligned}
\end{equation}
This gives 
\begin{equation}\label{s-t-nabla-F-inv-st}
   \eta_t\chi s_t^\top \nabla F^{-1}(z_t)[\dot{z}_t]s_t\le \eta_t^2\chi\delta_{F_s}C_w \|s_t\|_{F^{-1}(z_t)}^2\overset{\eqref{condition-deltafs-river-flow}}{\le} \frac{1}{2}\eta_t\chi\alpha\|s_t\|_{F^{-1}(z_t)}^2. 
\end{equation}
Substituting \eqref{A-bound-Pm-I} and   \eqref{s-t-nabla-F-inv-st} into 
\eqref{dWt-1} yields

\begin{equation}\label{dwt-2}
\begin{aligned}
\frac{d}{dt}W_t=&-\frac{1}{2}\eta_t\chi \beta_2\|P_\mathcal{R}(z_t) F(z_t)^{-1}\nabla f(z_t)\|_{F(z_t)}^2
 -\frac{1}{2}\eta_t\chi\alpha \|s_t\|_{F^{-1}(z_t)}^2 .
\end{aligned} 
\end{equation}
where we also uses $\dot{\eta}_t\le0$. Finally, integrating both sides of \eqref{dwt-2} from $0$ to $t$ yields the conclusion.

    \end{proof}
\end{lemma}

\subsection{Useful Lemmas}

\begin{lemma}[Cauchy's  Inequality]
For any vector $u,v$ and positive semi-definite matrices $A$ and $F$, by using $u^\top A v\le  \|A^{\frac{1}{2}} u\|_2 \|A^{\frac{1}{2}} v\|_2 \le \|A\|_2 \|u\|_2\|v\|_2$, we have
\begin{equation}
       u^\top A v=u^\top F^{\frac{1}{2}}F^{-\frac{1}{2}}AF^{-\frac{1}{2}} F^{\frac{1}{2}} v\le \|F^{-\frac{1}{2}}AF^{-\frac{1}{2}}\|_2 \|u\|_F\|v\|_F .
    \end{equation}
    
\end{lemma}

\begin{lemma}\label{tight-bound-Gronwall}
Let $a_t, b_t \ge 0$ be real-valued functions defined for $t \ge 0$. Suppose $a_t$ is continuously differentiable  and $b_t$ is continuous with respect to $t$. If there exists a constant $\lambda > 0$ such that
\begin{equation}\label{aux-lemma-reg}
\frac{d}{dt} a_t \le -\lambda a_t + \sqrt{a_t}\sqrt{b_t} ,
\end{equation}
then for all $t \ge 0$,
\begin{equation}
\sqrt{a_t} \le e^{-\frac{\lambda}{2}t}\sqrt{a_0} + \frac{1}{2} \int_0^t e^{-\frac{\lambda}{2}(t-s)} \sqrt{b_s}   ds.
\end{equation}

\begin{proof}  Let $\delta > 0$ be an arbitrary small constant. Define a perturbed variable $y_t$ as  $y_t = \sqrt{a_t + \delta}$. Since $a_t \ge 0$, we have $y_t \ge \sqrt{\delta} > 0$ for all $t$. Thus, $y_t$ is continuously differentiable. Substituting $\frac{d}{dt}y_t$ into \eqref{aux-lemma-reg} gives
 
\begin{equation}
    \begin{aligned}
        \frac{d}{dt}{y}_t \le& -\frac{\lambda}{2} y_t + \frac{\lambda \delta}{2y_t} + \frac{\sqrt{b_t}}{2} \frac{\sqrt{y_t^2 - \delta}}{y_t}
          \le -\frac{\lambda}{2} y_t + \frac{\sqrt{b_t}}{2} + \frac{\lambda \sqrt{\delta}}{2}.
    \end{aligned}
\end{equation}
    It follows that 
    
    \begin{equation}
        \frac{d}{dt} \left( e^{\frac{\lambda}{2}t} y_t \right) \le e^{\frac{\lambda}{2}t} \left( \frac{\sqrt{b_t}}{2} + \frac{\lambda \sqrt{\delta}}{2} \right).
    \end{equation} 
    Integrating from $0$ to $t$ yields:
\begin{equation}
     y_t \le e^{-\frac{\lambda}{2}t} y_0 + \frac{1}{2} \int_0^t e^{-\frac{\lambda}{2}(t-s)} \sqrt{b_s }\, ds + \frac{\lambda \sqrt{\delta}}{2} \int_0^t e^{-\frac{\lambda}{2}(t-s)} ds.   
    \end{equation}
    Take the limit as $\delta \to 0$, we get the conclusion.
    
    \end{proof}

{  \remark 
It seems natural to divide the inequality \eqref{aux-lemma-reg} by $\sqrt{a_t}$ and consider $\frac{d}{dt}\sqrt{a_t}$.  However, direct differentiation of $\sqrt{a_t}$ is not well-defined when $a_t=0$. To handle this singularity strictly, we introduce a regularization term $\delta$.
}

\end{lemma}

\begin{lemma}[Gronwall's Inequality]\label{f-Gronwall}
Let $u_t,\alpha_t$ and $\beta_t\ge0$ be continuous   functions for $t\ge 0$. If 
\begin{equation}
  u_t\le \alpha_t+\int_0^t\beta_su_s ds,  
\end{equation}  then
\begin{equation}
 u_t\le  \alpha_t+ \int_0^t \alpha_s\beta_s \exp\left(\int_s^t\beta_\tau d\tau\right) ds. 
\end{equation}
Additionally, if $\alpha_t$ is non-decreasing, then
\begin{equation}
 u_t\le  \alpha_t \exp\left( \int_0^t\beta_\tau d\tau \right). 
\end{equation}
\end{lemma}
\begin{lemma}[Wedin  $\sin(\Theta)$ Theorem]\label{Wedin}
    Let $M$ and $N$ be two matrices with singular values $\sigma_1, \dots, \sigma_n$ and $\tilde{\sigma}_1, \dots, \tilde{\sigma}_n$, respectively. 
    Denote by $U_M, V_M$ the first $k$ left and right singular vectors of $M$, and by $U_N, V_N$ those of $N$. Suppose there exists $\alpha > 0$ such that
    \[
         \alpha\leq \min_{ 1 \le i\le k} \sigma_i - \max_{ k+1 \le r\le n} \tilde{\sigma}_r.
    \]
  Denote the projection to the subspace spanned by $U_M$ as $\mathcal{P}_{U_M}$ (and similarly  for other matrices).  Then
    \[
        \max \bigl\{ \|\mathcal{P}_{U_M}- \mathcal{P}_{U_N}\|_2,\; \|\mathcal{P}_{V_M}- \mathcal{P}_{V_N}\|_2 \bigr\}
        \leq \frac{\bigl\| M-N \bigr\|_2}{\alpha}.
    \]
    
\end{lemma}



\newpage
\section*{NeurIPS Paper Checklist}


\begin{enumerate}

\item {\bf Claims}
    \item[] Question: Do the main claims made in the abstract and introduction accurately reflect the paper's contributions and scope?
    \item[] Answer: \answerYes{} 
    \item[] Justification: We believe that the main claims  in the abstract and introduction accurately reflect the contributions and scope of the paper.

\item {\bf Limitations}
    \item[] Question: Does the paper discuss the limitations of the work performed by the authors?
    \item[] Answer: \answerYes{} 
    \item[] Justification: The limitations are in Section Conclusion.

\item {\bf Theory assumptions and proofs}
    \item[] Question: For each theoretical result, does the paper provide the full set of assumptions and a complete (and correct) proof?
    \item[] Answer: \answerYes{} 
    \item[] Justification: See Section \ref{sec:theory} and Appendix \ref{app:proof}.

    \item {\bf Experimental result reproducibility}
    \item[] Question: Does the paper fully disclose all the information needed to reproduce the main experimental results of the paper to the extent that it affects the main claims and/or conclusions of the paper (regardless of whether the code and data are provided or not)?
    \item[] Answer: \answerYes{} 
    \item[] Justification: We have explained our method and provided experimental details in Sections \ref{sec:lite-approach}, \ref{experiments} and Appendices \ref{appen:alg-details}, \ref{appen:exp-details}.

\item {\bf Open access to data and code}
    \item[] Question: Does the paper provide open access to the data and code, with sufficient instructions to faithfully reproduce the main experimental results, as described in supplemental material?
    \item[] Answer: \answerYes{} 
    \item[] Justification: The dataset we used  is public and can be accessed by everyone. We have provided the   code that reproduces the main experimental results.

\item {\bf Experimental setting/details}
    \item[] Question: Does the paper specify all the training and test details (e.g., data splits, hyperparameters, how they were chosen, type of optimizer) necessary to understand the results?
    \item[] Answer: \answerYes{} 
    \item[] Justification:  The details are in Appendices B, C.

\item {\bf Experiment statistical significance}
    \item[] Question: Does the paper report error bars suitably and correctly defined or other appropriate information about the statistical significance of the experiments?
    \item[] Answer: \answerYes{} 
    \item[] Justification: We report the results in Appendix D.3. 

\item {\bf Experiments compute resources}
    \item[] Question: For each experiment, does the paper provide sufficient information on the computer resources (type of compute workers, memory, time of execution) needed to reproduce the experiments?
    \item[] Answer: \answerYes{} 
    \item[] Justification: The details are in Appendix C.2.
    
\item {\bf Code of ethics}
    \item[] Question: Does the research conducted in the paper conform, in every respect, with the NeurIPS Code of Ethics \url{https://neurips.cc/public/EthicsGuidelines}?
    \item[] Answer: \answerYes{} 
    \item[] Justification: We have confirmed that the research is conducted with the NeurIPS Code of Ethics.

\item {\bf Broader impacts}
    \item[] Question: Does the paper discuss both potential positive societal impacts and negative societal impacts of the work performed?
    \item[] Answer: \answerNA{} 
    \item[] Justification:  There is no societal impact of our work performed. 
    
\item {\bf Safeguards}
    \item[] Question: Does the paper describe safeguards that have been put in place for responsible release of data or models that have a high risk for misuse (e.g., pre-trained language models, image generators, or scraped datasets)?
    \item[] Answer: \answerNA{} 
    \item[] Justification: The paper poses no such risks.

\item {\bf Licenses for existing assets}
    \item[] Question: Are the creators or original owners of assets (e.g., code, data, models), used in the paper, properly credited and are the license and terms of use explicitly mentioned and properly respected?
    \item[] Answer: \answerYes{} 
    \item[] Justification: We have cited the original papers that produced the code package or dataset.

\item {\bf New assets}
    \item[] Question: Are new assets introduced in the paper well documented and is the documentation provided alongside the assets?
    \item[] Answer: \answerNA{} 
    \item[] Justification: The paper does not release new assets.

\item {\bf Crowdsourcing and research with human subjects}
    \item[] Question: For crowdsourcing experiments and research with human subjects, does the paper include the full text of instructions given to participants and screenshots, if applicable, as well as details about compensation (if any)? 
    \item[] Answer: \answerNA{} 
    \item[] Justification: The paper does not involve crowdsourcing nor research with human subjects.

\item {\bf Institutional review board (IRB) approvals or equivalent for research with human subjects}
    \item[] Question: Does the paper describe potential risks incurred by study participants, whether such risks were disclosed to the subjects, and whether Institutional Review Board (IRB) approvals (or an equivalent approval/review based on the requirements of your country or institution) were obtained?
    \item[] Answer: \answerNA{} 
    \item[] Justification: The paper does not involve crowdsourcing nor research with human subjects.

\item {\bf Declaration of LLM usage}
    \item[] Question: Does the paper describe the usage of LLMs if it is an important, original, or non-standard component of the core methods in this research? Note that if the LLM is used only for writing, editing, or formatting purposes and does \emph{not} impact the core methodology, scientific rigor, or originality of the research, declaration is not required.
    \item[] Answer: \answerNA{} 
    \item[] Justification:    LLMs are not used in the core methodology, experimental design, or result analysis.

\end{enumerate}

\end{document}